\documentclass[english]{article}
\usepackage[OT1]{fontenc}
\usepackage[latin9]{inputenc}
\usepackage{geometry}
\usepackage{babel}
\usepackage{float}
\usepackage{mathtools}
\usepackage{bm}
\usepackage{dsfont}
\usepackage{amsmath}
\usepackage{amssymb}
\usepackage{graphicx}
\usepackage[authoryear]{natbib}
\usepackage[unicode=true,pdfusetitle,
 bookmarks=true,bookmarksnumbered=false,bookmarksopen=false,
 breaklinks=false,pdfborder={0 0 1},backref=false,colorlinks=false]
 {hyperref}

\makeatletter

\providecommand{\tabularnewline}{\\}
\floatstyle{ruled}
\newfloat{algorithm}{tbp}{loa}
\providecommand{\algorithmname}{Algorithm}
\floatname{algorithm}{\protect\algorithmname}

\usepackage{babel}

\usepackage{algorithmic}

\usepackage{amsthm}
\usepackage{enumitem}

\theoremstyle{plain}

\newtheorem{lemma}{\textbf{Lemma}}
\newtheorem{theorem}{\textbf{Theorem}}

\newtheorem{fact}{\textbf{Fact}}
\newtheorem{proposition}{\textbf{Proposition}}

\theoremstyle{definition}

\usepackage{color}
\usepackage{dsfont}

\newcommand{\myalg}{$\mathrm{RP\text{-}MLE}$}
\newcommand{\myalgM}{$\mathrm{MRP\text{-}MLE}$}
\newcommand{\myalgW}{$\mathrm{WP\text{-}MLE}$}

\usepackage{authblk}

\author[1]{Yuepeng Yang} 
\author[2]{Cong Ma} 
\affil[1]{Department of Statistics and Data Science, Yale University} 
\affil[2]{Department of Statistics, University of Chicago} 
\date{June 21, 2024, Revised on October 29,  2025}
\@ifundefined{showcaptionsetup}{}{%
 \PassOptionsToPackage{caption=false}{subfig}}
\usepackage{subfig}
\makeatother

\begin{document}
\title{Random pairing MLE for estimation of item parameters\\ in Rasch model}
\maketitle

\begin{abstract}
The Rasch model, a classical model in the item response theory, is
widely used in psychometrics to model the relationship between individuals'
latent traits and their binary responses to assessments or questionnaires.
In this paper, we introduce a new likelihood-based estimator---random
pairing maximum likelihood estimator (\myalg) and its bootstrapped
variant multiple random pairing MLE (\myalgM) which faithfully estimate
the item parameters in the Rasch model. The new estimators have several
appealing features compared to existing ones. First, both work for
sparse observations, an increasingly important scenario in the big
data era. Second, both estimators are provably minimax optimal in
terms of finite sample $\ell_{\infty}$ estimation error. Lastly,
both admit precise distributional characterization that allows uncertainty
quantification on the item parameters, e.g., construction of confidence
intervals for the item parameters. The main idea underlying \myalg\
and \myalgM\ is to randomly pair user-item responses to form item-item
comparisons. This is carefully designed to reduce the problem size
while retaining statistical independence. We also provide empirical
evidence of the efficacy of the two new estimators using both simulated
and real data. 
\end{abstract}

\section{Introduction}

The item response theory (IRT)~\citep{embretson2013item} is a framework
widely used in psychometrics to model the relationship between individuals'
latent traits (such as ability or personality) and their responses
to assessments or questionnaires. It is particularly useful in the
development, analysis, and scoring of tests and assessments; see the
recent survey in~\cite{chen2025item} for a statistical account of IRT.

Among statistical models in IRT, the Rasch model~\citep{Rasch1960}
is a simple but fundamental one for modeling binary responses. Specifically,
for a user $t$ (e.g., test-taker) and an item $i$ (e.g., test problem),
the Rasch model assumes that the response of user $t$ to item $i$
is binary and obeys
\[
\mathbb{P}\left[\text{user \ensuremath{t} ``loses to'' item \ensuremath{i}}\right]=\frac{e^{\theta_{i}^{\star}}}{e^{\zeta_{t}^{\star}}+e^{\theta_{i}^{\star}}},
\]
where $\zeta_{t}^{\star},\theta_{i}^{\star}\in\mathbb{R}$ are latent
traits of user $t$ and item $i$, respectively. The term ``loses
to'' here refers to negative responses, such as answering an exam
question incorrectly, writing a negative review of a product, etc.
We also call this response a comparison between user $t$ and item
$i$.

In this paper, we focus on estimating the item parameters $\bm{\theta}^{\star}$,
which is one of the four main statistical tasks surrounding the Rasch
model (or IRT more generally) listed in~\cite{chen2025item}. Estimating
the item parameters is practically important. For instance, in education
testing, $\bm{\theta}^{\star}$ could reveal the difficulty of the
exam questions, while in product reviews, $\bm{\theta}^{\star}$ could
reveal the popularity of the products. Various methods have been proposed
for estimating the item parameters $\bm{\theta}^{\star}$, including
the joint maximum likelihood estimator (JMLE), the marginal maximum
likelihood estimator, the conditional maximum likelihood estimator
(CMLE), and the spectral estimator recently proposed in~\cite{nguyen2022spectral,nguyen2023optimal}.
We refer readers to a recent article~\citep{robitzsch2021comparison}
for comparisons between different item parameter estimation methods.
However, three main gaps remain in tackling item estimation in the
Rasch model: 
\begin{itemize}
\item \textbf{Non-asymptotic guarantee}. Apart from the recently proposed
spectral estimator \citep{nguyen2022spectral,nguyen2023optimal},
most theoretical guarantees for the likelihood-based estimators are
asymptotic. Since all the estimation procedures are necessarily applied
with finite samples, the asymptotic guarantee alone fails to inform
practitioners about the performance of different estimators when working
with a limited number of samples. 
\item \textbf{Sparse observations}.\textbf{ }It is not uncommon to encounter
situations where each user only responds to a handful of questions.
This brings the challenge of incomplete or sparse observations. Many
methods, such as CMLE, allow incomplete data \citep{Molenaar1995},
but most of them lack theoretical support in the sparse regime. 
While the spectral estimator~\citep{nguyen2022spectral,nguyen2023optimal}
is capable of handling incomplete observations, its theory still requires
the observations to be relatively dense. We will elaborate on this
point later. 
\item \textbf{Uncertainty quantification}. Beyond estimation, uncertainty
quantification on the item parameters is central to realizing the
full potential of the Rasch model. However, existing results do not
address this problem under sparse observations. An exception is the
recent work by \cite{Chen2023}, which is based on joint estimation and
inference on the item parameters $\bm{\theta}^{\star}$ and the user
parameters $\bm{\zeta}^{\star}$. Their sampling scheme is more restrictive
and requires a relatively dense sampling rate. 
\end{itemize}
In light of these gaps, we raise the following question:

\medskip

\emph{Can we develop an estimator for the item parameters }$\bm{\theta}^{\star}$\emph{
that (1) enjoys optimal estimation guarantee in finite sample, and
(2) is amenable to tight uncertainty quantification, when the observations
are sparse?}

\subsection{Main contributions}

The main contribution of our work is the proposal of a novel estimator
named random pairing maximum likelihood estimator (\myalg\ in short)
that achieves the two desiderata listed above. 

In essence, \myalg\ compiles user-item comparisons to item-item comparisons
by randomly pairing responses of the same user to different items.
This pairing procedure is carefully designed to extract information
of the item parameters while retaining statistical independence. After
this compilation step, item parameters $\bm{\theta}^{\star}$ are
estimated by the MLE $\widehat{\bm{\theta}}$ given the item-item
comparisons. 

Even when the observations are extremely sparse, \myalg\ achieves
the following:
\begin{itemize}
\item Regarding estimation, we show that both \myalg\ and its bootstrapped
version enjoy optimal finite sample $\ell_{\infty}$ error guarantee.
Compared to the conventional $\ell_{2}$ error guarantee, the $\ell_{\infty}$
guarantee, as an entrywise guarantee, is more fine-grained. Consequently,
we also show that \myalg\ can recover the top-$K$ items with optimal
sample complexity. 
\item While the optimal $\ell_{\infty}$ error guarantee directly yields
optimal $\ell_{2}$ guarantee, such guarantee is only correct in an
order-wise sense. We provide a refined finite-sample $\ell_{2}$ error
guarantee of \myalg\ that is precise even in the leading constant. 
\item Supplementing the estimation guarantee, we also build an inferential
framework based on \myalg\ $\widehat{\bm{\theta}}$. More specifically,
we precisely characterize the asymptotic distribution of $\widehat{\bm{\theta}}$.
This result facilitates several inferential tasks such as hypothesis
testing and construction of confidence regions of $\bm{\theta}^{\star}$.
\end{itemize}
We test our methods on both synthetic and real data, which clearly
show competitive empirical estimation performance. The inferential
result on synthetic data also closely matches our theoretical predictions.

\subsection{Prior art}

\paragraph{Item response theory.}

The item response theory is a popular statistical framework for modeling
response data. It often involves a probabilistic model that links
categorical responses to latent traits of both users and items. In the early
endeavors, \cite{Rasch1960} introduced the Rasch model
studied herein, and \cite{Lord1968} describes a more general
framework using parametric models. Popular IRT models include the
Rasch model, the two-parameter model (2PL), and the three-parameter
logistic model (3PL). As response data widely appears in real life,
IRT finds application in numerous fields including educational assessment
\citep{DeChamplain2010}, psychometrics \citep{Lord1968}, political
science \citep{VanHauwaert2020}, and medical assessment \citep{Fries2005}.
See \cite{chen2025item} for an overview of IRT.

\paragraph{Latent score estimation for Rasch model.}

An important statistical question in the Rasch model is to estimate
the item parameters. As the Rasch model is an explicit probabilistic
model, many estimation methods are based on the principle of maximizing
likelihood. For instance, marginal MLE assumes a prior on the user
parameters that is either given or optimized within a parametric distribution
family. The item parameter is then estimated by maximizing the marginal
likelihood. A drawback is that MMLE relies on a good prior. On the
other hand, joint MLE (JMLE) makes no distributional assumption and
maximizes the joint likelihood w.r.t.~both the item and user parameters.
However, it is not consistent for estimating the item parameters when
the number of items is fixed \citep{Ghosh1995}. Interested readers
may also consult \cite{Linacre1999} for an overview of other classical
estimators.

Several methods are more relevant to our proposed estimator \myalg\
as they follow a similar philosophy to form item-item comparisons
from user-item responses. Conditional MLE (CMLE) considers the tuple
of all items that are related to a user instead of examining the induced
item pairs. It maximizes the likelihood condition on the total number
of positive response a user has. Theoretically, \cite{Andersen1973}
has shown that CMLE is asymptotically normal and consistent. However,
no non-asymptotic rate in the setting of sparse observation has been
established. Pseudo MLE (PMLE) \citep{zwinderman1995pairwise} maximizes
the sum of the log-likelihood of all pairs of responses from the same
users to different items. However, due to the dependence, no satisfying
finite sample performance guarantee has been established. Another
related approach is the spectral method, in which a Markov chain on
the space of items is formed and the item parameters are estimated
via the stationary distribution of the Markov chain. The most recent
works in this category are \cite{nguyen2022spectral,nguyen2023optimal},
which essentially use the same idea as pseudo MLE in forming item-item
comparisons.

\paragraph{The Bradley-Terry-Luce model with sparse comparisons.}

An informed reader may realize that the Rasch model resembles the
Bradley-Terry-Luce (BTL) model \citep{Luce1959,bradley1952rank} in
the ranking literature. Indeed, one can view the Rasch model as a
special case of the BTL model that distinguishes the two groups of
users and items, and only allows inter-group comparisons. There has
been a recent surge in interest in studying top-$K$ ranking in the
BTL model \citep{simons1999asymptotics,yan2012sparse,chen2015spectral,jang2016top,chen2019spectral,han2020asymptotic,chen2022partial,gao2023uncertainty,liu2023lagrangian}
and its extensions~\citep{han2023unified, han2023general, fan2024covariate,fan2024uncertainty,fan2025spectral,fan2025ranking},
especially under sparse observations of the pairwise responses. Most
notably, under a uniform sampling scheme, \cite{chen2019spectral}
shows that (regularized) MLE and spectral methods are both optimal
in top-$K$ ranking and \cite{gao2023uncertainty} provides inference
results for both methods. 

Another line of research focuses on non-uniform sampling. \cite{hajek2014minimax}
and \cite{shah2016estimation} each studies the $\ell_{2}$ error
of MLE with a general sampling graph. They obtain high probability
upper bound and minimax lower bound that match under in some scenarios.
In particular, \cite{shah2016estimation} extensively discusses the
implication of their result in different graph topology. More recently,
general sampling graph has also been studied in top-$K$ ranking.
Several articles \citep{han2023unified, chen2023ranking,li2022ell_} investigate the performance of
MLE in the BTL model with a general comparison graph and later \cite{Yang2024}
improves the analysis to show the optimality of (weighted) MLE for
the BTL model in both uniform and semi-random sampling. 

\paragraph{Notation. }

For a positive integer $n$, we denote $[n]=\{1,2,\ldots,n\}$. For
any $a,b\in\mathbb{R}$, $a\wedge b$ means the minimum of $a,b$
and $a\vee b$ means the maximum of $a,b$. We use $\overset{\mathrm{a.s}}{\rightarrow}$,
$\overset{\mathrm{p}}{\rightarrow}$, and $\overset{\mathrm{d}}{\rightarrow}$
to denote convergence almost surely, in probability, and in distribution
respectively. For a symmetric matrix $\bm{A}\in\mathbb{R}^{n\times n}$,
we use $\lambda_{1}(\bm{A})\ge\lambda_{2}(\bm{A})\ge\cdots\ge\lambda_{n}(\bm{A})$
to denote its eigenvalues and $\bm{A}^{\dagger}$ to denote its Moore-Penrose
pseudo-inverse. For symmetric matrices $\bm{A},\bm{B}\in\mathbb{R}^{n\times n}$,
$\bm{A}\preceq\bm{B}$ means $\bm{B}-\bm{A}$ is positive semidefinite,
i.e., $\bm{v}^{\top}(\bm{B}-\bm{A})\bm{v}\ge0$ for any $\bm{v}\in\mathbb{R}^{n}$.
We use $\bm{e}_{i}$ to denote the standard unit vector with $1$
at $i$-th coordinate and 0 elsewhere. Unless specified otherwise,
$\log(\cdot)$ denotes the natural log.

\section{Problem setup and new estimators}

In this section, we first introduce the formal setup of the item parameter
estimation problem in the Rasch model. Then we present the news estimator
\myalg\ and \myalgM\ along with the rationale behind its development.

\subsection{Problem setup \label{subsec:problem_setup}}

The Rasch model considers pairwise comparisons between elements of
two groups: users and items. Let $n$ (resp.~$m$) be the number
of users (resp.~items). Rasch assumes a user parameter $\bm{\zeta}^{\star}\in\mathbb{R}^{n}$
and an item parameter $\bm{\theta}^{\star}\in\mathbb{R}^{m}$ that
measures the latent traits (e.g., difficulty of a problem) of users
and items, respectively. For a subset of possible user-item pairs
$\mathcal{E}_{X}\subset[n]\times[m]$, we observe binary responses
$\{X_{ti}\}_{(t,i)\in\mathcal{E}_{X}}$ obeying 
\begin{equation}
\mathbb{P}[X_{ti}=1]=\frac{e^{\theta_{i}^{\star}}}{e^{\zeta_{t}^{\star}}+e^{\theta_{i}^{\star}}}.\label{eq:sampling_prob}
\end{equation}
Here $X_{ti}=1$ means user $t$ has negative response against item
$i$ (e.g., unable to solve a problem). The goal is to estimate $\bm{\theta}^{\star}$,
the item parameters. 

To model sparse user-item responses, we assume that $\mathbb{P}[(t,i)\text{ is compared}]=p$
independently for every $(t,i)\in[n]\times[m]$. To put it in the
language of graph theory, we denote the associated bipartite comparison
graph to be $\mathcal{G}_{X}=(\mathcal{V}_{X},\mathcal{E}_{X})$,
where $\mathcal{V}_{X}$ consists of $n$ users and $m$ items. Then
essentially, we are assuming that the bipartite graph follows an Erd\H{o}s-R\'enyi
random model.\footnote{Alternatively, we can assume each user responds
to $mp$ items uniformly at random. Our estimator and performance
guarantee continue to work in this sampling scheme.}

Before moving on, we define condition numbers to characterize the
range of the latent traits. Let $\kappa_{1}$, $\kappa_{2}$, and
$\kappa$ be defined by $\log\;(\kappa_{1})\coloneqq\max_{ij}\{|\theta_{i}^{\star}-\theta_{j}^{\star}|\}$,
$\log(\kappa_{2})\coloneqq\max_{ti}\{|\zeta_{t}^{\star}-\theta_{i}^{\star}|\}$,
and $\kappa\coloneqq\max\{\kappa_{1},\kappa_{2}\}$, respectively. 

\subsection{Random pairing maximum likelihood estimator\label{subsec:algo}}

\begin{algorithm}[t]
\caption{Random Pairing Maximum Likelihood Estimator (\myalg)\label{alg:RP-MLE}}

\begin{enumerate}
\item For each tester $t$, 
\begin{enumerate}
\item Randomly split the $m_{t}$ problems taken by tester $t$ into $\lfloor m_{t}/2\rfloor$
pairs of problems.
\item For each $(i,j)\in[m]\times[m]$, do the following: 
\begin{enumerate}
\item If $(i,j)$ is selected as a pair in Step 1(a), $R_{ij}^{t}=1$. Furthermore,
if $X_{ti}\neq X_{tj}$, let $Y_{ij}^{t}=\mathds{1}\{X_{ti}<X_{tj}\}$
and $L_{ij}^{t}=1$; if $X_{ti}=X_{tj}$, let $L_{ij}^{t}=0$. 
\item If $(i,j)$ is not selected as a pair in Step 1(a), let $L_{ij}^{t}=0$
and $R_{ij}^{t}=0$. 
\end{enumerate}
\end{enumerate}
\item Let $\mathcal{E}_{Y}$ be a set of edges defined by $\mathcal{E}_{Y}\coloneqq\{(i,j):\sum_{t=1}^{n}L_{ij}^{t}\ge1\}$
and let $\mathcal{G}_{Y}=([m],\mathcal{E}_{Y})$. For each $(i,j)\in\mathcal{E}_{Y}$,
let $L_{ij}\coloneqq\sum_{t=1}^{n}L_{ij}^{t}$ and $Y_{ij}\coloneqq(1/L_{ij})\sum_{\{t:L_{ij}^{t}=1\}}Y_{ij}^{t}$.
\item Compute MLE on $Y_{ij}$, i.e.,
$
\widehat{\bm{\theta}}\coloneqq\arg\min_{\bm{\theta}:\bm{1}_{m}^{\top}\bm{\theta}=0}\mathcal{L}(\bm{\theta})$,
where 
\begin{equation}
\mathcal{L}(\bm{\theta})=\sum_{(i,j)\in\mathcal{E}_{Y},i>j}L_{ij}\left(-Y_{ji}(\theta_{i}-\theta_{j})+\log(1+e^{\theta_{i}-\theta_{j}})\right).\label{eq:MLE_loss}
\end{equation}
\item Return the top-$K$ items by selecting the top-$K$ entries of $\widehat{\bm{\theta}}$. 
\end{enumerate}
\end{algorithm}

In this section, we present our main method \myalg; see Algorithm~\ref{alg:RP-MLE}.
The algorithm can be divided into two parts. The first part---Steps
1 and 2---uses random pairing to compile the observed user-item responses
$\bm{X}\in\mathbb{R}^{n\times m}$ to item-item comparisons $\bm{Y}\in\mathbb{R}^{m\times m}$.
The second part---Steps 3 and 4---computes a standard MLE on the
item-item comparisons. Some intuition regarding the development of
\myalg\ is in order. 

\paragraph{Random pairing to construct item-item comparisons. }

The idea of pairing is that by matching the responses $X_{ti}$ with
$X_{tj}$, we form a comparison between items $i$ and $j$ to directly
extract information of item parameters $\theta_{i}^{\star}$ and $\theta_{j}^{\star}$.
More specifically, the item-item comparisons $\bm{Y}$ follow the
Bradley-Terry-Luce model \citep{bradley1952rank,Luce1959}, i.e., $\mathbb{P}[Y_{ij}^{t}=1]=e^{\theta_{j}^{\star}}/(e^{\theta_{i}^{\star}}+e^{\theta_{j}^{\star}})$. For the claims made in this part, please
see Section~\ref{subsec:Reduction} for a formal argument. 

By compiling user-item responses to item-item comparisons, we reduce
the size of the data matrix from $n\times m$ to $m\times m$, and
also the number of intrinsic parameters from $n+m$ to $m$, since
the likelihood function~(\ref{eq:MLE_loss}) of $Y_{ij}^{t}$ is
completely independent with the user parameter $\zeta_{t}^{\star}$. 

More importantly, the pairing is performed in a disjoint fashion.
This ensures that all constructed item-item comparisons $Y_{ij}^{t}$
are independent with each other; see Section~\ref{subsec:Reduction}
for a formal statement. This is the key ingredient that enables us
to improve over previous implementation of item-item comparisons,
such as pseudo-likelihood \citep{choppin1982fully,zwinderman1995pairwise}
and spectral methods \citep{nguyen2022spectral,nguyen2023optimal}. 

\paragraph{A variant via bootstrapping. }

A drawback of this random pairing is that it potentially induces a
loss of information since not every possible pairing is considered.
Once $X_{ti}$ is paired with $X_{tj}$, we cannot pair $X_{ti}$
with another response $X_{tl}$. That being said, we will later show
that the $\ell_{\infty}$ error of \myalg\ is still rate-optimal
up to logarithmic factors. Hence the loss of information can at most
incur a small constant factor in terms of estimation error. Nevertheless,
we provide a remedy to this phenomenon in \myalgM\ (Algorithm~\ref{alg:MRP-MLE})
by running (in other words, bootstrapping) the \myalg\ multiple times
with different random data splitting and averaging the resulting estimates.
\myalgM\ trivially enjoys the same non-asymptotic $\ell_{\infty}$
error rate (cf.~Theorem~\ref{thm:rasch_infty}) while improving
the estimation error in practice over \myalg. See Figure~\ref{fig:err_diff_method}
in Section~\ref{subsec:Simulation} for the empirical evidence.

\begin{algorithm}[t]
\caption{Multiple Random Pairing Maximum Likelihood Estimator (\myalgM) \label{alg:MRP-MLE}}

\begin{enumerate}
\item Let $n_{\mathrm{split}}$ be the number of runs. For $i=1,\ldots,n_{\mathrm{split}}$,
run \myalg\ (Algorithm~\ref{alg:RP-MLE}), each time with an independent
random splitting in Step 1. Let the estimated latent scores be $\widehat{\bm{\theta}}^{(i)}$.
\item Estimate the latent score with
\[
\widehat{\bm{\theta}}_{\mathrm{MRP}}=\frac{1}{n_{\mathrm{split}}}\sum_{i=1}^{n_{\mathrm{split}}}\widehat{\bm{\theta}}_{(i)}.
\]
\item Estimate the top-$K$ items by selecting the top-$K$ entries of $\widehat{\bm{\theta}}_{\mathrm{MRP}}$.
\end{enumerate}
\end{algorithm}

\section{Main results}

In this section, we collect the main theoretical guarantees for \myalg\
and its variant \myalgM. Section~\ref{subsec:infty-error} focuses
on the finite sample $\ell_{\infty}$ error bound. In Section~\ref{subsec:Non_asymp_expansion},
we present a non-asymptotic expansion that describes the distribution
of \myalg, and we apply this expansion to reach a Berry-Esseen type
theorem and a much sharper characterization of the $\ell_{2}$ error
of \myalg. Lastly in Section~\ref{subsec:Asymptotic_normality},
we prove the asymptotic normality of \myalgM\ when $m$ and $p$
are fixed and draw a connection between \myalgM~and (weighted) pseudo
MLE. 

\subsection{$\ell_{\infty}$ error bounds and top-$K$ recovery \label{subsec:infty-error}}

Without loss of generality, we assume that the scores of the items
are ordered, i.e., $\theta_{1}^{\star}\ge\theta_{2}^{\star}\geq\cdots\geq\theta_{m}^{\star}$,
and denote $\Delta_{K}\coloneqq\theta_{K}^{\star}-\theta_{K+1}^{\star}$.
In words, $\Delta_{K}$ measures the difference between the difficulty
levels of items $K$ and $K+1$. The following theorem provides $\ell_{\infty}$
error bounds and top-$K$ recovery guarantee for both \myalg\ and
\myalgM. We defer its proof to Section~\ref{subsec:Analysis_infty}. 

\begin{theorem}\label{thm:rasch_infty}Suppose that $mp\ge2$ and
$np\ge C_{1}\kappa_{1}^{4}\kappa_{2}^{5}\log^{3}(n)$ for some sufficiently
large constant $C_{1}>0$. Suppose that there exists some constant
$\alpha>0$ such that $m\le n^{\alpha}$. Let $\widehat{\bm{\theta}}$
be the \myalg\ estimator. With probability at least $1-O(n^{-10})$,
$\widehat{\bm{\theta}}$ satisfies
\[
\|\widehat{\bm{\theta}}-\bm{\theta}^{\star}\|_{\infty}\le C_{2}\kappa_{1}\kappa_{2}^{1/2}\sqrt{\frac{\log(n)}{np}}.
\]
Consequently, the estimator is able to exactly recover the top-$K$
items as soon as 
\[
np\ge\frac{C_{3}\kappa_{1}^{2}\kappa_{2}^{1}\log(n)}{\Delta_{K}^{2}}.
\]
Here $C_{2},C_{3}>0$ are some universal constants. All the claims
continue to hold for \myalgM\ as long as there exists some constant
$\beta>0$ such that $n_{\mathrm{split}}\le n^{\beta}$.

\end{theorem}Some remarks are in order.

\paragraph{Finite sample minimax optimality. }

Based on results from the ranking literature, it is reasonable to
guess that the optimal $\ell_{\infty}$ error is $O(\sqrt{m/(mnp)})=O(1/\sqrt{np})$.
Here $m$ is the number of item parameters, and $mnp$ is the number
of user-item comparisons. This guess is indeed correct, as formalized
in the following result from~\cite{nguyen2023optimal}. \begin{proposition}[Minimax
lower bound, Theorems~3.3 and 3.4 in \cite{nguyen2023optimal}]\label{prop:L_infty_lb}
Assume that $np\ge C_{1}$ for some sufficiently large constant $C_{1}>0$.
For any $n$ and $m$, there exists a class of user and item parameters
$\Theta$ such that
\[
\inf_{\widehat{\bm{\theta}}}\sup_{(\bm{\zeta}^{\star},\bm{\theta}^{\star})\in\Theta}\mathbb{E}\left\Vert \widehat{\bm{\theta}}-\bm{\theta}^{\star}\right\Vert _{2}^{2}\ge\frac{C_{2}m}{np},\qquad\text{and}\qquad\inf_{\widehat{\bm{\theta}}}\sup_{(\bm{\zeta}^{\star},\bm{\theta}^{\star})\in\Theta}\mathbb{E}\left\Vert \widehat{\bm{\theta}}-\bm{\theta}^{\star}\right\Vert _{\infty}^{2}\ge\frac{C_{2}}{np},
\]
where $C_{2}>0$ is some constant. Moreover if $np\le C_{K}\log(m)/\Delta_{K}^{2}$
for some constant $C_{K}>0$, we have 
\[
\inf_{\widehat{\bm{\theta}}}\sup_{(\bm{\zeta}^{\star},\bm{\theta}^{\star})\in\Theta}\mathbb{P}\left[\widehat{\bm{\theta}}\text{ fails to identify all top-}\ensuremath{K}\text{ items}\right]\ge\frac{1}{2}.
\]
\end{proposition} Comparing our upper bounds with the lower bound
in the proposition, we can see that both \myalg\ and \myalgM\ are
rate-optimal in $\ell_{\infty}$ estimation error and top-$K$ recovery
sample complexity, up to logarithmic and $\kappa$ factors.

\paragraph{Sample size requirement. }

While the rates are optimal, it is worth noting that in Theorem~\ref{thm:rasch_infty}
we have made several sample size requirements. We now elaborate on
them.

First, the assumption $mp\ge2$ is a mild requirement on the expected
number of items compared by each user. This is required as we need
user $t$ to compare at least two items to form a comparison between
items. In fact, if a user only responds to one item, it is clear that
this data point is not useful at all for item parameter estimation. 

Second, it is a standard and necessary requirement to have $np\gtrsim\log(n)$
to make sure that each item is compared to at least one user with
high probability. In Theorem~\ref{thm:rasch_infty} we require an
extra $\log^{2}(n)$ factor to suppress a quadratic error term that
comes up in the analysis. This cubic log factor can possibly be loose,
but it is a minor issue and we leave it to future research. 

Lastly, $m\le n^{\alpha}$ and $n_{\mathrm{split}}\le n^{\beta}$
are both minor as we only need these to allow union bounds over $m$
and $n_{\mathrm{split}}$.

\paragraph{Comparison with \cite{nguyen2023optimal}. }

The closest result to our paper in terms of $\ell_{\infty}$ guarantee
for the Rasch model appears in the recent work by \cite{nguyen2023optimal}.
Their spectral method uses a similar construction of the item-item
comparisons but without disjoint pairing. To provide detailed comparisons,
we restate their results below.

\begin{proposition}[Informal, Theorem~3.1 in \cite{nguyen2023optimal}]Assume
that $p\gtrsim\log(m)/\sqrt{n}$ and $mp\gtrsim\log(n)$, with probability
at least $1-O(m^{-10}+n^{-10})$, spectral estimator $\widehat{\bm{\theta}}_{\mathrm{spectral}}$
satisfies
\[
\|\widehat{\bm{\theta}}_{\mathrm{spectral}}-\widehat{\bm{\theta}}\|_{\infty}\lesssim\kappa^{9}\sqrt{\frac{\log(m)}{np}}.
\]
\end{proposition} This error rate is similar to ours. However, the
required sample size is much larger as they require $p\gtrsim\log(m)/\sqrt{n}$.
Our result makes a significant improvement by allowing a much smaller
sampling rate $p$, cf\@.~$mp\ge2$ and $np\gtrsim\log^{3}(n)$.
In fact, as we have argued earlier, it is nearly the sparsest possible
regime for estimating item parameters. In addition, our methods enjoy
a significantly better error rate dependency on $\kappa$. In Section~\ref{subsec:Simulation},
we provide empirical evidence for this improvement: when $\kappa$
is large, \myalg\ and \myalgM\ outperform the spectral methods in
\cite{nguyen2023optimal}.

\paragraph*{Analysis via reduction to BTL model.} As mentioned in Section~\ref{subsec:algo}, the random pairing in \myalg~reduces the problem to the Bradley-Terry-Luce model with non-uniform sampling. This reduction results in an item-item comparison graph with nice spectral properties, and allows us to invoke the general theory of
MLE in the BTL model established in the recent work by~\cite{Yang2024}. 
See Section~\ref{subsec:Analysis_infty} for the full analysis.

\subsection{Non-asymptotic expansion of \myalg \label{subsec:Non_asymp_expansion}}

An important aspect of statistical estimators is the quantification
of the variability. In this section we provide a non-asymptotic expansion,
which precisely characterizes of the distribution of the \myalg\
estimator $\widehat{\bm{\theta}}$. We supplement this result with
a Berry-Esseen theorem and an application to obtain a precise $\ell_{2}$
error characterization.

We start with some necessary notation. Let $\sigma(x)=e^{x}/(1+e^{x})$
be the sigmoid function. Let $z_{ij}\coloneqq e^{\theta_{i}^{\star}}e^{\theta_{j}^{\star}}/(e^{\theta_{i}^{\star}}+e^{\theta_{j}^{\star}})^{2}$,
$\widehat{z}_{ij}\coloneqq e^{\hat{\theta}_{i}}e^{\hat{\theta}_{j}}/(e^{\hat{\theta}_{i}}+e^{\hat{\theta}_{j}})^{2}$
and $\epsilon_{ij}^{t}\coloneqq Y_{ji}^{t}-\sigma(\theta_{i}^{\star}-\theta_{j}^{\star})$.
Let $L_{\mathrm{total}}\coloneqq\sum_{i>j:(i,j)\in\mathcal{E}_{Y}}L_{ij}$
be the total number of observed comparisons in $\mathcal{G}_{Y}$. 

We define $\bm{\bm{B}}\in\mathbb{R}^{m\times L_{\mathrm{total}}}$
and $\widehat{\bm{\epsilon}}\in\mathbb{R}^{L_{\mathrm{total}}}$ via
\[
\bm{B}\coloneqq\left[\cdots,\sqrt{z_{ij}}(\bm{e}_{i}-\bm{e}_{j}),\cdots\right]_{i>j:(i,j)\in\mathcal{E}_{Y}}\text{\ensuremath{\quad}(repeat \ensuremath{L_{ij}} times for edge \ensuremath{(i,j)})}
\]
and
\[
\widehat{\bm{\epsilon}}\coloneqq\left[\cdots,\epsilon_{ij}^{t}/\sqrt{z_{ij}},\cdots\right]_{(i,j,t):i>j,(i,j)\in\mathcal{E}_{Y},L_{ij}^{t}=1}\in\mathbb{R}^{L_{\mathrm{total}}}.
\]
Moreover, define a weighted graph Laplacian 
\begin{equation}
\bm{L}_{L\tilde{z}}\coloneqq\sum_{(i,j)\in\mathcal{E}_{Y},i>j}L_{ij}\widetilde{z}_{ij}(\bm{e}_{i}-\bm{e}_{j})(\bm{e}_{i}-\bm{e}_{j})^{\top}\label{eq:Laplacian}
\end{equation}
for $\tilde{z}$ being $z$ or $\widehat{z}$, and let $\bm{L}_{L\tilde{z}}^{\dagger}$
be its pseudo-inverse.

The following theorem characterizes the distribution of \myalg\ with
a non-asymptotic expansion. The analysis is deferred to Section~A.3
and the full proof is deferred to Section~\ref{subsec:Proof_rasch_dist}. 

\begin{theorem}\label{thm:rasch_distribution} Instate the assumptions
of Theorem~\ref{thm:rasch_infty}. With probability at least $1-O(n^{-10})$,
the estimator $\widehat{\bm{\theta}}$ given by the Algorithm~\ref{alg:RP-MLE}
can be written as
\begin{align}
\widehat{\bm{\theta}}-\bm{\theta}^{\star} & =-\left[\nabla^{2}\mathcal{L}(\bm{\theta}^{\star})\right]^{\dagger}\nabla\mathcal{L}(\bm{\theta}^{\star})+\bm{r}=-\bm{L}_{Lz}^{\dagger}\bm{B}\widehat{\bm{\epsilon}}+\bm{r},\label{eq:delta_hat_RV}
\end{align}
where $\bm{r}\in\mathbb{R}^{m}$ is a random vector obeying $\|\bm{r}\|_{\infty}\le C\kappa_{1}^{6}\log^{2}(n)/(np)$
for some constant $C>0$.

\end{theorem} This theorem shows $\widehat{\bm{\theta}}-\bm{\theta}^{\star}$
can be well approximated by $-[\nabla^{2}\mathcal{L}(\bm{\theta}^{\star})]^{\dagger}\nabla\mathcal{L}(\bm{\theta}^{\star})$,
a form that frequently appears in the analysis of maximum likelihood
estimators. Moreover, it can be written as a linear transformation
of the random vector $\widehat{\bm{\epsilon}}$, whose entries are
independent outcomes of the item-item comparisons shifted and scaled
to be mean-zero and variance-one. The term $\bm{L}_{Lz}^{\dagger}\bm{B}$
accounts for the geometry induced by the comparison graph $\mathcal{G}_{Y}$.
As the residual term $\bm{r}$ has small magnitude, we may analyze
the properties of $\widehat{\bm{\theta}}-\bm{\theta}^{\star}$ by
focusing on the leading term $-\bm{L}_{Lz}^{\dagger}\bm{B}\widehat{\bm{\epsilon}}$.

We compare our result with the inference result in~\cite{Chen2023}.
Theorem~8 therein studies the Rasch model and provides inferential
results for a joint estimator of $\bm{\theta}^{\star}$ and $\bm{\zeta}^{\star}$.
However, it requires a dense sampling scheme when $n\ge m$, with
\[
p\gtrsim\sqrt{\frac{1}{m}}\vee\frac{\log^{2}(m)}{n}\vee\frac{n\log^{2}(n)}{m^{2}}.
\]
It also requires that both $m$ and $n$ tend to infinity. These two
assumptions are significantly more restrictive than ours.

\paragraph{Normal approximation. }

The main term in (\ref{eq:delta_hat_RV}) can be approximated with
a normal random variable, allowing various applications such as hypothesis
testing on $\bm{\theta}^{\star}$. Formally, we present the following
Berry-Esseen type theorem. The proof is deferred to Section~\ref{subsec:Proof_conv_dist}.

\begin{proposition}\label{prop:conv_dist} Instate the assumptions
of Theorem~\ref{thm:rasch_infty}. Let $\bm{x}$ be a normal random
variable in $\mathbb{R}^{m}$ with variance $\bm{L}_{Lz}^{\dagger}$.
Let $\mathcal{C}_{m}$ be the set of all the measurable convex subset
of $\{\bm{\theta}\in\mathbb{R}^{m}:\bm{\theta}^{\top}\bm{1}_{m}=0\}$.
Then we have that 
\[
\sup_{A\in\mathcal{C}_{m}}\left|\mathbb{P}\left[\bm{L}_{Lz}^{\dagger}\bm{B}\widehat{\bm{\epsilon}}\in A\mid\mathcal{G}_{Y}\right]-\mathbb{P}(\bm{x}\in A)\right|\le C_{1}n^{-10}+C_{2}\frac{m^{5/4}\kappa_{1}^{3/2}\kappa_{2}^{3/2}}{(np)^{1/2}}.
\]

\end{proposition}

\paragraph{Refined $\ell_{2}$ error characterization.}

Another possible application of Theorem~\ref{thm:rasch_distribution}
is a refined characterization of the $\ell_{2}$ estimation error
$\|\widehat{\bm{\theta}}-\bm{\theta}^{\star}\|$. The $\ell_{\infty}$
error guarantee in Theorem~\ref{thm:rasch_infty} immediately implies
an $\ell_{2}$ error bound 
\begin{equation}
\|\widehat{\bm{\theta}}-\bm{\theta}^{\star}\|\le C\kappa_{1}\kappa_{2}^{1/2}\sqrt{\frac{m\log(n)}{np}},\label{eq:l2_naive}
\end{equation}
which is rate-optimal compared to the minimax lower bound in Proposition~\ref{prop:L_infty_lb}.
However, this guarantee is only correct in an order-wise sense. Here,
we present a refined characterization of $\|\widehat{\bm{\theta}}-\bm{\theta}^{\star}\|$
that is precise in the leading constant. In the following theorem,
we show that $\|\widehat{\bm{\theta}}-\bm{\theta}^{\star}\|$ concentrates
tightly around $[\mathrm{Trace}(\bm{L}_{L\tilde{z}}^{\dagger})]^{1/2}$.
We defer the complete proof to Section~\ref{subsec:Proof_L2}.

\begin{proposition}\label{prop:rasch_l2} Instate the assumptions
of Theorem~\ref{thm:rasch_infty}. Then for some constants $C_{1},C_{2}>0$,
with probability at least $1-O(n^{-10})$, we have 
\begin{align}
\left|\|\widehat{\bm{\theta}}-\bm{\theta}^{\star}\|-\sqrt{\mathrm{Trace}(\bm{L}_{Lz}^{\dagger})}\right| & \le C_{1}\kappa_{1}^{3}\kappa_{2}\sqrt{\frac{\log(n)}{np}}+\frac{C_{2}\kappa_{1}^{6}\sqrt{m}\log^{2}(n)}{np};\label{eq:rasch_l2}\\
\left|\|\widehat{\bm{\theta}}-\bm{\theta}^{\star}\|-\sqrt{\mathrm{Trace}(\bm{L}_{L\widehat{z}}^{\dagger})}\right| & \le C_{1}\kappa_{1}^{3}\kappa_{2}\sqrt{\frac{\log(n)}{np}}+\frac{C_{2}(\kappa_{1}^{6}+\kappa_{1}^{7/2}\kappa_{2}^{2})\sqrt{m}\log^{2}(n)}{np}.\label{eq:rasch_l2_hat}
\end{align}

\end{proposition}

Theorem~\ref{prop:rasch_l2} is more refined compared to (\ref{eq:l2_naive}).
First, it provides both upper and lower bounds for $\|\widehat{\bm{\theta}}-\bm{\theta}^{\star}\|$.
Second, there is no hidden constant in front of the leading term $[\mathrm{Trace}(\bm{L}_{Lz}^{\dagger})]^{1/2}$.
In addition, inspecting the proof of Theorem~\ref{prop:rasch_l2},
we see that 
\[
\sqrt{\frac{m-1}{np}}\le\sqrt{\mathrm{Trace}(\bm{L}_{Lz}^{\dagger})}\le4\kappa_{1}^{1/2}\kappa_{2}^{1/2}\sqrt{\frac{m}{np}},
\]
and the same holds for $[\mathrm{Trace}(\bm{L}_{L\widehat{z}}^{\dagger})]^{1/2}$.
Consequently, the right hand sides of both (\ref{eq:rasch_l2}) and
(\ref{eq:rasch_l2_hat}) are lower order terms compared to $[\mathrm{Trace}(\bm{L}_{Lz}^{\dagger})]^{1/2}$
when $n,m\rightarrow\infty$. Indeed this recovers the naive $\ell_{2}$
bound (\ref{eq:l2_naive}) under appropriate sample size assumptions.

\paragraph*{Analysis via projected gradient descent trajectory.}

Inspired by \cite{chen2023ranking}, we analyze MLE via the projected gradient descent trajectory. This approach gives us an alternative to the leave-one-out type argument in \cite{chen2019spectral,chen2022partial} and the leave-two-out argument in \cite{gao2023uncertainty}. In particular, unlike the leave-one-out and leave-two-out arguments, the projected gradient descent approach does not require independence in the sampling of the compared pairs. This is crucial to our analysis as the disjoint pairing in Step 1(a) of Algorithm~\ref{alg:RP-MLE} induces dependent item-item edges. 
See Section~\ref{subsec:Analysis_non_asymp_expansion} for the full analysis.

\subsection{Asymptotic normality of \myalgM\ and pseudo MLE \label{subsec:Asymptotic_normality}}

In this section, we consider the inference setting where $m$ and
$p$ are fixed and $n$ tends to infinity. We establish the asymptotic
normality of \myalgM\ and connect it to a weighted variant of pseudo
MLE, which we call \myalgW. We show that these two estimators are
asymptotically equal in distribution. 

We start by describing the setup of an infinite sequence of estimators
$\{\widehat{\bm{\theta}}_{\mathrm{MRP}}^{(n)}\}$ and relevant notations.
We will consider the user parameters to be fixed and study the asymptotic
normality of $\widehat{\bm{\theta}}_{\mathrm{MRP}}^{(n)}$ that accounts
for the randomness in sampling, random pairing, and the comparison
between paired outcomes. We use subscripts to denote the source of
randomness in each expectation. For sampling, i.e., the generation
of comparison graph $\mathcal{G}_{X}$, we use $\mathbb{E}_{\mathrm{s}}$;
for random pairing, where we match item-item-user tuple, we use $\mathbb{E}_{\mathrm{r}}$;
for forming an item-item comparison, i.e., determining whether the
responses $X_{ti}$ and $X_{tj}$ are different for a given item-item-user
tuple $(i,j,t)$, we use $\mathbb{E}_{\mathrm{d}}$; for comparison,
i.e., given $X_{ti}\neq X_{tj}$, whether $X_{ti}<X_{tj}$, we use
$\mathbb{E}_{\mathrm{c}}$. Multiple letters can be combined with
+ sign in the subscript. Let $\mathcal{L}_{k}(\bm{\theta})$ be the
loss function for \myalg\ using $k$-th random splitting and $\mathcal{L}_{k}^{(t)}(\bm{\theta})$
be the sum of the terms corresponding to user $t$, i.e., $\mathcal{L}_{k}(\bm{\theta})\coloneqq\sum_{t=1}^{n}\mathcal{L}_{k}^{(t)}(\bm{\theta})$
and
\[
\mathcal{L}_{k}^{(t)}(\bm{\theta})\coloneqq-\sum_{\substack{(i,j):i>j:\\
(i,j,t)\in\Omega_{k}
}
}\left[\log\left(\frac{e^{\theta_{i}}}{e^{\theta_{i}}+e^{\theta_{j}}}\right)\mathds{1}\{X_{ti}>X_{tj}\}+\log\left(\frac{e^{\theta_{j}}}{e^{\theta_{i}}+e^{\theta_{j}}}\right)\mathds{1}\{X_{ti}<X_{tj}\}\right].
\]
Here $\Omega_{k}$ denotes the set of paired item-item-user tuples
for the $k$-th splitting. 

We assume there is an infinite sequence of users, which has an infinite
sequence of user parameters $\{\zeta_{n}^{\star}\}_{n\in\mathbb{N}^{+}}$,
random sampling $\{A_{it}:i\in[m]\}_{t=1}^{\infty}$, and responses
$\{X_{it}:A_{it}=1\}_{t=1}^{\infty}$. Moreover we assume there are
random splittings $\Omega_{k}$ for $k=1,\ldots,n_{\mathrm{split}}$.
We label the estimators as $\widehat{\bm{\theta}}_{(k)}^{(n)}$ to
denote the \myalgM\ with $k$-th splitting and the random sampling,
responses, and random splittings associated with the first $n$ users.
Similarly we use $\widehat{\bm{\theta}}_{\mathrm{WP}}^{(n)}$ to denote
\myalgW\ with first $n$ users. Moreover, we suppose that the following
limits of average expectation and covariance matrices exist:
\begin{subequations}\label{eq:limits_fix_m}
\begin{align}
\bm{H}^{\infty} & \coloneqq\lim_{n\rightarrow\infty}\frac{1}{n}\sum_{t=1}^{n}\mathbb{E}_{\mathrm{s+r+d+c}}\nabla^{2}\mathcal{L}_{1}^{(t)}(\bm{\theta}^{\star});\label{eq:H_infty}\\
\bm{V}_{\mathrm{same}}^{\infty} & \coloneqq\lim_{n\rightarrow\infty}\frac{1}{n}\sum_{t=1}^{n}\mathbb{E}_{\mathrm{s+r+d+c}}\nabla\mathcal{L}_{1}^{(t)}(\bm{\theta}^{\star})\nabla\mathcal{L}_{1}^{(t)}(\bm{\theta}^{\star})^{\top};\label{eq:V_same_infty}\\
\bm{V}_{\mathrm{diff}}^{\infty} & \coloneqq\lim_{n\rightarrow\infty}\frac{1}{n}\sum_{t=1}^{n}\mathbb{E}_{\mathrm{s+r+d+c}}\nabla\mathcal{L}_{1}^{(t)}(\bm{\theta}^{\star})\nabla\mathcal{L}_{2}^{(t)}(\bm{\theta}^{\star})^{\top}\label{eq:V_diff_infty}
\end{align}
\end{subequations}
Note that by the Cauchy-Schwarz inequality (see
Section~\ref{subsec:Proof_Vsame_Vdiff} for a complete proof), we
have 
\begin{equation}
\bm{V}_{\mathrm{diff}}^{\infty}\preceq\bm{V}_{\mathrm{same}}^{\infty}.\label{eq:Vsame_Vdiff}
\end{equation}
The two sides are the same when $m_{t}=2$ for all $t$.

With this setup of a infinite sequence of estimators, we may study
the asymptotic normality of \myalgM\ as $n$ tends to infinity, accounting
for all randomness in $\mathbb{E}_{\mathrm{s+r+d+c}}$. We have the
following result on \myalgM\ $\widehat{\bm{\theta}}_{\mathrm{MRP}}^{(n)}=(1/n_{\mathrm{split}})\sum_{i=1}^{n_{\mathrm{split}}}\widehat{\bm{\theta}}_{(i)}^{(n)}$.
The analysis is deferred to Section~\ref{subsec:Analysis_asymp_normal}
and the full proof is deferred to Section~\ref{subsec:Proof_asymp_MRPMLE}.

\begin{theorem}\label{thm:asymp_normality}Instate the assumptions
of Theorem~\ref{thm:rasch_infty}. Consider \myalgM\ with $n_{\mathrm{split}}$
random splits with fixed $m$ and $p$. Suppose that the limits in
(\ref{eq:limits_fix_m}) exist. Then as $n\rightarrow\infty$, 
\begin{equation}
\sqrt{n}\left(\widehat{\bm{\theta}}_{\mathrm{MRP}}^{(n)}-\bm{\theta}^{\star}\right)\overset{\mathrm{d}}{\rightarrow}\mathcal{N}\left(\bm{0},\left(\bm{H}^{\infty}\right)^{\dagger}\left[\frac{1}{n_{\mathrm{split}}}\bm{V}_{\mathrm{same}}^{\infty}+\frac{n_{\mathrm{split}}-1}{n_{\mathrm{split}}}\bm{V}_{\mathrm{diff}}^{\infty}\right]\left(\bm{H}^{\infty}\right)^{\dagger}\right).\label{eq:MRPMLE_asymp_normality}
\end{equation}

\end{theorem}Theorem~\ref{thm:asymp_normality} provides an asymptotic
result for inference of \myalgM\  when $m$ and $p$ are fixed and
$n\rightarrow\infty$. In particular, it reveals the decrease in asymptotic
covariance of $\widehat{\bm{\theta}}_{\mathrm{MRP}}^{(n)}$ from $(\bm{H}^{\infty})^{\dagger}\bm{V}_{\mathrm{same}}^{\infty}(\bm{H}^{\infty})^{\dagger}$
to $(\bm{H}^{\infty})^{\dagger}\bm{V}_{\mathrm{same}}^{\infty}(\bm{H}^{\infty})^{\dagger}$
as $n_{\mathrm{split}}$ goes from $1$ to $\infty$. In what follows,
we first connect this result with the asymptotic normality of \myalgW,
which takes all possible item-item pairs with overlap instead of doing
random pairing. Then we quantify the asymptotic variance of \myalgM\
in a special instance to see how much using multiple random splitting
helps.

\paragraph{Asymptotic normality of weighted pseudo MLE.}

Pseudo MLE \citep{zwinderman1995pairwise} is a method for item parameter
estimation similar to our approach \myalg\ and \myalgM. Instead
of random pairing, pseudo MLE forms all possible item-item pairs with
overlaps. For the non-asymptotic analysis, the lack of independence
between comparisons induced by the overlaps is clumsy. However, here
we show that a variant of pseudo MLE is asymptotically normal and
relates to our method \myalgM.

Now we formally introduce the weighted pseudo MLE, denoted as \myalgW.
Let $m_{t}$ be the total number of responses of user $t$ and $\widetilde{m}_{t}$
be the largest even number smaller or equal to $m_{t}$. Consider
the negative log-likelihood functions 
\[
\mathcal{L}_{\mathrm{WP}}^{(t)}(\bm{\theta})\coloneqq-\sum_{\substack{(i,j):i>j:\\
(t,i),(t,j)\in\mathcal{G}_{X}
}
}\frac{\widetilde{m}_{t}}{m_{t}(m_{t}-1)}\left[\log\left(\frac{e^{\theta_{i}}}{e^{\theta_{i}}+e^{\theta_{j}}}\right)\mathds{1}\{X_{ti}>X_{tj}\}+\log\left(\frac{e^{\theta_{j}}}{e^{\theta_{i}}+e^{\theta_{j}}}\right)\mathds{1}\{X_{ti}<X_{tj}\}\right]
\]
and
\[
\mathcal{L}_{\mathrm{WP}}(\bm{\theta})\coloneqq\sum_{t=1}^{n}\mathcal{L}_{\mathrm{WP}}^{(t)}(\bm{\theta}).
\]
Then \myalgW\ is defined as 
\begin{equation}
\widehat{\bm{\theta}}_{\mathrm{WP}}=\arg\min_{\bm{\theta}\in\mathbb{R}^{m},\bm{\theta}^{\top}\bm{1}_{m}=0}\mathcal{L}_{\mathrm{WP}}(\bm{\theta}).\label{eq:WPMLE_def}
\end{equation}
Similar to $\widehat{\bm{\theta}}_{\mathrm{MRP}}^{(n)}$, we also
use the notation $\widehat{\bm{\theta}}_{\mathrm{WP}}^{(n)}$ when
appropriate. The intuition behind this reweighting is to account for
the different numbers of responses for each user. It can be easily
shown that $\mathbb{E}_{\mathrm{r}}\mathcal{L}_{k}(\bm{\theta})=\mathcal{L}_{\mathrm{WP}}(\bm{\theta})$.
Similar to the result for \myalgM, we have the following asymptotic
normality result for \myalgW. The proof is deferred to Section~\ref{subsec:Proof_asymp_WPMLE}.

\begin{theorem}\label{thm:asymp_normality_WPMLE} Instate the assumptions
of Theorem~\ref{thm:rasch_infty}. Consider \myalgW\ with fixed
$m$ and $p$. Suppose that the limits in (\ref{eq:limits_fix_m})
exist. In addition, assume that for every fixed $\bm{\theta}$ obeying
$\|\bm{\theta}-\bm{\theta}^{\star}\|_{\infty}\le10$, the following
limit exists: 
\[
\overline{\mathcal{L}}_{\mathrm{WP}}(\bm{\theta})=\lim_{n\rightarrow\infty}\frac{1}{n}\sum_{t=1}^{n}\mathbb{E}_{\mathrm{s+d+c}}\mathcal{L}_{\mathrm{WP}}^{(t)}(\bm{\theta}).
\]
Then as $n\rightarrow\infty$, 
\begin{equation}
\sqrt{n}\left(\widehat{\bm{\theta}}_{\mathrm{WP}}^{(n)}-\bm{\theta}^{\star}\right)\overset{\mathrm{d}}{\rightarrow}\mathcal{N}\left(\bm{0},\left(\bm{H}^{\infty}\right)^{\dagger}\bm{V}_{\mathrm{diff}}^{\infty}\left(\bm{H}^{\infty}\right)^{\dagger}\right).\label{eq:WPMLE_asymp_normality}
\end{equation}

\end{theorem} This theorem establishes the asymptotic normality of
$\widehat{\bm{\theta}}_{\mathrm{WP}}^{(n)}$. We observe that the
asymptotic covariance of $\widehat{\bm{\theta}}_{\mathrm{WP}}^{(n)}$
is equal to the asymptotic covariance $\widehat{\bm{\theta}}_{\mathrm{MRP}}^{(n)}$
as $n_{\mathrm{split}}$ tends to infinity. This connects \myalgW\
and \myalgM, showing that they are asymptotically equivalent in distribution
when $n\rightarrow\infty$ and $n_{\mathrm{split}}\rightarrow\infty$.

\paragraph{Quantifying the shift of asymptotic covariance in \myalgM. }

We have shown in (\ref{eq:MRPMLE_asymp_normality}) that the asymptotic
covariance of \myalgM\ goes from $(\bm{H}^{\infty})^{\dagger}\bm{V}_{\mathrm{same}}^{\infty}(\bm{H}^{\infty})^{\dagger}$
when $n_{\mathrm{split}}=1$ to $(\bm{H}^{\infty})^{\dagger}\bm{V}_{\mathrm{diff}}^{\infty}(\bm{H}^{\infty})^{\dagger}$
when $n_{\mathrm{split}}\rightarrow\infty$. We have also established
a qualified comparison in (\ref{eq:Vsame_Vdiff}) that shows 
\[
(\bm{H}^{\infty})^{\dagger}\bm{V}_{\mathrm{diff}}^{\infty}(\bm{H}^{\infty})^{\dagger}\preceq(\bm{H}^{\infty})^{\dagger}\bm{V}_{\mathrm{same}}^{\infty}(\bm{H}^{\infty})^{\dagger}.
\]
However, $\bm{V}_{\mathrm{same}}^{\infty}$ and $\bm{V}_{\mathrm{diff}}^{\infty}$
are not explicit. Here we make a quantified illustration in a special
case to better understand this shift in asymptotic covariance. We
make a few simplifications to make it straightforward. First, we suppose
that the user parameters $\zeta_{t}^{\star}$ are independently drawn
from a distribution $\pi$, and we denote the expectation with respect
to the random user parameter with $\mathbb{E}_{\mathrm{u}}$. Second,
we set the item parameter to be $\bm{\theta}^{\star}=\bm{0}_{m}$.
Third, we assume the sampling model where each user response to $mp$
items uniformly at random, for some even integer $mp$. We also need
a constant $\beta$, which is a scalar defined by 
\begin{equation}
\beta\coloneqq\mathbb{E}_{\mathrm{u}}\left[\frac{e^{\zeta_{1}^{\star}}}{(e^{\zeta_{1}^{\star}}+1)^{2}}\right].\label{eq:beta_def}
\end{equation}
Note that $\zeta_{1}^{\star}$ in this definition can be replaced
by $\zeta_{t}^{\star}$ for any $t$. 

In this special setting, we have the following result that quantifies
the shift of asymptotic covariance for different $n_{\mathrm{split}}$.
This proposition is special case for Theorem~\ref{thm:asymp_normality}.
The proof is deferred to Section~\ref{subsec:proof_special_case}.

\begin{proposition}\label{prop:quant_asymp_normality}Instate the
assumptions of Theorem~\ref{thm:rasch_infty}. For fixed $m$ and
$p$, as $n\rightarrow\infty$,
\begin{align}
\sqrt{n}\left(\widehat{\bm{\theta}}_{\mathrm{MRP}}^{(n)}-\bm{\theta}^{\star}\right) & \overset{\mathrm{d}}{\rightarrow}\mathcal{N}\left(\bm{0},\frac{8(m-1)}{\beta mp}\left(\frac{1}{n_{\mathrm{split}}}+\frac{n_{\mathrm{split}}-1}{n_{\mathrm{split}}}\cdot\frac{mp}{2(mp-1)}\right)\left[\bm{I}_{m}-\frac{1}{m}\bm{1}_{m}\bm{1}_{m}^{\top}\right]\right);\label{eq:asymp_covariance_special}\\
\sqrt{n}\left(\widehat{\bm{\theta}}_{\mathrm{WP}}^{(n)}-\bm{\theta}^{\star}\right) & \overset{\mathrm{d}}{\rightarrow}\mathcal{N}\left(\bm{0},\frac{8(m-1)}{\beta mp}\cdot\frac{mp}{2(mp-1)}\left[\bm{I}_{m}-\frac{1}{m}\bm{1}_{m}\bm{1}_{m}^{\top}\right]\right).\nonumber 
\end{align}
  \end{proposition}This proposition shows that in this special instance,
the norm of the asymptotic covariance roughly scales as 
\[
\frac{1}{n_{\mathrm{split}}}+\frac{n_{\mathrm{split}}-1}{n_{\mathrm{split}}}\cdot\frac{mp}{2(mp-1)}.
\]
This equals $1$ when $n_{\mathrm{split}}=1$ and goes to $mp/(2mp-2)$
when $n_{\mathrm{split}}\rightarrow\infty$. The use of multiple random
splitting in \myalgM\ or \myalgW\ can shrink the asymptotic covariance
of the estimators by a factor of $mp/(2mp-2)$.

\section{Experiment}

In this section, we demonstrate the empirical performance of \myalg\
and \myalgM\ using both simulated and real data. 

\subsection{Simulations\label{subsec:Simulation}}

We use simulated data to validate our theoretical results and compare
our estimators with existing ones for the Rasch model. The data generating
process follows the model specified in Section~\ref{subsec:problem_setup}.
Unless specified otherwise, in each trial, the user and item parameters
are randomly drawn from
\[
\widetilde{\bm{\zeta}}^{\star}\sim\mathcal{N}(0,\bm{I}_{n}),\qquad\text{and}\qquad\widetilde{\bm{\theta}}^{\star}\sim\mathcal{N}(0,\bm{I}_{m}).
\]
Afterwards, $\bm{\zeta}^{\star}$ and $\bm{\theta}^{\star}$ is computed
by shifting $\widetilde{\bm{\zeta}}^{\star}$ and $\widetilde{\bm{\theta}}^{\star}$
to zero mean.

\subsubsection{$\ell_{\infty}$ estimation error}

We investigate the $\ell_{\infty}$ estimation error with the following
goals:
\begin{enumerate}
\item We validate the theoretical result in $\ell_{\infty}$ estimation
error of \myalg\ in Theorem~\ref{thm:rasch_infty}. 
\item We show how much advantage the \myalgM\ brings through multiple runs
of data splitting. 
\item We compare our methods with existing comparison-based algorithms,
including the case where $\kappa_{1},\kappa_{2}$ are large.
\end{enumerate}

\paragraph{Validating the theoretical result.}

Theorem~\ref{thm:rasch_infty} tells us that the $\ell_{\infty}$
error scales as $1/\sqrt{np}$. Figure~\ref{fig:err_vs_np} shows
that $\|\widehat{\bm{\theta}}-\bm{\theta}^{\star}\|_{\infty}$ exhibits
a near-linear relationship with respect to both $1/\sqrt{n}$ and
$1/\sqrt{p}$, which is consistent with our theoretical predictions.

\begin{figure}
\begin{centering}
\subfloat[$\|\widehat{\bm{\theta}}-\bm{\theta}^{\star}\|_{\infty}$ v.s. $1/\sqrt{n}$.
The parameter is chosen to be $m=50,p=0.1$ and $n$ varies from $10000$
to $40000$.]{\includegraphics[width=0.45\textwidth]{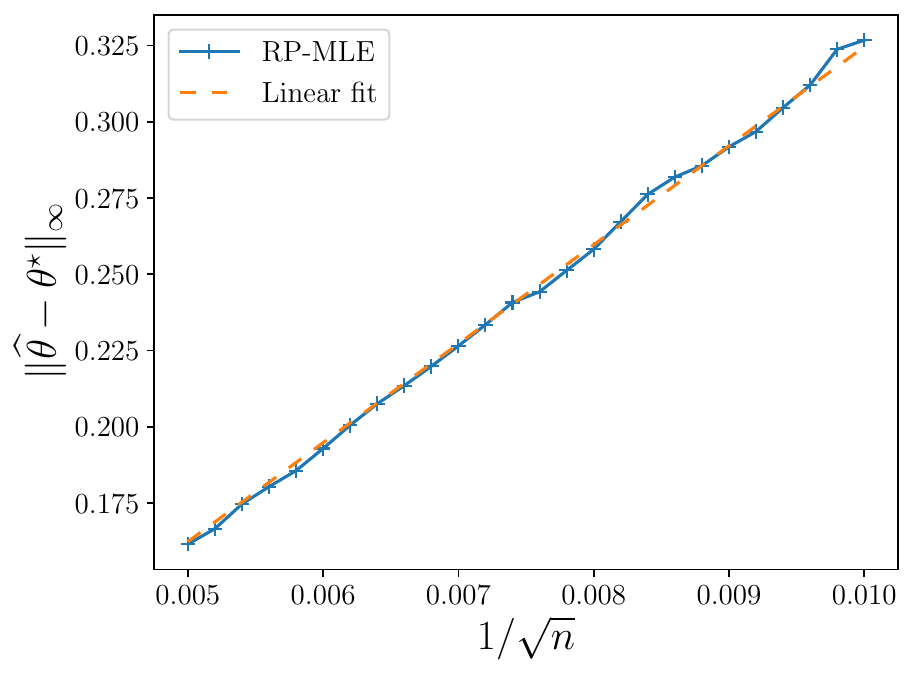}}\qquad{}\subfloat[$\|\widehat{\bm{\theta}}-\bm{\theta}^{\star}\|_{\infty}$ v.s. $1/\sqrt{p}$.
The parameter is chosen to be $m=50,n=10000$ and $p$ varies from
$1/9$ to $1$. ]{\includegraphics[width=0.45\textwidth]{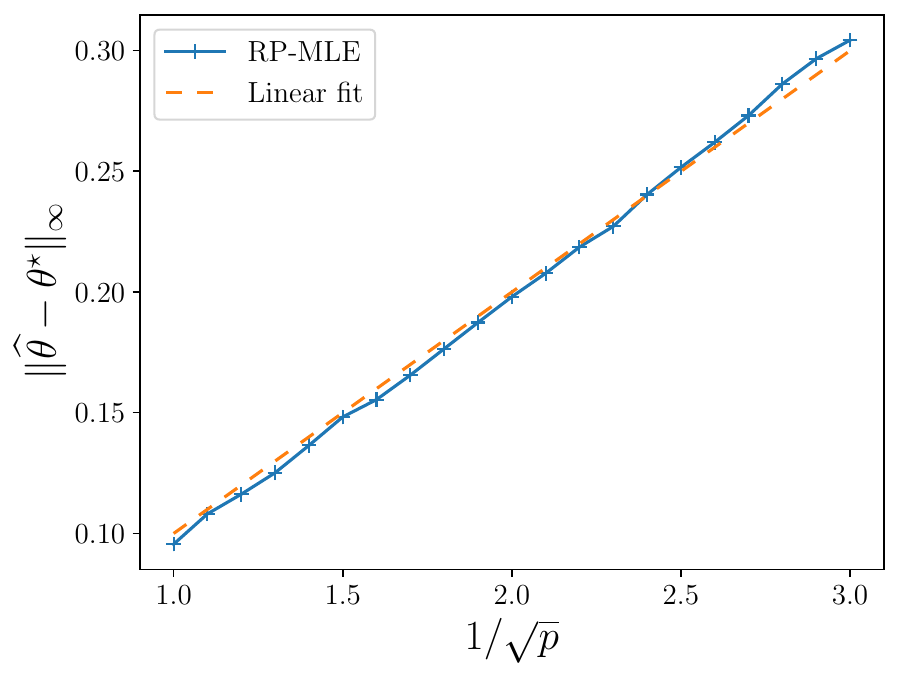}}
\par\end{centering}
\centering{}\caption{\label{fig:err_vs_np}Estimation error $\|\widehat{\bm{\theta}}-\bm{\theta}^{\star}\|_{\infty}$
of \myalg\ with varying $n$ and $p$. Each point represents the
average of 1000 trials.}
\end{figure}

\paragraph{Multiple runs in \myalgM.}

As we have discussed after Theorem~\ref{thm:rasch_infty}, the random
data splitting could incur a small loss of information. We have introduced
a remedy \myalgM\ (Algorithm~\ref{alg:MRP-MLE}) to address this
by averaging over multiple runs with independent data splitting. Moreover,
in Proposition~\ref{prop:quant_asymp_normality} we have a quantitative
characterization of the improvement in $\ell_{2}$ error achieved
through multiple data splittings. 

Figure~\ref{fig:multirun_infty} shows that by averaging over more
runs of data splittings, \myalgM\ achieves an improved $\ell_{\infty}$
estimation error that improves over PMLE and is close to \myalgW.
In addition, we observe in Figure~\ref{fig:multirun_l2} that the
improvement in squared $\ell_{2}$ error scales linearly with $1/n_{\mathrm{split}}$,
consistent with the theoretical findings in Proposition~\ref{prop:quant_asymp_normality}.

\begin{figure}
\begin{centering}
\subfloat[\label{fig:multirun_infty}The $\ell_{\infty}$ estimation error of
\myalgM\ v.s. $n_{\mathrm{split}}$. The dash-dotted and dashed lines
are the performance of PMLE and \myalgW, respectively. ]{\includegraphics[width=0.45\textwidth]{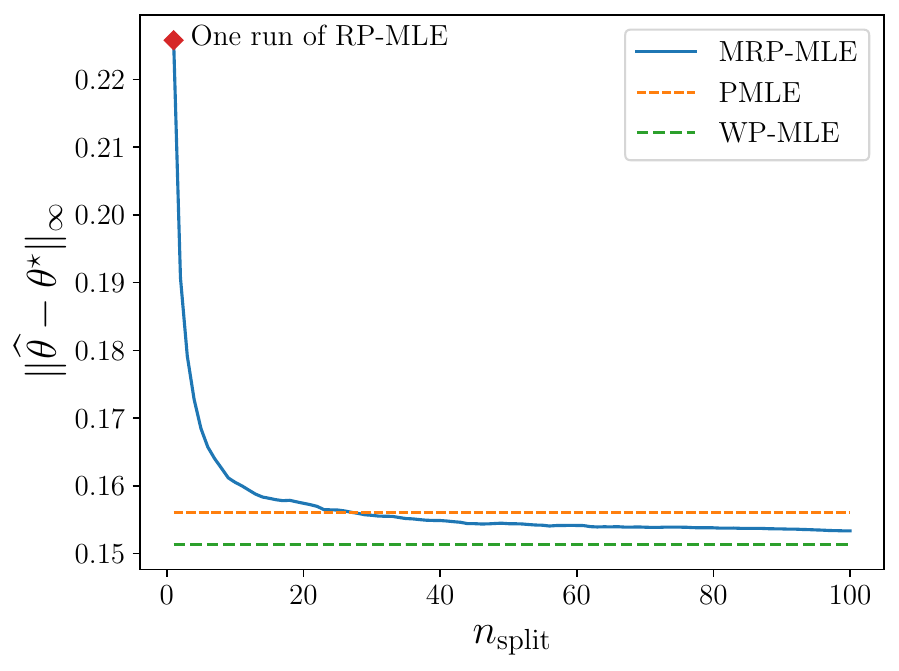}}\qquad{}\subfloat[\label{fig:multirun_l2} The squared error $\|\widehat{\bm{\theta}}-\bm{\theta}^{\star}\|^{2}$
of \myalgM\ v.s. $1/n_{\mathrm{split}}$ with varying $n_{\mathrm{split}}$.
The dash-dotted and dashed lines are the performance of PMLE and \myalgW,
respectively. ]{\includegraphics[width=0.45\textwidth]{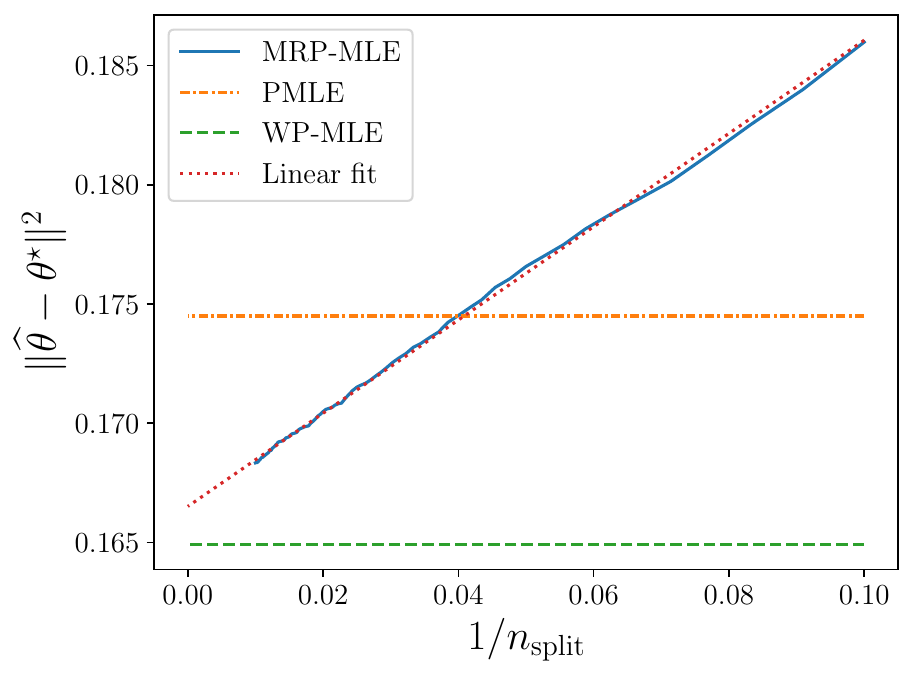}}
\par\end{centering}
\centering{}\caption{\label{fig:multirun}Estimation error of \myalgM\ with varying number
of data splittings. For each trial, we record $\|\frac{1}{k}\sum_{i=1}^{k}\widehat{\bm{\theta}}_{(i)}-\bm{\theta}^{\star}\|$
for $k=1,\ldots,100$. The parameters are chosen to be $m=50,p=0.2,n=10000$.
The latent scores are all 0 and the each user is assigned with $mp$
item uniformly-at-random. Each point is averaged over 1000 trials.}
\end{figure}

\begin{figure}
\centering{}
\end{figure}

\paragraph{Comparison with existing estimators.}

We compare our algorithms with two other comparison-based algorithms:
the pseudo MLE (PMLE) and the spectral method from \cite{nguyen2023optimal}.
In Figure~\ref{fig:err_diff_method}, We can see that the performance
of \myalgM\ is comparable to PMLE and slightly outperforms the spectral
method. Our proposed algorithm not only offers stronger theoretical
guarantees but also demonstrates competitive practical performance.

\paragraph{Performance with large $\kappa_{1},\kappa_{2}$.}

While we assume $\kappa=\max\{\kappa_{1},\kappa_{2}\}=O(1)$ is most
of this article, scenarios with large $\kappa$ can be practically
relevant. To evaluate the performance in such cases, we compare the
$\ell_{\infty}$ error of different methods under different condition
numbers. For a fixed $\kappa$, we draw the user and item parameters
as
\[
\widetilde{\bm{\zeta}}^{\star}\sim\mathrm{Unif}(0,\log(\kappa))\qquad\text{and}\qquad\widetilde{\bm{\theta}}^{\star}\sim\mathrm{Unif}(0,\log(\kappa))
\]
and compute $\bm{\zeta}^{\star}$ and $\bm{\theta}^{\star}$ by shifting
$\widetilde{\bm{\zeta}}^{\star}$ and $\widetilde{\bm{\theta}}^{\star}$
to have zero mean. Figure~\ref{fig:err_kappa} illustrates the performance
of different estimators as $\kappa$ varies. The MLE-based approaches
including \myalg\ and \myalgM\ achieve better $\ell_{\infty}$ error
than the spectral method when $\kappa$ is large.

\begin{figure}
\begin{minipage}[t]{0.45\textwidth}%
\begin{center}
\includegraphics[width=1\textwidth]{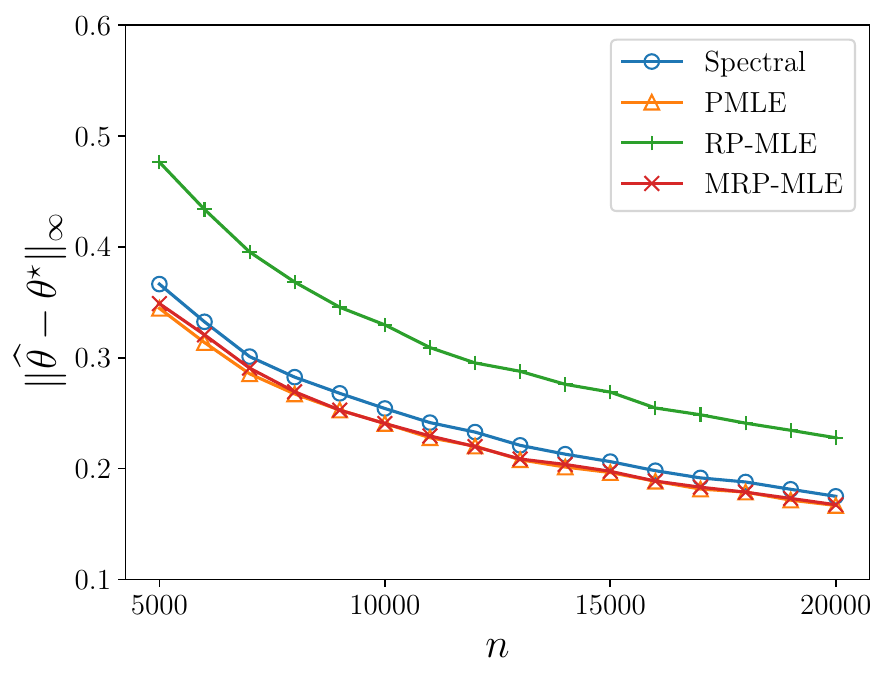}
\par\end{center}
\caption{\label{fig:err_diff_method} $\|\widehat{\bm{\theta}}-\bm{\theta}^{\star}\|_{\infty}$
v.s. $n$ using Spectral method, PMLE, \myalg, and \myalgM\ using
20 data splittings. The parameter is chosen to be $m=50,p=0.1$ and
$n$ varies from $5000$ to $20000$. The result is averaged over
1000 trials.}
\end{minipage}\hfill{}%
\begin{minipage}[t]{0.45\textwidth}%
\begin{center}
\includegraphics[width=1\textwidth]{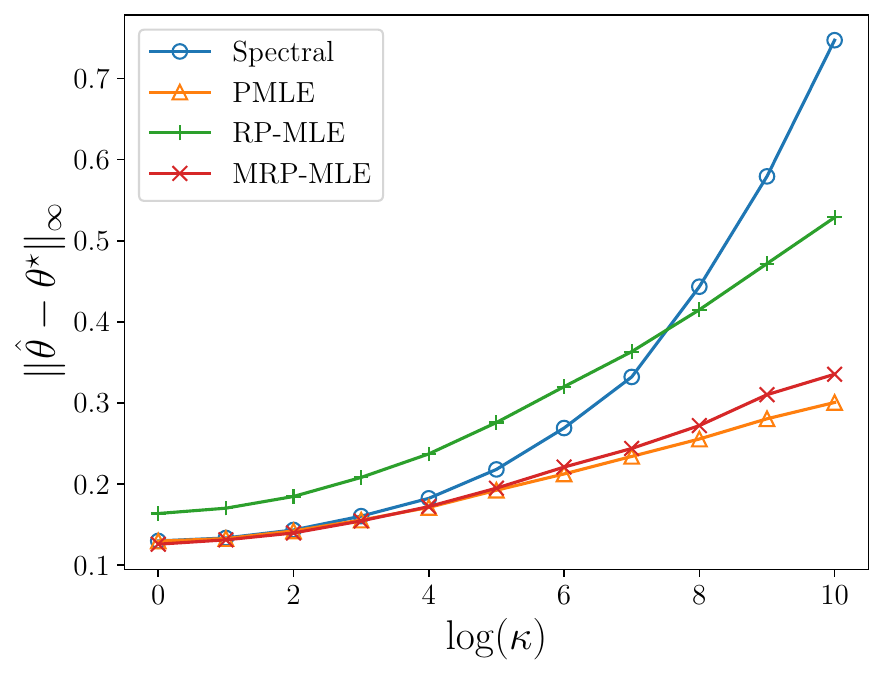}
\par\end{center}
\caption{\label{fig:err_kappa} $\|\widehat{\bm{\theta}}-\bm{\theta}^{\star}\|_{\infty}$
v.s. $\log(\kappa)$ using Spectral method, PMLE, \myalg, and \myalgM\
using 20 data splittings. The parameter is chosen to be $m=50,p=0.1,n=20000$
and $\kappa$ varies from $1$ to $e^{10}$. The result is averaged
over 1000 trials.}
\end{minipage}
\end{figure}

\subsubsection{Top-$K$ recovery}

We investigate the performance of different algorithms in top-$K$
recovery. Set $\theta_{i}^{\star}=(1-K/m)\Delta_{K}$ for $i\le K$
and $\theta_{i}^{\star}=(-K/m)\Delta_{K}$ otherwise. For any estimator
$\widehat{\bm{\theta}}$, we define top-$K$ recovery rate to be 
\[
\frac{1}{K}\left|\left\{ i\le K:i\in A_{K}\right\} \right|,
\]
where $A_{K}$ is an arbitrary $K$-element set such that $\widehat{\theta}_{i}\ge\widehat{\theta}_{j}$
for any $i\in A_{K},j\notin A_{K}$. We compare the top-$K$ recovery
rate of PMLE and the spectral method in \cite{nguyen2023optimal}
with \myalg\ and \myalgM\ in Figure~\ref{fig:TopK}. The recovery
rate of PMLE, spectral method and \myalgM\ is similar, indicating
again that our algorithm performs well in practice.

\begin{figure}
\centering{}%
\begin{minipage}[t]{0.45\columnwidth}%
\begin{center}
\includegraphics[width=1\textwidth]{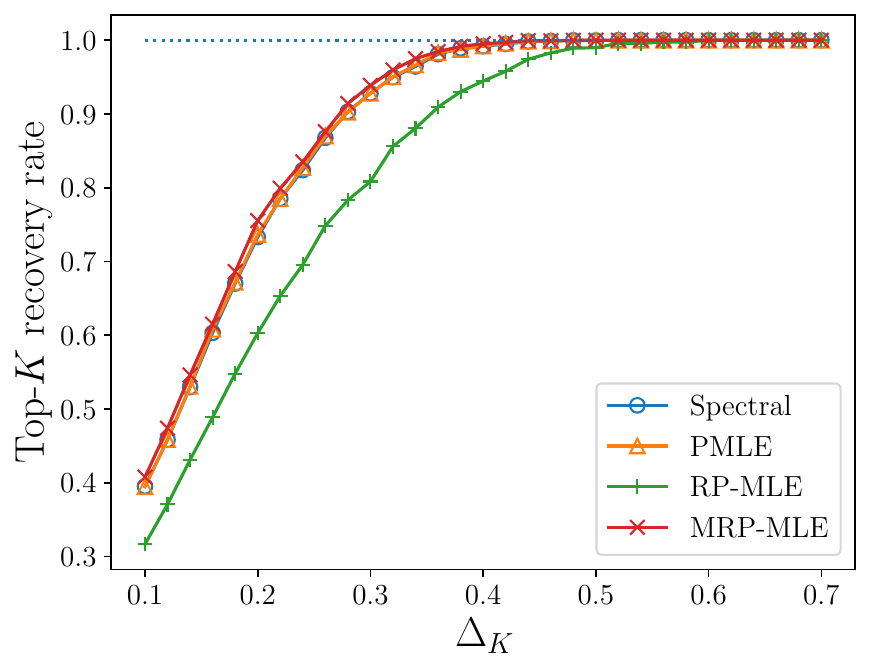}
\par\end{center}
\caption{\label{fig:TopK}Top-$K$ recovery rate using spectral method, PMLE,
\myalg\, and \myalgM\ using 20 data splittings. The parameter is
chosen to be $m=10000,m=50,p=0.1,K=5$ and $\Delta_{K}$ varies from
0.1 to 0.7. The result is averaged over 1000 trials.}
\end{minipage}\hfill{}%
\begin{minipage}[t]{0.45\columnwidth}%
\begin{center}
\includegraphics[width=1\textwidth]{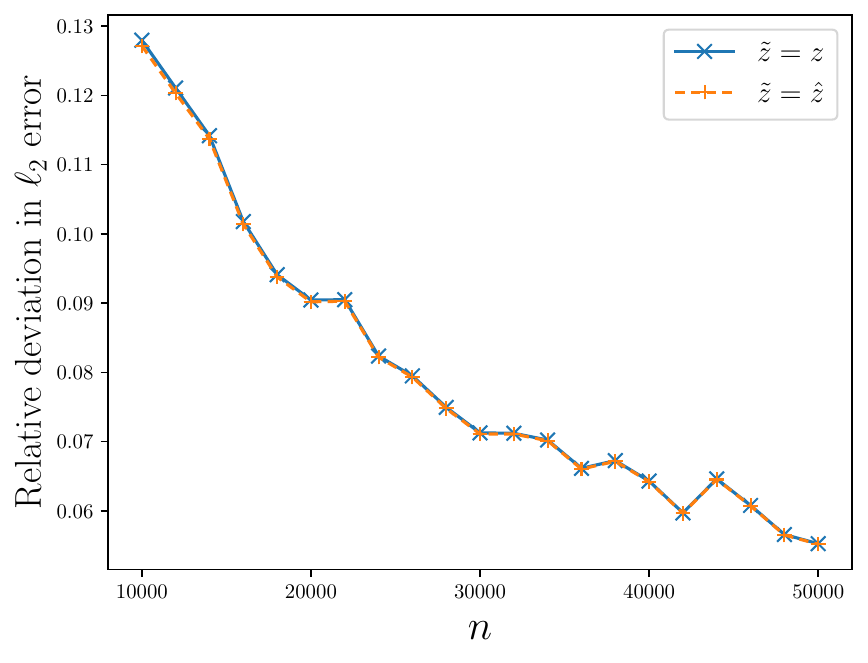}
\par\end{center}
\caption{\label{fig:refined_l2} The relative deviation of $\|\widehat{\bm{\theta}}-\bm{\theta}^{\star}\|$
from $\sqrt{\mathrm{Trace}(\bm{L}_{L\tilde{z}}^{\dagger})}$ v.s.
$n$ for both $\widetilde{z}=\widehat{z}$ and $\widetilde{z}=z$.
The parameter is chosen to be $p=0.1$, $n$ varies from $10000$
to $50000$ and $m=n/500$. The result is averaged over 1000 trials.}
\end{minipage}
\end{figure}

\subsubsection{Refined $\ell_{2}$ estimation error}

In Theorem~\ref{prop:rasch_l2} we have shown that the $\ell_{2}$
error concentrate around $\sqrt{\mathrm{Trace}(\bm{L}_{L\tilde{z}}^{\dagger})}$
for $\widetilde{z}\in\left\{ \widehat{z},z\right\} $. In each trial
we compute the following quantity
\[
\frac{\left|\|\widehat{\bm{\theta}}-\bm{\theta}^{\star}\|-\sqrt{\mathrm{Trace}(\bm{L}_{L\tilde{z}}^{\dagger})}\right|}{\sqrt{\mathrm{Trace}(\bm{L}_{L\tilde{z}}^{\dagger})}}
\]
for both $\widetilde{z}=\widehat{z}$ and $\widetilde{z}=z$. This
measures the relative deviation of $\|\widehat{\bm{\theta}}-\bm{\theta}^{\star}\|$
from $\sqrt{\mathrm{Trace}(\bm{L}_{L\tilde{z}}^{\dagger})}$. In Figure~\ref{fig:refined_l2}
we consider the regime where $p$ and $n/m$ is fixed. In this case,
Theorem~\ref{prop:rasch_l2} implies that 
\[
\frac{\left|\|\widehat{\bm{\theta}}-\bm{\theta}^{\star}\|-\sqrt{\mathrm{Trace}(\bm{L}_{L\tilde{z}}^{\dagger})}\right|}{\sqrt{\mathrm{Trace}(\bm{L}_{L\tilde{z}}^{\dagger})}}\lesssim\frac{\sqrt{\frac{1}{np}}+\frac{\sqrt{m}}{np}}{\sqrt{\frac{m}{np}}}\lesssim\frac{1}{\sqrt{n}}.
\]
In other words, $\|\widehat{\bm{\theta}}-\bm{\theta}^{\star}\|$ concentrate
tightly around $\mathrm{Trace}(\bm{L}_{L\tilde{z}}^{\dagger})$. We
can see that the deviation is very small between $\widetilde{z}=\widehat{z}$
and $\widetilde{z}=z$. In both cases, the relative deviation of $\|\widehat{\bm{\theta}}-\bm{\theta}^{\star}\|$
from $\sqrt{\mathrm{Trace}(\bm{L}_{L\tilde{z}}^{\dagger})}$ decreases
as $n$ and $m$ increase as expected.

\subsubsection{Confidence intervals for \myalgM\ and \myalgW\label{subsec:CI}}

\begin{figure}
\begin{centering}
\includegraphics[width=0.5\textwidth]{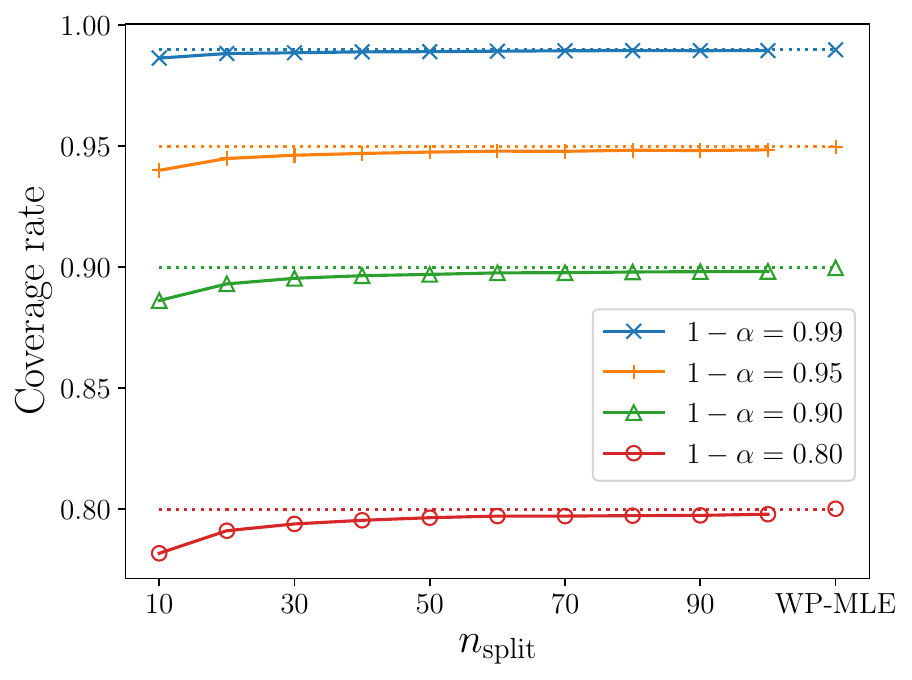}\caption{\label{fig:CI}Theoretical coverage rate $1-\alpha$ (dashed line)
and empirical coverage rate $1-\widehat{\alpha}$ of the two-sided
confidence intervals $[\mathcal{C}_{i}^{-}(\alpha/2),\mathcal{C}_{i}^{+}(\alpha/2)]$.
The first 10 points on the left in each level are computed with \myalgM\
with $n_{\mathrm{split}}$ vaying from 10 to 100. The rightmost point
on each level is computed with \myalgW. The parameters are set to
be $n=10000,m=20,p=0.5$. The confidence interval is $1-\alpha=0.8,0.9,0.95,0.99$.
Each point is averaged over $10000$ trials.}
\par\end{centering}
\end{figure}
The asymptotic normality results in Theorem~\ref{thm:asymp_normality}
and \ref{thm:asymp_normality_WPMLE} allow us to construct confidence
intervals for $\theta_{i}^{\star}$ with \myalgM\ and \myalgW. For
large enough $n_{\mathrm{split}}$, we can approximate the variance
of both estimators with 
$\frac{1}{n}(\bm{H}^{\infty})^{\dagger}\bm{V}_{\mathrm{diff}}^{\infty}(\bm{H}^{\infty})^{\dagger}$, 
where $\bm{V}_{\mathrm{diff}}^{\infty}$ and $\bm{H}^{\infty}$ are
estimated using the following plug-in estimates:
\begin{align*}
\widehat{\bm{H}}^{\infty} & =\frac{1}{n\cdot n_{\mathrm{split}}}\sum_{i=1}^{n_{\mathrm{split}}}\sum_{t=1}^{n}\nabla^{2}\mathcal{L}_{i}^{(t)}(\widehat{\bm{\theta}}),\\
\widehat{\bm{V}}_{\mathrm{diff}}^{\infty} & =\frac{1}{n}\sum_{t=1}^{n}\nabla\mathcal{L}_{\mathrm{WP}}^{(t)}(\widehat{\bm{\theta}})\nabla\mathcal{L}_{\mathrm{WP}}^{(t)}(\widehat{\bm{\theta}})^{\top}.
\end{align*}
For a given confidence level $1-\alpha$, the confidence interval
$[\mathcal{C}_{i}^{-}(\alpha/2),\mathcal{C}_{i}^{+}(\alpha/2)]$ is
\begin{align*}
\mathcal{C}_{i}^{-}(\alpha/2) & \coloneqq\widehat{\theta}_{i}-z_{1-\alpha/2}\cdot\left[\frac{1}{n}(\widehat{\bm{H}}^{\infty})^{\dagger}\widehat{\bm{V}}_{\mathrm{diff}}^{\infty}(\widehat{\bm{H}}^{\infty})^{\dagger}\right]_{ii}^{1/2},\\
\mathcal{C}_{i}^{+}(\alpha/2) & \coloneqq\widehat{\theta}_{i}+z_{1-\alpha/2}\cdot\left[\frac{1}{n}(\widehat{\bm{H}}^{\infty})^{\dagger}\widehat{\bm{V}}_{\mathrm{diff}}^{\infty}(\widehat{\bm{H}}^{\infty})^{\dagger}\right]_{ii}^{1/2}.
\end{align*}
In Figure~\ref{fig:CI}, we compare the empirical coverage rate
\[
1-\widehat{\alpha}\coloneqq\frac{1}{m}\sum_{i=1}^{m}\mathds{1}\{\mathcal{C}_{i}^{-}(\alpha/2)\le\theta_{i}^{\star}\le\mathcal{C}_{i}^{+}(\alpha/2)\}.
\]
of these two-sided confidence intervals with the theoretical ones.
We can see that when $n_{\mathrm{split}}$ is reasonably large, the
empirical coverage rate matches well with the theoretical coverage
rate for both \myalgM\ and \myalgW.

\subsection{LSAT dataset}

We study a real-world dataset (LSAT) on the Law School Admissions
Test from \cite{Bock1970}. LSAT has full observation of 1000 people
answering 5 problems, with each person-item pair recording whether
the answer was correct. The second row in Table~\ref{tab:lsat} lists
how many people answer each problem correctly. From the first look,
Problem 3 appears to be the hardest question. 

We proceed to quantify the hardness of these problems under the Rasch
model and infer how confident we are in claiming it is the hardest.
Using \myalgM\, we compute a latent score estimate and construct
two-sided confidence intervals at significance level $\alpha=0.01$
for each coordinate, following the methodology introduced in Section~\ref{subsec:CI}.
The result is summarized in Table~\ref{tab:lsat}, where higher latent
score correspond to greater difficulty. The estimated parameters align
inversely with the total number of correct answers. Notably, the lower
bound of the confidence interval for $\theta_{3}^{\star}$ is larger
than the upper bounds of the confidence intervals of $\theta_{1}^{\star},\theta_{2}^{\star},\theta_{4}^{\star}$,
and $\theta_{5}^{\star}$. With Bonferroni correction, we can conclude
that with 95\% confidence Problem 3 is the most difficult problem
in this dataset.

\begin{table}
\begin{centering}
\par\end{centering}
\begin{centering}
\begin{tabular}{|c|c|c|c|c|c|}
\hline 
Problem & 1 & 2 & 3 & 4 & 5\tabularnewline
\hline 
Total correct & 924 & 709 & 553 & 763 & 870\tabularnewline
$\theta$ estimate & -1.2824  & 0.4511 & 1.2800 & 0.1926 & -0.6413\tabularnewline
CI lower bound & -1.5579  & 0.2696 & 1.0958 & 0.0017 & -0.8711\tabularnewline
CI upper bound & -1.0069 & 0.6327 & 1.4641 & 0.3834 & -0.4116\tabularnewline
\hline 
\end{tabular}
\par\end{centering}
\caption{\label{tab:lsat} Latent score estimate calculated using \myalgM\
with 20 data splittings and confidence interval calculated with the
construction introduced in Section~\ref{subsec:CI}. Higher latent
score here means higher difficulty. The significance level is chosen
to be $\alpha=0.01$ for each coordinate. }
\end{table}

Now we assume Problem 3 is the top-$1$ item in latent score and investigate
the top-$1$ recovery rate of different algorithms on LSAT under incomplete
observation. In each trial we randomly select $\widetilde{n}$ people,
and for each of them we randomly select their outcome on $\widetilde{m}$
problems. We then estimate $\bm{\theta}^{\star}$ using the subsampled
data with different methods and compare the proportion of trials where
the top-$1$ item is correctly identified. In Figure~\ref{fig:Top_1_lsat},
we see results similar to the simulation. Our algorithm \myalgM\
has a similar recovery rate compared to PMLE and spectral method.
\begin{figure} 
\begin{centering}
\includegraphics[width=0.5\textwidth]{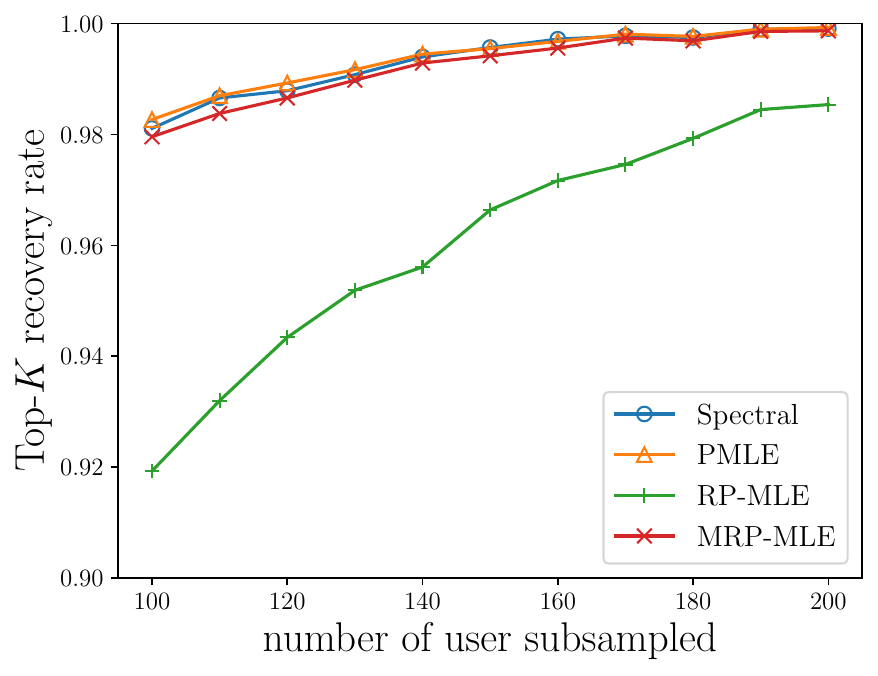}
\par\end{centering}
\caption{\label{fig:Top_1_lsat}Top-1 recovery rate using Spectral method,
PMLE, \myalg\, and \myalgM\ using 20 data splittings. The parameters
are chosen to be $\widetilde{m}=4$ and $\widetilde{n}$ varies from
100 to 200. The result is averaged over 10000 trials.}

\end{figure}

\section{Discussion}

In this paper, we propose two new likelihood-based estimators \myalg\
and \myalgM\ for item parameter estimation in the Rasch model. Both
enjoy optimal finite sample estimation guarantee and asymptotic normality
that allows for tight uncertainty quantification. All this is achieved
even when the user-item response data are extremely sparse (cf.~\cite{nguyen2023optimal}).
Below, we identify several questions that are interesting for further
investigation: 
\begin{itemize}
\item \textbf{Does PMLE or CMLE achieve optimal theoretical guarantee?}\emph{
}In our experiments, pseudo MLE has shown a similar performance to
\myalgM. This naturally leads to the question of whether PMLE can
enjoy the same theoretical guarantee. This is relevant to our work
because our methods can be viewed as a modification of pseudo MLE
by incorporating random disjoint pairing to decouple statistical dependency
among paired $Y_{ij}$'s. It remains unclear whether such dependency
is a fundamental bottleneck. On the other hand, conditional MLE is
another popular method used in practice. In Section~A.1
of the supplement we will mention that CMLE can also be viewed as a reduction to a less studied item-only
model. This reduction is more complicated for analysis as it constructs
item tuples rather than item pairs. It would be interesting to know
whether we can transfer the techniques we have used here to develop
a non-asymptotic analysis for CMLE. 
\item \textbf{Extending random pairing to other models in IRT.} Some IRT
models parameterize the latent score of users and items differently
from the Rasch model. For instance, consider the two-parameter logistic
model (2PL) with discrimination parameter on the users. It assumes
that $X_{ti}$, the response of user $t$ to item $i$, follows the
law
\[
\mathbb{P}[X_{ti}=1]=\frac{1}{1+\exp(a_{t}^{\star}(\zeta_{t}^{\star}-\theta_{i}^{\star}))},
\]
where $\bm{\theta}^{\star}$ is the latent scores of the items , $\bm{\zeta}^{\star}$
is the latent scores of the users, while $\bm{a}^{\star}$ is the
discrimination parameters. Unlike the Rasch model, in the 2PL model,
\[
\mathbb{P}[X_{ti}>X_{tj}\mid X_{ti}\neq X_{tj}]=\frac{\exp(a_{t}^{\star}\theta_{i}^{\star})}{\exp(a_{t}^{\star}\theta_{i}^{\star})+\exp(a_{t}^{\star}\theta_{j}^{\star})}
\]
is not independent of the user discrimination parameter $\bm{a}^{\star}$.
Therefore the reduction to the BTL model is no longer true in this
case. However, one could employ a partially-Bayesian approach by putting
a prior on $\bm{\alpha}^{\star}$ and maximize this marginally likelihood,
which is a function of $\theta_{i}^{\star}-\theta_{j}^{\star}$ independent
of $\bm{\zeta}$. It is interesting and non-trivial to extend the
idea of random pairing to the 2PL model. 
\item \textbf{Extension to joint estimation of user and item parameters.
}It is sometimes of interest to estimate both the user and the item
parameters. We expect our method \myalgM\ continues to work with
slight modifications. In a high level, the idea is to estimate the
mean-shifted parameters $\bm{\theta}^\star - (1/m) \bm{1}_m\bm{1}_m^\top \bm{\theta}^\star$
and $\bm{\zeta}^\star - (1/m) \bm{1}_m\bm{1}_m^\top \bm{\zeta}^\star$
using \myalgM\ twice. In the end, one estimates the difference in
the means using MLE over the comparison outcomes. We leave the detailed
investigation to future work. 
\end{itemize}

\bibliographystyle{alpha}
\bibliography{All-of-Bibs}

\appendix

\section{Analysis\label{sec:Analysis}}

In this section, we present the main steps to obtain theoretical results
in the previous section. Section~\ref{subsec:Reduction} provides
a complete argument on the reduction to the BTL model we mentioned
in Section~\ref{subsec:algo}. Section~\ref{subsec:Analysis_infty}
provides the analysis of the $\ell_{\infty}$ error, Section~\ref{subsec:Analysis_non_asymp_expansion}
provides the analysis of the non-asymptotic expansion, and Section~\ref{subsec:Analysis_asymp_normal}
sketches the proof of the asymptotic normality of~\myalgM\ and \myalgW.

\subsection{Reduction to Bradley-Terry-Luce model\label{subsec:Reduction}}

  A key component in \myalg\ is the random pairing in Steps 1 and
2 of Algorithm~\ref{alg:RP-MLE}. It compiles the user-item responses
$\bm{X}$ to item-item comparisons $\bm{Y}$. In this section, we
make a detailed argument that $\bm{Y}$ follows the Bradley-Terry-Luce
model with a non-uniform sampling scheme. 

Recall that $L_{ij}^{t}\coloneqq\mathds{1}\{X_{ti}\neq X_{tj}\}$
and $Y_{ij}^{t}\coloneqq\mathds{1}\{X_{ti}<X_{tj}\}$. The following
fact provides the distribution of $Y_{ij}^{t}$ conditional on $L_{ij}^{t}=1$.
We defer its proof to Section~\ref{subsec:Proof_BTL}.

\begin{fact}\label{fact:BTL} Let $i,j$ be two items and $t$ be
a user. Suppose that user $t$ has responded to both items $i$ and
$j$. Let $X_{ti}$ and $X_{tj}$ be the responses sampled from the
probability model (\ref{eq:sampling_prob}). Then we have
\[
\mathbb{P}[X_{ti}<X_{tj}\mid L_{ij}^{t}=1]=\frac{e^{\theta_{j}^{\star}}}{e^{\theta_{i}^{\star}}+e^{\theta_{j}^{\star}}},
\]
\begin{equation}
\mathbb{P}[L_{ij}^{t}=1]\ge\frac{2\kappa_{2}}{(1+\kappa_{2})^{2}}.\label{eq:prob_L_ij}
\end{equation}

\end{fact}

Fact~\ref{fact:BTL} shows that conditional on $L_{ij}^{t}=1$, $Y_{ij}^{t}$
follows the BTL model with parameters $\bm{\theta}^{\star}$. More
importantly, as we deploy random pairing (cf.~Step 1a), each response
$X_{ti}$ is used at most once. As a result, conditional on $\{L_{ij}^{t}\}_{ijt}$,
$Y_{ij}^{t}$'s are jointly independent across users and items. In
light of these, we can equivalently describe the data generating process
of $\bm{Y}$ as follows:
\begin{enumerate}
\item For each user-item pair $(t,i)$, there is a comparison between them
with probability $p$ independently.
\item Randomly split the $m_{t}$ problems taken by user $t$ into $\lfloor m_{t}/2\rfloor$
pairs of problems. (Step 1(a) of Algorithm~\ref{alg:RP-MLE})
\item For all $(i,j,t)$, items $i$ and $j$ are compared by user $t$
if $L_{ij}^{t}\coloneqq\mathds{1}\{X_{ti}\neq X_{tj}\}=1$. 
\item Conditioned on $L_{ij}^{t}=1$, one observes the outcome $Y_{ij}^{t}\coloneqq\mathds{1}\{X_{ti}<X_{tj}\}$.
\end{enumerate}
Steps~1--3 generates a non-uniform comparison graph $\mathcal{E}_{Y}$
between items. Step 4 reveals the independent outcomes of these comparisons
following the BTL model, conditional on the graph $\mathcal{E}_{Y}$.
This justifies that (\ref{eq:MLE_loss}) is truly the likelihood function
of the BTL model conditional on the comparison graph $\mathcal{E}_{Y}$. 

In addition, we would like to comment on another popular method conditional
MLE, which can also be viewed as a reduction to a item-only model.
The CMLE maximizes the likelihood conditioned on total number of positive
responses. It can be computed that
\[
\mathbb{P}\left[X_{ti_{1}},X_{ti_{2}},\cdots,X_{ti_{m_{t}}}\mid\sum_{l=1}^{m_{t}}X_{ti_{l}}=k\right]=\frac{\prod_{l=1}^{m_{t}}e^{\theta_{i_{l}}^{\star}}\mathds{1}\{X_{ti_{l}}=1\}}{\sum_{\bm{\alpha}\in\{0,1\}^{m_{t}}}\prod_{l=1}^{m_{t}}e^{\theta_{i_{l}}^{\star}}\mathds{1}\{\alpha_{l}=1\}}.
\]
It is easy to see that the conditional probability is not dependent
on the user parameters $\bm{\zeta}^{\star}$. However, the model CMLE
reduces to is less studied than the BTL model, especially in the setting
of sparse observations. 

\subsection{Analysis for entrywise error bound \label{subsec:Analysis_infty}}

We have seen that analyzing \myalg\ under the Rasch model can be
reduced to analyzing the MLE under the BTL model. This reduction allows
us to invoke the result 
in the recent work~\cite{Yang2024} established for MLE in the BTL model with a general comparison graph. 

To facilitate the presentation, we introduce the necessary notation.
For any $i\in[m]$, let $d_{i}\coloneqq\sum_{j:j\neq i}L_{ij}$ be
the weighted degree of item $i$ in $\mathcal{G}_{Y}$ and $d_{\max}=\max_{i\in[m]}d_{i}$.
Let the weighted graph Laplacian $\bm{L}_{L}$ be
\[
\bm{L}_{L}\coloneqq\sum_{i,j:i>j}L_{ij}(\bm{e}_{i}-\bm{e}_{j})(\bm{e}_{i}-\bm{e}_{j})^{\top}.
\]
The following lemma adapts Theorem~3 of the recent work \cite{Yang2024}
to our setting.

\begin{lemma}[Theorem~3 in \cite{Yang2024}]\label{lemma:MLE_general}
Assume that $\mathcal{G}_{Y}$ is connected, and that 
\begin{equation}
[\lambda_{m-1}(\bm{L}_{L})]^{5}\ge C_{1}\kappa_{1}^{4}(d_{\max})^{4}\log^{2}(n)\label{eq:MLE_general_cond}
\end{equation}
for some large enough constant $C_{1}>0$. Then with probability at
least $1-n^{-10}$, we have
\[
\|\widehat{\bm{\theta}}-\bm{\theta}^{\star}\|_{\infty}\le C_{2}\kappa_{1}\sqrt{\frac{\log(n)}{\lambda_{m-1}(\bm{L}_{L})}}
\]
for some constant $C_{2}>0$.

\end{lemma}

To leverage this general result, we need to characterize the spectral
and degree properties of the comparison graph $\mathcal{G}_{Y}$,
which is achieved in the following two lemmas. The proofs are deferred
to Section~\ref{sec:Proofs_deg_spec}.

\begin{lemma}[Degree bound in $\mathcal{G}_{Y}$]\label{lem:degree_Y}
Suppose that $np\ge C\kappa_{2}^{2}\log(n)$ for some large enough
constant $C>0$ and $m\le n^{\alpha}$ for some sufficiently large
constant $\alpha>0$. With probability at least $1-2n^{-10}$, for
all $i\in[m]$, 
\begin{equation}
\frac{1}{24\kappa_{2}}np\le d_{i}\le\frac{3}{2}np.\label{eq:deg_Y}
\end{equation}

\end{lemma}

\begin{lemma}\label{lemma:spectral} Suppose $mp\ge2$, $np\ge C\kappa_{2}^{2}\log(n)$
for some large enough constant $C$, and $m\le n^{\alpha}$ for some
constant $\alpha>0$. With probability at least $1-10n^{-10}$, we
have 
\begin{equation}
\frac{np}{4\kappa_{2}}\le\lambda_{m-1}(\bm{L}_{L})\le\lambda_{1}(\bm{L}_{L})\le3np\label{eq:L_L_spectral},
\end{equation}
\begin{equation}
\frac{np}{16\kappa_{1}\kappa_{2}}\le\lambda_{m-1}(\bm{L}_{Lz})\le\lambda_{1}(\bm{L}_{Lz})\le np.\label{eq:L_Lz_spectral}
\end{equation}

\end{lemma}

\subsubsection{Proof of Theorem~\ref{thm:rasch_infty}}

Now we are ready to prove Theorem~\ref{thm:rasch_infty}. We focus
on analyzing \myalg, as the analysis of \myalgM\ follows immediately
from the union bound of the different data splitting and the triangular
inequality: 
\[
\|\widehat{\bm{\theta}}-\bm{\theta}^{\star}\|_{\infty}\leq\frac{1}{n_{\mathrm{split}}}\sum_{i=1}^{n_{\mathrm{split}}}\|\widehat{\bm{\theta}}^{(i)}-\bm{\theta}^{\star}\|_{\infty}.
\]
By assumption we have $mp\ge2$ and $np\ge C_{1}\kappa_{1}^{4}\kappa_{2}^{5}\log^{3}(n)$
for some constant $C_{1}>0$. Then we can apply Lemmas~\ref{lemma:spectral}
and \ref{lem:degree_Y} to see that 
\[
\frac{np}{4\kappa_{2}}\le\lambda_{m-1}(\bm{L}_{L}),\qquad\text{and}\qquad d_{\max}\le\frac{3}{2}np.
\]
We observe that (\ref{eq:MLE_general_cond}) is satisfied as long
as $np\ge C_{1}\kappa_{1}^{4}\kappa_{2}^{5}\log^{3}(n)$ for some
constant $C_{1}$ that is large enough. Invoking Lemma~\ref{lemma:MLE_general},
we conclude that 
\[
\|\widehat{\bm{\theta}}-\bm{\theta}^{\star}\|_{\infty}\le C_{2}\kappa_{1}\sqrt{\frac{\log(n)}{\lambda_{m-1}(\bm{L}_{L})}}\le2C_{2}\kappa_{1}\kappa_{2}^{1/2}\sqrt{\frac{\log(n)}{np}}.
\]

It remains to show the top-$K$ recovery sample complexity. As $\theta_{1}^{\star}\ge\cdots\ge\theta_{K}^{\star}>\theta_{K+1}^{\star}\ge\cdots\ge\theta_{m}^{\star}$
by assumption, it suffices to show $\widehat{\theta}_{i}-\widehat{\theta}_{j}>0$
for any $i\le K$ and $j>K$. Using the $\ell_{\infty}$ error bound,
we have that 
\begin{align*}
\widehat{\theta}_{i}-\widehat{\theta}_{j} & \ge\left(\theta_{i}^{\star}-\theta_{j}^{\star}\right)-\left|\widehat{\theta}_{i}-\theta_{i}^{\star}\right|-\left|\widehat{\theta}_{j}-\theta_{j}^{\star}\right|\ge\Delta_{K}-4C_{2}\kappa_{1}\kappa_{2}^{1/2}\sqrt{\frac{\log(n)}{np}}.
\end{align*}
Then $\widehat{\theta}_{i}-\widehat{\theta}_{j}>0$ as long as 
\[
np\ge\frac{16C_{2}^{2}\kappa_{1}^{2}\kappa_{2}\log(n)}{\Delta_{K}^{2}}.
\]

\subsection{Analysis for non-asymptotic expansion\label{subsec:Analysis_non_asymp_expansion}}

To make the main text concise, we provide a sketch of the proof of
Theorem~\ref{thm:rasch_distribution} and leave the full one to Section~\ref{subsec:Proof_rasch_dist}.

The proof is inspired by the proof of Theorem~1 in \cite{chen2023ranking},
which analyzes MLE via the trajectory of the preconditioned gradient
descent (PGD) dynamic starting from ground truth. More precisely,
letting $\bm{\theta}^{0}=\bm{\theta}^{\star}$, we consider the PGD
iterates defined by
\[
\bm{\theta}^{t+1}=\bm{\theta}^{t}-\eta\bm{L}_{Lz}^{\dagger}\nabla\mathcal{L}(\bm{\theta}^{t}),
\]
where $\eta>0$ is the step size of PGD. \cite{chen2023ranking} shows
that this dynamic converges to $\widehat{\bm{\theta}}$. We proceed
one step further by establishing precise distributional characterization
of $\widehat{\bm{\theta}}$ via analyzing PGD. With Taylor expansion,
the gradient can be decomposed into 
\[
\nabla\mathcal{L}(\bm{\theta}^{t})=\bm{L}_{Lz}(\bm{\theta}^{t}-\bm{\theta}^{\star})-\bm{B}\widehat{\bm{\epsilon}}+\bm{r}^{t},
\]
where $\bm{r}^{t}$ is a residual vector with small magnitude. Then
the PGD update becomes 
\[
\bm{\theta}^{t+1}-\bm{\theta}^{\star}=\left(1-\eta\right)(\bm{\theta}^{t}-\bm{\theta}^{\star})-\eta\left(\bm{L}_{Lz}^{\dagger}\bm{B}\widehat{\bm{\epsilon}}-\bm{L}_{Lz}^{\dagger}\bm{r}^{t}\right).
\]
We establish Theorem~\ref{thm:rasch_distribution} by solving this
recursive relation. More specifically, as $\bm{L}_{Lz}^{\dagger}\bm{B}\widehat{\bm{\epsilon}}$
does not depend on $t$ and $\|\bm{L}_{Lz}^{\dagger}\bm{r}^{t}\|_{\infty}$
can be controlled for each step $t$, taking $t\rightarrow\infty$,
we see that 
\[
\widehat{\bm{\theta}}-\bm{\theta}^{\star}=\lim_{t\rightarrow\infty}\bm{\theta}^{t}-\bm{\theta}^{\star}=-\bm{L}_{Lz}^{\dagger}\bm{B}\widehat{\bm{\epsilon}}+\bm{r}
\]
for some residual term $\bm{r}$ that is well controlled in $\ell_{\infty}$
norm. 

\subsection{Analysis for asymptotic normality\label{subsec:Analysis_asymp_normal}}

The analysis of the asymptotic normality when $m,p$ are fixed is
standard for maximum likelihood estimators. Here we illustrate the
idea on \myalg\ for one random splitting. By mean value theorem
\begin{align*}
\sum_{t=1}^{n}\nabla\mathcal{L}_{k}^{(t)}(\bm{\theta}^{\star}) & =\sum_{t=1}^{n}\nabla\mathcal{L}_{k}^{(t)}(\widehat{\bm{\theta}}_{k}^{(n)})+\left[\int_{\tau=0}^{1}\sum_{t=1}^{n}\nabla^{2}\mathcal{L}_{k}^{(t)}(\bm{\theta}^{\star}+\tau(\widehat{\bm{\theta}}_{k}^{(n)}-\bm{\theta}^{\star}))\right]\mathrm{d}\tau(\bm{\theta}^{\star}-\widehat{\bm{\theta}}_{k}^{(n)})\\
  & =\left[\int_{\tau=0}^{1}\sum_{t=1}^{n}\nabla^{2}\mathcal{L}_{k}^{(t)}(\bm{\theta}^{\star}+\tau(\widehat{\bm{\theta}}_{k}^{(n)}-\bm{\theta}^{\star}))\mathrm{d}\tau\right](\bm{\theta}^{\star}-\widehat{\bm{\theta}}_{k}^{(n)}).
\end{align*}
Here the second row comes from the optimality condition of $\widehat{\bm{\theta}}_{k}^{(n)}$.
Now under some regularity conditions, we can use the consistency of
$\widehat{\bm{\theta}}_{k}^{(n)}$ to show 
\begin{equation}
\int_{\tau=0}^{1}\sum_{t=1}^{n}\nabla^{2}\mathcal{L}_{k}^{(t)}(\bm{\theta}^{\star}+\tau(\widehat{\bm{\theta}}_{k}^{(n)}-\bm{\theta}^{\star}))\mathrm{d}\tau\approx\bm{H}^{\infty}.\label{eq:Hessian_conv_groundtruth-1}
\end{equation}
Then $\bm{\theta}^{\star}-\widehat{\bm{\theta}}_{k}^{(n)}\approx(\bm{H}^{\infty})^{\dagger}\sum_{t=1}^{n}\nabla\mathcal{L}_{k}^{(t)}(\bm{\theta}^{\star})$.
Note that $\nabla\mathcal{L}_{k}^{(t)}(\bm{\theta}^{\star})$ is zero-mean
and independent between different user $t$. This independence also
holds for $\sum_{k=1}^{n_{\mathrm{split}}}\nabla\mathcal{L}_{k}^{(t)}(\bm{\theta}^{\star})$
and $\nabla\mathcal{L}_{\mathrm{WP}}^{(t)}(\bm{\theta}^{\star})$.
Then we can invoke central limit theorem to reach the desired result.

\section{Degree and spectral properties of the comparison graphs\label{sec:Proofs_deg_spec}}

In this section, we present the analysis for lemmas that characterize
the degree and spectral properties of the comparison graphs. We start
with a lemma that controls the degrees in $\mathcal{G}_{X}$, and
then prove Lemmas~\ref{lem:degree_Y} and \ref{lemma:spectral}.

\subsection{Degree range of $\mathcal{G}_{X}$ \label{subsec:Degree_X}}

Recall that $m_{t}$ is the number of neighbors of user $t$ in $\mathcal{G}_{X}$.
Furthermore, we denote $n_{i}$ as the number of users that is compared
with problem $i$ and at least another item, i.e., 
\begin{equation}
n_{i}\coloneqq\left|\{t:(t,i)\in\mathcal{E}_{X},m_{t}\ge2\}\right|.\label{eq:m_i_def}
\end{equation}
The following lemma controls the size of $m_{t}$ and $n_{i}$. 

\begin{lemma}[Degree bounds in $\mathcal{G}_{X}$] \label{lemma:degree_X}
Suppose that $np\ge C\log(n)$ for some large enough constant $C>0$
and that $m\le n^{\alpha}$ for some constant $\alpha>0$. Then with
probability at least $1-2n^{-10}$, for all $i\in[m]$, we have
\begin{equation}
\frac{1}{4}np\le n_{i}\le\frac{3}{2}np.\label{eq:deg_m}
\end{equation}
Moreover, with probability at least $1-n^{-10}$, for all $t\in[n]$,
we have
\begin{equation}
m_{t}\le\left(\frac{3}{2}mp\right)\vee165\log(n).\label{eq:deg_n}
\end{equation}

\end{lemma}\begin{proof}

We prove the two claims in the lemma sequentially. 

Fix any $t,i$. One has
\begin{align}
\mathbb{P}\left[t:(t,i)\in\mathcal{E}_{X},m_{t}\ge2\right] & =\mathbb{P}\left[(t,i)\in\mathcal{E}_{X}\right]-\mathbb{P}\left[(t,i)\in\mathcal{E}_{X}\text{ and }m_{t}=1\right]\nonumber \\
 & =p-p(1-p)^{m-1}\nonumber \\
 & \ge p(1-e^{-(m-1)p})\nonumber \\
 & \ge\frac{1}{2}p,\label{eq:non_singular_neighbor}
\end{align}
as long as $mp\ge2$. Let $\mu_{i}\coloneqq\mathbb{E}[n_{i}]$. By
the linearity of expectation, we have
\begin{equation}
np/2\le\mu_{i}\le\sum_{t}\mathbb{P}\left[(t,i)\in\mathcal{E}_{X}\right]=np.\label{eq:mu_i}
\end{equation}
Fix $i\in[m]$. Since the sampling is independent with different $t$,
by the Chernoff bound, 
\begin{align*}
\mathbb{P}[|n_{i}-\mu_{i}|\le(1/2)\mu_{i}] & \le2e^{-\frac{1}{12}\mu_{i}}\le2e^{-\frac{1}{24}np}\le m^{-1}n^{-10}
\end{align*}
as long as $np\ge C\log(n)$ for large enough constant $C$. Applying
(\ref{eq:mu_i}) and union bound on $i\in[m]$ yields (\ref{eq:deg_m}).

Moving on to (\ref{eq:deg_n}), we first consider the case where $mp\ge110\log(n)$.
By Chernoff bound, 
\[
\mathbb{P}[m_{t}\ge(3/2)mp]\le e^{-\frac{1}{10}mp}\le n^{-11}.
\]
In the case of $mp<110\log(n)$, the quantity $\mathbb{P}\left[m_{t}\ge165\log(n)\right]$
clear decreases as $p$ decreases. So we may use the case $mp=110\log(n)$
to bound this quantity and conclude that 
\[
\mathbb{P}\left[m_{t}\ge165\log(n)\right]\le n^{-11}.
\]
Finally we apply union bound on $t\in[n]$ to reach (\ref{eq:deg_n}). 

\end{proof}

\subsection{Proof of Lemma~\ref{lem:degree_Y}}

The assumption of this lemma allows us to invoke Lemma~\ref{lemma:degree_X}.
For the upper bound of $d_{i}$, since (\ref{eq:deg_m}) is true,
\begin{align*}
d_{i} & =\sum_{j:j\neq i}L_{ij}\\
 & =\sum_{t:(t,i)\in\mathcal{E}_{X}}\sum_{j:j\neq i}L_{ij}^{t}\\
 & \overset{(\text{i})}{=}\sum_{t:(t,i)\in\mathcal{E}_{X},m_{t}\ge2}\sum_{j:j\neq i}L_{ij}^{t}\\
 & \overset{(\text{ii})}{\le}\sum_{t:(t,i)\in\mathcal{E}_{X},m_{t}\ge2}\sum_{j:j\neq i}R_{ij}^{t}\le n_{i}\le\frac{3}{2}np.
\end{align*}
Here (i) holds since $L_{ij}^{t}$ can only be 0 when $m_{t}\le1$,
and (ii) holds since $L_{ij}^{t}\le R_{ij}^{t}$ by definition. For
the lower bound of $d_{i}$, notice that for any $(t,i)$, $\sum_{j}L_{ij}^{t}$
is either $0$ or 1. Fix $\mathcal{E}_{X}$ and only consider randomness
on $L_{ij}^{t}$. By Hoeffding's inequality, 
\begin{align}
\mathbb{P}\left\{ d_{i}-\sum_{t:(t,i)\in\mathcal{E}_{X},m_{t}\ge2}\mathbb{E}\left[\sum_{j:j\neq i}L_{ij}^{t}\mid(t,i)\in\mathcal{E}_{X}\right]\le-\frac{1}{12\kappa_{2}}np\right\}  & \le\exp\left(-\frac{(1/72)\kappa_{2}^{-2}n^{2}p^{2}}{n_{i}}\right)\nonumber \\
 & \le\exp\left(-\frac{\kappa_{2}^{-2}np}{108}\right)\nonumber \\
 & \le m^{-1}n^{-10}\label{eq:d_i_Hoeff}
\end{align}
as long as $np\ge1200\kappa_{2}^{2}\log(n)$. The second to last inequality
uses (\ref{eq:deg_m}). For each $(t,i)\in\mathcal{E}_{X}$,
\begin{align*}
\mathbb{E}\left[\sum_{j}L_{ij}^{t}\mid(t,i)\in\mathcal{E}_{X}\right] & =\sum_{j}\mathbb{P}\left[R_{ij}^{t}=1\mid(t,i)\in\mathcal{E}_{X}\right]\mathbb{P}\left[\sum_{j}L_{ij}^{t}=1\mid R_{ij}^{t}=1\right]\\
 & \ge\sum_{j}\mathbb{P}\left[R_{ij}^{t}=1\mid(t,i)\in\mathcal{E}_{X}\right]\frac{2\kappa_{2}}{(1+\kappa_{2})^{2}}\\
 & \ge\frac{2\lfloor m_{t}/2\rfloor}{m_{t}}\cdot\frac{2\kappa_{2}}{(1+\kappa_{2})^{2}}\ge\frac{1}{3\kappa_{2}}
\end{align*}
as long as $m_{t}\ge2$. The first inequality here uses Fact~\ref{fact:BTL}.
Then by definition of $n_{i}$ in (\ref{eq:m_i_def}),
\begin{equation}
\sum_{t:(t,i)\in\mathcal{E}_{X},m_{t}\ge2}\mathbb{E}\left[\sum_{j}L_{ij}^{t}\mid(t,i)\in\mathcal{E}_{X}\right]\ge\frac{1}{3\kappa_{2}}n_{i}.\label{eq:d_i_expectation}
\end{equation}
Combining (\ref{eq:d_i_Hoeff}), (\ref{eq:d_i_expectation}) and (\ref{eq:deg_m}),
\[
d_{i}\ge\frac{1}{3\kappa_{2}}n_{i}-\frac{1}{12\kappa_{2}}np\ge\frac{1}{24\kappa_{2}}np.
\]
Applying union bound over $i\in[m]$ yields the desired result.

\subsection{Proof of Lemma~\ref{lemma:spectral}}

We first consider the spectrum of $\bm{L}_{L}$. Recall that
\begin{align*}
\bm{L}_{L} & =\sum_{(i,j)\in\mathcal{E}_{Y},i>j}L_{ij}(\bm{e}_{i}-\bm{e}_{j})(\bm{e}_{i}-\bm{e}_{j})^{\top}\\
 & =\sum_{t=1}^{n}\underbrace{\sum_{(i,j)\in\mathcal{E}_{Y},i>j}L_{ij}^{t}(\bm{e}_{i}-\bm{e}_{j})(\bm{e}_{i}-\bm{e}_{j})^{\top}}_{\bm{L}_{L}^{t}}.
\end{align*}
For the upper bound, it is clear from Lemmas~\ref{lemma:top_eigen_Laplacian}
and \ref{lem:degree_Y} that $\lambda_{1}(\bm{L}_{L})\le2\max_{i}d_{i}\le3np$.
For the lower bound, we use the matrix Chernoff inequality (see Section
5 of \cite{tropp2015matrix}). Let $\bm{R}\in\mathbb{R}^{(n-1)\times n}$
be a partial isometry such that $\bm{R}\bm{R}^{\top}=\bm{I}_{m-1}$
and $\bm{R}\bm{1}=\bm{0}$. Then $\lambda_{m-1}(\bm{L}_{L})=\lambda_{m-1}(\bm{R}\bm{L}_{L}\bm{R}^{\top})$.
For any $t\in[n]$, by (\ref{eq:deg_n}),
\[
0\le\lambda_{m-1}(\bm{R}\bm{L}_{L}^{t}\bm{R}^{\top})\le\lambda_{1}(\bm{R}\bm{L}_{L}^{t}\bm{R}^{\top})=\lambda_{1}(\bm{L}_{L}^{t})\le2.
\]
The last inequality follows from Lemma~\ref{lemma:top_eigen_Laplacian}
since $\sum_{j}\bm{L}_{ij}^{t}\le1$ for any $t\in[n]$ and $i\in[m]$.
By Fact~\ref{fact:BTL}, 
\[
\mathbb{P}[L_{ij}^{t}=1\mid R_{ij}^{t}=1]\ge\frac{2\kappa_{2}}{(1+\kappa_{2})^{2}}\ge1/(2\kappa_{2}).
\]
Then
\begin{align*}
\lambda_{m-1}(\mathbb{E}\bm{R}\bm{L}_{L}\bm{R}^{\top}) & =\lambda_{m-1}\left(\bm{R}\sum_{t=1}^{n}\sum_{i>j}\mathbb{E}L_{ij}^{t}(\bm{e}_{i}-\bm{e}_{j})(\bm{e}_{i}-\bm{e}_{j})^{\top}\bm{R}^{\top}\right)\\
 & \ge\frac{1}{2\kappa_{2}}\lambda_{m-1}\left(\bm{R}\sum_{t=1}^{n}\sum_{i>j}\mathbb{E}R_{ij}^{t}(\bm{e}_{i}-\bm{e}_{j})(\bm{e}_{i}-\bm{e}_{j})^{\top}\bm{R}^{\top}\right)
\end{align*}
Moreover $\sum_{i>j}\mathbb{E}R_{ij}^{t}\ge\frac{1}{2}(\mathbb{E}[m_{t}]-1)=(mp-1)/2$,
where the $-1$ accounts for possible unpaired $X_{ti}$. By symmetry
$\mathbb{E}R_{ij}^{t}$ is the same for any $(i,j)$. Then for any
$(i,j)$, 
\[
\mathbb{E}R_{ij}^{t}\ge\frac{mp-1}{2}/\binom{m}{2}\ge\frac{p}{2m}
\]
 as long as $mp\ge2$. Thus 
\begin{align*}
\lambda_{m-1}(\mathbb{E}\bm{R}\bm{L}_{L}\bm{R}^{\top}) & \ge\frac{1}{2\kappa_{2}}\lambda_{m-1}\left(\bm{R}\sum_{t=1}^{n}\sum_{i>j}\frac{p}{2m}(\bm{e}_{i}-\bm{e}_{j})(\bm{e}_{i}-\bm{e}_{j})^{\top}\bm{R}^{\top}\right)\\
 & =\frac{1}{2\kappa_{2}}\cdot n\cdot\frac{p}{2m}\cdot m\\
 & =\frac{np}{4\kappa_{2}}.
\end{align*}
Now invoke the matrix Chernoff inequality, we have
\begin{align*}
\mathbb{P}\left\{ [\lambda_{m-1}(\bm{R}\bm{L}_{L}\bm{R}^{\top})]\le\frac{np}{8\kappa_{2}}\right\}  & \le m\cdot\left[\frac{e^{-1/2}}{(1/2)^{1/2}}\right]^{\frac{np}{4\kappa_{2}}/2}\le n^{-10}
\end{align*}
as long as $np\ge C\kappa_{2}\log(n)$ for some large enough constant
$C$. 

The spectrum of $\bm{L}_{Lz}$ comes directly from the spectrum of
$\bm{L}_{L}$. Recall 
\[
\bm{L}_{Lz}=\sum_{(i,j)\in\mathcal{E}_{Y},i>j}L_{ij}z_{ij}(\bm{e}_{i}-\bm{e}_{j})(\bm{e}_{i}-\bm{e}_{j})^{\top}.
\]
By Lemma~\ref{lemma:z_range},
\begin{align*}
\lambda_{1}(\bm{L}_{Lz}) & =\max_{\bm{v}:\|\bm{v}\|=1}\bm{v}^{\top}\sum_{(i,j)\in\mathcal{E}_{Y},i>j}L_{ij}z_{ij}(\bm{e}_{i}-\bm{e}_{j})(\bm{e}_{i}-\bm{e}_{j})^{\top}\bm{v}\\
 & \le\frac{1}{4}\max_{\bm{v}:\|\bm{v}\|=1}\bm{v}^{\top}\sum_{(i,j)\in\mathcal{E}_{Y},i>j}L_{ij}(\bm{e}_{i}-\bm{e}_{j})(\bm{e}_{i}-\bm{e}_{j})^{\top}\bm{v}\\
 & =\frac{1}{4}\lambda_{1}(\bm{L}_{L})\le np,
\end{align*}
and 
\begin{align*}
\lambda_{m-1}(\bm{L}_{Lz}) & =\min_{\bm{v}:\|\bm{v}\|=1,\bm{v}^{\top}\bm{1}_{m}=0}\bm{v}^{\top}\sum_{(i,j)\in\mathcal{E}_{Y},i>j}L_{ij}z_{ij}(\bm{e}_{i}-\bm{e}_{j})(\bm{e}_{i}-\bm{e}_{j})^{\top}\bm{v}\\
 & \ge\frac{1}{4\kappa_{1}}\min_{\bm{v}:\|\bm{v}\|=1,\bm{v}^{\top}\bm{1}_{m}=0}\bm{v}^{\top}\sum_{(i,j)\in\mathcal{E}_{Y},i>j}L_{ij}(\bm{e}_{i}-\bm{e}_{j})(\bm{e}_{i}-\bm{e}_{j})^{\top}\bm{v}\\
 & =\frac{1}{4\kappa_{1}}\lambda_{m-1}(\bm{L}_{L})\ge\frac{np}{16\kappa_{1}\kappa_{2}}.
\end{align*}

\section{Proofs for Section~\ref{subsec:Non_asymp_expansion}}

In this section, we provide the proofs for the results related to the
non-asymptotic expansion in Section~\ref{subsec:Non_asymp_expansion}. 

\subsection{Proof of Theorem~\ref{thm:rasch_distribution}\label{subsec:Proof_rasch_dist}}

We study the MLE $\widehat{\bm{\theta}}$ by analyzing the iterates
of preconditioned gradient descent starting from the ground truth.
Let $\bm{\theta}_{0}=\bm{\theta}^{\star}$ be the starting point and
$\eta>0$ be the stepsize that is small enough. At iteration $\tau$,
the preconditioned gradient descent update is given by
\begin{equation}
\bm{\theta}^{\tau+1}=\bm{\theta}^{\tau}-\eta\bm{L}_{Lz}^{\dagger}\nabla\mathcal{L}(\bm{\theta}^{\tau}).\label{eq:precondGD}
\end{equation}
Recall the definitions: $\sigma(x)=e^{x}/(1+e^{x})$ is the sigmoid
function. We also have $z_{ij}\coloneqq e^{\theta_{i}^{\star}}e^{\theta_{j}^{\star}}/(e^{\theta_{i}^{\star}}+e^{\theta_{j}^{\star}})^{2}$,
$\widehat{z}_{ij}\coloneqq e^{\hat{\theta}_{i}}e^{\hat{\theta}_{j}}/(e^{\hat{\theta}_{i}}+e^{\hat{\theta}_{j}})^{2}$
and $\epsilon_{ij}^{t}\coloneqq Y_{ji}^{t}-\sigma(\theta_{i}^{\star}-\theta_{j}^{\star})$.
The total number of observed effective comparisons in $\mathcal{G}_{Y}$
is $L_{\mathrm{total}}\coloneqq\sum_{i>j:(i,j)\in\mathcal{E}_{Y}}L_{ij}$.

We have defined $\bm{\bm{B}}\in\mathbb{R}^{m\times L_{\mathrm{total}}}$
and $\widehat{\bm{\epsilon}}\in\mathbb{R}^{L_{\mathrm{total}}}$ as
\[
\bm{B}\coloneqq\left[\cdots,\sqrt{z_{ij}}(\bm{e}_{i}-\bm{e}_{j}),\cdots\right]_{i>j:(i,j)\in\mathcal{E}_{Y}}\text{\ensuremath{\quad}(repeat \ensuremath{L_{ij}} times for edge \ensuremath{(i,j)})}
\]
and
\[
\widehat{\bm{\epsilon}}\coloneqq\left[\cdots,\epsilon_{ij}^{t}/\sqrt{z_{ij}},\cdots\right]_{(i,j,t):i>j,(i,j)\in\mathcal{E}_{Y},L_{ij}^{t}=1}\in\mathbb{R}^{L_{\mathrm{total}}}.
\]
Moreover, define a weighted graph Laplacian
\[
\bm{L}_{L\tilde{z}}\coloneqq\sum_{(i,j)\in\mathcal{E}_{Y},i>j}L_{ij}\widetilde{z}_{ij}(\bm{e}_{i}-\bm{e}_{j})(\bm{e}_{i}-\bm{e}_{j})^{\top}
\]
for $\tilde{z}$ being $z$ or $\widehat{z}$, and $\bm{L}_{L\tilde{z}}^{\dagger}$
is its pseudo-inverse.

Consider the Taylor expansion of $\nabla\mathcal{L}(\bm{\theta}^{\tau})$,
we have
\begin{align*}
\nabla\mathcal{L}(\bm{\theta}^{\tau}) & =\sum_{i>j:(i,j)\in\mathcal{E}_{Y}}\sum_{t:L_{ij}^{t}=1}\left(\left(-Y_{ji}^{t}+\sigma(\theta_{i}^{\tau}-\theta_{j}^{\tau})\right)(\bm{e}_{i}-\bm{e}_{j})\right)\\
 & =\sum_{i>j:(i,j)\in\mathcal{E}_{Y}}\sum_{t:L_{ij}^{t}=1}\left(\left(-\epsilon_{ji}^{t}-\sigma(\theta_{i}^{\star}-\theta_{j}^{\star})+\sigma(\theta_{i}^{\tau}-\theta_{j}^{\tau})\right)(\bm{e}_{i}-\bm{e}_{j})\right)\\
 & =\sum_{i>j:(i,j)\in\mathcal{E}_{Y}}\sum_{t:L_{ij}^{t}=1}\left(\left(-\epsilon_{ji}^{t}+\sigma'(\theta_{i}^{\star}-\theta_{j}^{\star})(\delta_{i}^{\tau}-\delta_{j}^{\tau})+\frac{1}{2}\sigma''(\xi_{ij}^{\tau})(\delta_{i}^{\tau}-\delta_{j}^{\tau})^{2}\right)(\bm{e}_{i}-\bm{e}_{j})\right).
\end{align*}
Here $\bm{\delta}^{\tau}\coloneqq\bm{\theta}^{\tau}-\bm{\theta}^{\star}$
and for all $(i,j)$, $\xi_{ij}^{\tau}\in\mathbb{R}$ is some number
that lies between $\theta_{i}^{\star}-\theta_{j}^{\star}$ and $\theta_{i}^{\tau}-\theta_{j}^{\tau}$.
As $\sigma'(\theta_{i}^{\star}-\theta_{j}^{\star})=z_{ij}$ and $\delta_{i}^{\tau}-\delta_{j}^{\tau}=(\bm{e}_{i}-\bm{e}_{j})^{\top}\bm{\delta}^{\tau}$,
we can rewrite $\nabla\mathcal{L}(\bm{\theta}^{\tau})$ as
\[
\nabla\mathcal{L}(\bm{\theta}^{\tau})=\bm{L}_{Lz}\bm{\delta}^{\tau}-\bm{B}\widehat{\bm{\epsilon}}+\bm{r}^{\tau},
\]
where $\bm{r}^{\tau}=\sum_{i>j:(i,j)\in\mathcal{E}_{Y}}L_{ij}\cdot[\frac{1}{2}\sigma''(\xi_{ij}^{\tau})(\delta_{i}^{\tau}-\delta_{j}^{\tau})^{2}(\bm{e}_{i}-\bm{e}_{j})]$.
Feeding this into (\ref{eq:precondGD}), we have
\begin{equation}
\bm{\delta}^{\tau+1}=\left(1-\eta\right)\bm{\delta}^{\tau}-\eta\left(\bm{L}_{Lz}^{\dagger}\bm{B}\widehat{\bm{\epsilon}}-\bm{L}_{Lz}^{\dagger}\bm{r}^{\tau}\right)\label{eq:precGD_itr_t}
\end{equation}
By definition $\bm{\delta}^{0}=\bm{0}$. Applying this recursive relation
$\tau-1$ times we obtain
\begin{align*}
\bm{\delta}^{\tau} & =-\eta\sum_{i=0}^{\tau-1}(1-\eta)^{i}\bm{L}_{Lz}^{\dagger}\bm{B}\widehat{\bm{\epsilon}}+\sum_{i=0}^{\tau-1}(1-\eta)^{\tau-1-i}\bm{L}_{Lz}^{\dagger}\bm{r}^{i}\\
 & =-\left[1-(1-\eta)^{\tau}\right]\bm{L}_{Lz}^{\dagger}\bm{B}\widehat{\bm{\epsilon}}+\sum_{i=0}^{\tau-1}(1-\eta)^{\tau-1-i}\bm{L}_{Lz}^{\dagger}\bm{r}^{i}.
\end{align*}
At this point, we invoke an existing result on these terms that have
been studied in \cite{Yang2024} for a more general setting. The proof
of Lemma~6 combined with Lemmas~7 and 8 in \cite{Yang2024} reveal
the following properties of (\ref{eq:precondGD}). The proof is deferred
to Section~\ref{subsec:proof_precondGD}.

\begin{lemma}\label{lemma:preconGD}Instate the assumptions of Theorem~\ref{thm:rasch_infty}.
Suppose that 
\begin{equation}
\kappa_{1}^{3}\frac{(d_{\max})^{2}\log^{2}(n)}{(\lambda_{m-1}(\bm{L}_{Lz}))^{3}}\le C_{1}\kappa_{1}\sqrt{\frac{\log(n)}{\lambda_{m-1}(\bm{L}_{Lz})}},\label{eq:condition_precGD}
\end{equation}
for some constant $C_{1}>0$. Then with probability at least $1-n^{-10}$,
the precondition gradient descent dynamic satisfies the following
properties:
\begin{enumerate}
\item There exists a unique minimizer $\widehat{\bm{\theta}}$ of (\ref{eq:MLE_loss}).
\item There exist some $\alpha_{1},\alpha_{2}$ obeying $0<\alpha_{1}\le\alpha_{2}$
such that any $\tau\in\mathbb{N}$, 
\[
\|\bm{\theta}^{\tau}-\widehat{\bm{\theta}}\|_{\bm{L}_{Lz}}\le(1-\eta\alpha_{1})^{\tau}\|\bm{\theta}^{0}-\widehat{\bm{\theta}}\|_{\bm{L}_{Lz}},
\]
provided that $0<\eta\le1/\alpha_{2}$.
\item For any $k,l$ and iteration $\tau\ge0$,
\begin{equation}
\left|\left(\theta_{k}^{\tau}-\theta_{l}^{\tau}\right)-\left(\theta_{k}^{\star}-\theta_{l}^{\star}\right)\right|\le C_{2}\kappa_{1}\sqrt{\frac{\log(n)}{\lambda_{m-1}(\bm{L}_{L})}}\label{eq:B_itr_t}
\end{equation}
for some constant $C_{2}>0$.
\item For any $k,l$ and iteration $\tau\ge0$,
\[
\left|(\bm{e}_{k}-\bm{e}_{l})^{\top}\bm{L}_{Lz}^{\dagger}\bm{r}^{\tau}\right|\le C_{3}\kappa_{1}^{3}\frac{(d_{\max})^{2}\log^{2}(n)}{(\lambda_{m-1}(\bm{L}_{L}))^{3}}
\]
for some constant $C_{3}>0$.
\end{enumerate}
\end{lemma}

Lemmas~\ref{lemma:spectral} and \ref{lem:degree_Y} imply that
\[
\frac{np}{4\kappa_{1}\kappa_{2}}\le\lambda_{m-1}(\bm{L}_{Lz})\qquad\text{and}\qquad d_{\max}\le\frac{3}{2}np.
\]
Then the condition (\ref{eq:condition_precGD}) holds as long as $np\ge C_{4}\kappa_{1}^{2}\kappa_{2}\log^{3}(n)$
for some large enough constant $C_{4}$. Invoke Lemma~\ref{lemma:preconGD}
to see that for any $(k,l,\tau)$,
\begin{equation}
\left|(\bm{e}_{k}-\bm{e}_{l})^{\top}\bm{L}_{Lz}^{\dagger}\bm{r}^{\tau}\right|\le C_{5}\kappa_{1}^{6}\frac{\log^{2}(n)}{np},\label{eq:r_bound}
\end{equation}
where $C_{5}>0$ is some constant. Furthermore, the convergence given
by Lemma~\ref{lemma:preconGD} implies that
\begin{align}
\widehat{\bm{\delta}} & =\lim_{\tau\rightarrow\infty}\bm{\delta}^{\tau}=-\bm{L}_{Lz}^{\dagger}\bm{B}\widehat{\bm{\epsilon}}+\underbrace{\eta\lim_{\tau\rightarrow\infty}\sum_{i=0}^{\tau-1}(1-\eta)^{\tau-1-i}\bm{L}_{Lz}^{\dagger}\bm{r}^{i}}_{\eqqcolon\bm{r}}.\label{eq:delta_decomp}
\end{align}

It remains to control the $\ell_{\infty}$ norm of $\bm{r}$. For
any $(k,l$), (\ref{eq:r_bound}) shows that
\[
|r_{k}-r_{j}|\le\eta\lim_{\tau\rightarrow0}\sum_{i=0}^{\tau-1}(1-\eta)^{\tau-1-i}C_{5}\kappa_{1}^{6}\frac{\log^{2}(n)}{np}=C_{5}\kappa_{1}^{6}\frac{\log^{2}(n)}{np}.
\]
As $\bm{1}^{\top}\bm{L}_{Lz}^{\dagger}=0$, the above inequality implies
that
\begin{align*}
|r_{k}| & =\left|\frac{1}{m}\cdot m\bm{e}_{k}^{\top}\bm{r}\right|=\left|\frac{1}{m}\sum_{l=1}^{m}(\bm{e}_{k}-\bm{e}_{l})^{\top}\bm{r}\right|\\
 & =\left|\frac{1}{m}\sum_{l=1}^{m}(r_{k}-r_{l})\right|\le C_{5}\kappa_{1}^{6}\frac{\log^{2}(n)}{np}.
\end{align*}
The proof is now completed.

\subsubsection{Proof of Lemma~\ref{lemma:preconGD} \label{subsec:proof_precondGD}}

Lemma~\ref{lemma:preconGD} is a direct combination of Lemmas~7,
Lemma~8, and the proof of Lemma~6 in~\cite{Yang2024}. The results
therein describe controls the error for the same dynamic with a more
general setting. Thus it is directly applicable to our case. For clarity,
in this section we explain the connection and relate our notations
with the ones used in \cite{Yang2024}. 

For any $(k,l)\in[m]^{2}$, let
\[
Q_{kl}\coloneqq C_{1}\kappa_{1}^{3}\frac{(d_{\max})^{2}\log^{2}(n)}{(\lambda_{m-1}(\bm{L}_{Lz}))^{3}}\qquad\text{and}\qquad B_{kl}\coloneqq C_{2}\kappa_{1}\sqrt{\frac{\log(n)}{\lambda_{m-1}(\bm{L}_{Lz})}},
\]
where $C_{1},C_{2}$ are some constants. Note that our setting in
this paper is unweighted and has different number of observation for
each edge in $\mathcal{G}_{Y}$. Moreover $\bm{L}_{wz}$ in~\cite{Yang2024}
correspond to $\bm{L}_{Lz}$ in this paper. Then Lemma~3 and Lemma~4
in~\cite{Yang2024} imply that $Q_{kl}$ and $B_{kl}$ satisfies
equation (9) therein. Furthermore, as we have assumed in(\ref{eq:condition_precGD}),
as long as $C_{2}$ is large enough, $Q_{kl}\le4B_{kl}$ for any $(k,l)$.
Thus we can invoke Lemma~2 as well as all its proof. Lemma~20 in~\cite{Yang2024}
implies the first two items in Lemma~\ref{lemma:preconGD} herein.
Lemma~19 in~\cite{Yang2024} implies the item 3. Finally, item 4
appears in the second-to-last equation block in the proof of Lemma~19
in~\cite{Yang2024}. 

\subsection{Proof of Proposition~\ref{prop:conv_dist}\label{subsec:Proof_conv_dist}}

We condition the whole analysis on the high probability event when
Lemmas~\ref{thm:rasch_distribution} and \ref{lemma:spectral} hold,
which happens with probability at least $1-O(n^{-10})$. 

We can expand $\bm{L}_{Lz}^{\dagger}\bm{B}\widehat{\bm{\epsilon}}$
as 
\[
\bm{L}_{Lz}^{\dagger}\bm{B}\widehat{\bm{\epsilon}}=-\bm{L}_{Lz}^{\dagger}\sum_{i>j:(i,j)\in\mathcal{E}_{Y}}\sum_{t:L_{ij}^{t}=1}\left[Y_{ji}^{t}-\sigma(\theta_{i}^{\star}-\theta_{j}^{\star})\right](\bm{e}_{i}-\bm{e}_{j})+\bm{r}
\]
where $\|\bm{r}\|_{\infty}\le C_{1}\kappa_{1}^{6}\log^{2}(n)/(np)$
for some constant $C_{1}$.

Consider $\bm{L}_{Lz}(\widehat{\bm{\theta}}-\bm{\theta}^{\star})$.
As $\bm{L}_{Lz}\bm{1}_{m}=\bm{0}$, $\lambda_{m-1}(\bm{L}_{Lz})>0$,
and $\widehat{\bm{\delta}}^{\top}\bm{1}_{m}=0$,
\begin{align*}
(\widehat{\bm{\theta}}-\bm{\theta}^{\star}) & =-\bm{L}_{Lz}^{\dagger}\bm{B}\widehat{\bm{\epsilon}}+\bm{r}\\
 & =-\bm{L}_{Lz}^{\dagger}\sum_{i>j:(i,j)\in\mathcal{E}_{Y}}\sum_{t:L_{ij}^{t}=1}\underbrace{\left[Y_{ji}^{t}-\sigma(\theta_{i}^{\star}-\theta_{j}^{\star})\right](\bm{e}_{i}-\bm{e}_{j})}_{\eqqcolon\bm{u}_{ij}^{t}}+\bm{r}.
\end{align*}
Conditional on $(\mathcal{E}_{Y},\{t:L_{ij}^{t}=1\})$, $\bm{u}_{ij}^{t}$
are independent random variables. It is also easy to see that $\bm{u}_{ij}^{t}$
is zero-mean and has covariance 
\begin{align*}
\mathbb{E}\left[\bm{u}_{ij}^{t}\bm{u}_{ij}^{t\top}\right] & =z_{ij}(\bm{e}_{i}-\bm{e}_{j})(\bm{e}_{i}-\bm{e}_{j})^{\top}.
\end{align*}
When rescaled by $(\bm{L}_{Lz}^{\dagger})^{1/2}$, it is also bounded
in the third moment of spectral norm by
\begin{equation}
\mathbb{E}\left[\|(\bm{L}_{Lz}^{\dagger})^{1/2}\bm{u}_{ij}^{t}\|^{3}\right]\le2^{3/2}\|\bm{L}_{Lz}^{\dagger}\|^{3/2}\le\frac{2^{15/2}\kappa_{1}^{3/2}\kappa_{2}^{3/2}}{(np)^{3/2}},\label{eq:u_third_moment}
\end{equation}
where the last inequality uses Lemma~\ref{lemma:spectral}. Summing
up across $i,j$ and $l$, we have 
\begin{align*}
\sum_{i>j:(i,j)\in\mathcal{E}_{Y}}\sum_{t:L_{ij}^{t}=1}\mathbb{E}\left[\bm{u}_{ij}^{t}\bm{u}_{ij}^{t\top}\right] & =\sum_{i>j:(i,j)\in\mathcal{E}_{Y}}L_{ij}z_{ij}(\bm{e}_{i}-\bm{e}_{j})(\bm{e}_{i}-\bm{e}_{j})^{\top}\\
 & =\bm{L}_{Lz}
\end{align*}
The last line holds since $\bm{L}_{Lz}\bm{1}_{m}=\bm{0}$ and $\lambda_{m-1}(\bm{L}_{Lz})>0$.
Now using multivariate Berry--Esseen theorem (see, e.g., \cite{Raic2019}),
let $\widetilde{\bm{x}}$ be a random variable such that 
\[
\bm{x}\sim\mathcal{N}\left(\bm{0}_{m},\bm{L}_{Lz}^{\dagger}\right),
\]
then
\begin{align*}
\sup_{A\in\mathcal{C}_{m}}\left|\mathbb{P}\left[-\bm{L}_{Lz}^{\dagger}\bm{B}\widehat{\bm{\epsilon}}\in A\mid\mathcal{G}_{Y}\right]-\mathbb{P}(\bm{x}\in A)\right| & \le C_{2}m^{1/4}\sum_{i>j:(i,j)\in\mathcal{E}_{Y}}\sum_{t:L_{ij}^{t}=1}\mathbb{E}\left[\|(\bm{L}_{Lz}^{\dagger})^{1/2}\bm{u}_{ij}^{t}\|^{3}\right]\\
 & \overset{\text{(i)}}{\le}C_{3}\frac{m^{5/4}\kappa_{1}^{3/2}\kappa_{2}^{3/2}}{(np)^{1/2}}.
\end{align*}
Here (i) uses \ref{lem:degree_Y} and (\ref{eq:u_third_moment}),
and $C_{2},C_{3}$ are some absolute constants. We then reach the
desired conclusion by adding the probability upper bound that Lemmas~\ref{thm:rasch_distribution}
and \ref{lemma:spectral} fail.

\subsection{Proof of Proposition~\ref{prop:rasch_l2} \label{subsec:Proof_L2}}

We start with the proof of (\ref{eq:rasch_l2}). By Theorem~\ref{thm:rasch_distribution}
we can express the MLE estimation error $\widehat{\bm{\theta}}-\bm{\theta}^{\star}$
as 
\begin{equation}
\widehat{\bm{\theta}}-\bm{\theta}^{\star}=-\bm{L}_{Lz}^{\dagger}\bm{B}\widehat{\bm{\epsilon}}+\bm{r}\label{eq:delta_expansion}
\end{equation}
for $\bm{B},\widehat{\bm{\epsilon}}$ defined in Section~\ref{subsec:Non_asymp_expansion}
and $\bm{r}\in\mathbb{R}^{m}$ is a residual term obeying $\|\bm{r}\|_{\infty}\le C_{1}\kappa_{1}^{6}\log^{2}n/(np)$
for some constant $C_{1}$. 

We first focus on the main term $\bm{L}_{Lz}^{\dagger}\bm{B}\widehat{\bm{\epsilon}}$.
Expanding $\bm{B}$ and $\widehat{\bm{\epsilon}}$, we rewrite it
as 
\[
\bm{L}_{Lz}^{\dagger}\bm{B}\widehat{\bm{\epsilon}}=\sum_{i>j:(i,j)\in\mathcal{E}_{Y}}\sum_{t:L_{ij}^{t}=1}\underbrace{\epsilon_{ij}^{t}\bm{L}_{Lz}^{\dagger}(\bm{e}_{i}-\bm{e}_{j})}_{\eqqcolon\bm{u}_{ij}^{t}}.
\]
It is easy to see that conditional on $(\mathcal{E}_{Y},\{t:L_{ij}^{t}=1\})$,
$\{\bm{u}_{ij}^{t}\}_{i,j,t}$ is a set of independent zero-mean random
variables. Thus we can expand $\mathbb{E}[\|\bm{L}_{Lz}^{\dagger}\bm{B}\widehat{\bm{\epsilon}}\|^{2}]$
to be 
\begin{align}
\mathbb{E}\left[\|\bm{L}_{Lz}^{\dagger}\bm{B}\widehat{\bm{\epsilon}}\|^{2}\right] & =\mathbb{E}\sum_{i>j:(i,j)\in\mathcal{E}_{Y}}\sum_{t:L_{ij}^{t}=1}\bm{u}_{ij}^{t\top}\bm{u}_{ij}^{t}\nonumber \\
 & =\mathbb{E}\sum_{i>j:(i,j)\in\mathcal{E}_{Y}}\sum_{t:L_{ij}^{t}=1}\mathrm{Trace}\left(\bm{u}_{ij}^{t}\bm{u}_{ij}^{t\top}\right)\nonumber \\
 & \overset{\text{(i)}}{=}\mathrm{Trace}\left(\sum_{i>j:(i,j)\in\mathcal{E}_{Y}}\sum_{t:L_{ij}^{t}=1}\bm{L}_{Lz}^{\dagger}z_{ij}(\bm{e}_{i}-\bm{e}_{j})(\bm{e}_{i}-\bm{e}_{j})^{\top}\bm{L}_{Lz}^{\dagger}\right)\nonumber \\
 & \overset{\text{(ii)}}{=}\mathrm{Trace}(\bm{L}_{Lz}^{\dagger}).\label{eq:E_LBepsilon}
\end{align}
Here (i) follows from the equality 
\[
z_{ij}=\mathrm{Var}(\epsilon_{ji}^{t})=\sigma'(\theta_{i}^{\star}-\theta_{j}^{\star})=\frac{e^{\theta_{i}^{\star}}e^{\theta_{j}^{\star}}}{(e^{\theta_{i}^{\star}}+e^{\theta_{j}^{\star}})^{2}},
\]
 and (ii) follows from the definition of $\bm{L}_{Lz}$. By Lemma~\ref{lemma:spectral},
\begin{equation}
\frac{m}{2np}\le(m-1)\lambda_{m-1}(\bm{L}_{Lz}^{\dagger})\le\mathrm{Trace}(\bm{L}_{Lz}^{\dagger})\le m\|\bm{L}_{Lz}^{\dagger}\|\le\frac{16\kappa_{1}\kappa_{2}m}{np}.\label{eq:trace_L_Lz}
\end{equation}
Moreover, $\{\epsilon_{ij}^{t}\}_{(i,j,t):i>j,(i,j)\in\mathcal{E}_{Y},L_{ij}^{t}=1}$
is a set of sub-Gaussian random variable with variance proxy $1/z_{ij}$
(see, e.g., Section~2.5 in \cite{vershynin2018high} for the definition
of sub-Gaussian random variable) and $1/z_{ij}\le4\kappa_{1}$ by
Lemma~\ref{lemma:z_range}. Applying Hanson-Wright inequality (see
\cite{rudelson2013hanson}), for any scalar $a>0$ we have
\begin{align}
 & \mathbb{P}\left[\left|\|\bm{L}_{Lz}^{\dagger}\bm{B}\widehat{\bm{\epsilon}}\|^{2}-\mathbb{E}\left[\|\bm{L}_{Lz}^{\dagger}\bm{B}\widehat{\bm{\epsilon}}\|^{2}\right]\right|>a\right]\label{eq:HW_LBepsilon}\\
 & \quad\le2\exp\left[-C_{2}\left(\frac{a^{2}}{\left(4\kappa_{1}\right)^{4}\|\bm{B}^{\top}\bm{L}_{Lz}^{\dagger}\bm{L}_{Lz}^{\dagger}\bm{B}\|_{\mathrm{F}}^{2}}\wedge\frac{a}{\left(4\kappa_{1}\right)^{2}\|\bm{B}^{\top}\bm{L}_{Lz}^{\dagger}\bm{L}_{Lz}^{\dagger}\bm{B}\|}\right)\right]\nonumber 
\end{align}
for some constant $C_{2}>0$. For $\|\bm{B}^{\top}\bm{L}_{Lz}^{\dagger}\bm{L}_{Lz}^{\dagger}\bm{B}\|_{\mathrm{F}}$
and $\|\bm{B}^{\top}\bm{L}_{Lz}^{\dagger}\bm{L}_{Lz}^{\dagger}\bm{B}\|$
we have 
\begin{align*}
\|\bm{B}^{\top}\bm{L}_{Lz}^{\dagger}\bm{L}_{Lz}^{\dagger}\bm{B}\| & \le\|\bm{B}^{\top}\bm{L}_{Lz}^{\dagger}\bm{L}_{Lz}^{\dagger}\bm{B}\|_{\mathrm{F}}\\
 & =\sqrt{\mathrm{Trace}\left(\bm{B}^{\top}\bm{L}_{Lz}^{\dagger}\bm{L}_{Lz}^{\dagger}\bm{B}\bm{B}^{\top}\bm{L}_{Lz}^{\dagger}\bm{L}_{Lz}^{\dagger}\bm{B}\right)}\\
 & =\sqrt{\mathrm{Trace}\left(\bm{L}_{Lz}^{\dagger}\bm{L}_{Lz}^{\dagger}\bm{B}\bm{B}^{\top}\bm{L}_{Lz}^{\dagger}\bm{L}_{Lz}^{\dagger}\bm{B}\bm{B}^{\top}\right)}\\
 & \overset{\text{(i)}}{=}\sqrt{\mathrm{Trace}\left(\bm{L}_{Lz}^{\dagger}\bm{L}_{Lz}^{\dagger}\right)}\\
 & \overset{\text{(ii)}}{\le}\frac{\sqrt{m}}{\lambda_{m-1}(\bm{L}_{Lz})}\overset{\text{(iii)}}{\le}\frac{16\kappa_{1}\kappa_{2}\sqrt{m}}{np}.
\end{align*}
Here (i) follows from the fact that $\bm{B}\bm{B}^{\top}=\bm{L}_{Lz}$,
(ii) follows from Lemma~\ref{lemma:spectral} and the fact that $\mathrm{Trace}(\bm{M})\le m\|\bm{M}\|$
for any $m\times m$ matrix $\bm{M}$, and (iii) follows from Lemma~\ref{lemma:spectral}.
Now substitute $a=C_{3}\kappa_{1}^{3}\kappa_{2}\sqrt{m}\log(n)/(np)$
in (\ref{eq:HW_LBepsilon}) for some large enough constant $C_{3}$.
We have that with probability at least $1-2n^{-10}$,
\begin{equation}
\left|\|\bm{L}_{Lz}^{\dagger}\bm{B}\widehat{\bm{\epsilon}}\|^{2}-\mathbb{E}\left[\|\bm{L}_{Lz}^{\dagger}\bm{B}\widehat{\bm{\epsilon}}\|^{2}\right]\right|\le\frac{C_{3}\kappa_{1}^{3}\kappa_{2}\sqrt{m}\log(n)}{np}.\label{eq:LBepsilon_concentration}
\end{equation}
Combining this with (\ref{eq:E_LBepsilon}) and (\ref{eq:trace_L_Lz}),
\begin{align*}
\left|\|\bm{L}_{Lz}^{\dagger}\bm{B}\widehat{\bm{\epsilon}}\|-\sqrt{\mathrm{Trace}(\bm{L}_{Lz}^{\dagger})}\right| & =\frac{\left|\|\bm{L}_{Lz}^{\dagger}\bm{B}\widehat{\bm{\epsilon}}\|^{2}-\mathrm{Trace}(\bm{L}_{Lz}^{\dagger})\right|}{\|\bm{L}_{Lz}^{\dagger}\bm{B}\widehat{\bm{\epsilon}}\|+\sqrt{\mathrm{Trace}(\bm{L}_{Lz}^{\dagger})}}\\
 & \le\frac{C_{3}\kappa_{1}^{3}\kappa_{2}\sqrt{m}\log(n)/(np)}{\sqrt{m/(2np)}}\\
 & \le C_{4}\kappa_{1}^{3}\kappa_{2}\sqrt{\frac{\log(n)}{np}}
\end{align*}
for some constant $C_{4}>0$.

Substituting (\ref{eq:LBepsilon_concentration}) and (\ref{eq:E_LBepsilon})
into (\ref{eq:delta_expansion}), we have that for some constant $C_{1},C_{2}$,

\begin{align}
\left|\|\widehat{\bm{\theta}}-\bm{\theta}^{\star}\|-\sqrt{\mathrm{Trace}(\bm{L}_{Lz}^{\dagger})}\right| & \le\left|\|\bm{L}_{Lz}^{\dagger}\bm{B}\widehat{\bm{\epsilon}}\|-\sqrt{\mathrm{Trace}(\bm{L}_{Lz}^{\dagger})}\right|+\|\bm{r}\|\nonumber \\
 & \le C_{4}\kappa_{1}^{3}\kappa_{2}\sqrt{\frac{\log(n)}{np}}+\frac{C_{1}\kappa_{1}^{6}\sqrt{m}\log^{2}(n)}{np}.\label{eq:z_L2_err}
\end{align}
The proof of (\ref{eq:rasch_l2}) is now completed. For (\ref{eq:rasch_l2_hat}),
the following lemma connects $\widehat{z}$ and $z$. The proof is
deferred to the end of this section.

\begin{lemma}\label{lemma:Lz_perturb} Instate the assumptions of
Theorem \ref{thm:rasch_infty}, then with probability at least $1-C_{5}n^{-10}$
for some constant $C_{5}>0$,

\[
\|\bm{L}_{Lz}^{\dagger}-\bm{L}_{L\hat{z}}^{\dagger}\|\le\frac{C_{6}\kappa_{1}^{7/2}\kappa_{2}^{2}}{(np)^{3/2}}
\]
for some large enough constant $C_{6}$.

\end{lemma}Combining this lemma with Weryl's inequality, we have
that 
\[
\left|\mathrm{Trace}(\bm{L}_{Lz}^{\dagger})-\mathrm{Trace}(\bm{L}_{L\hat{z}}^{\dagger})\right|\le\frac{C_{6}m\kappa_{1}^{7/2}\kappa_{2}^{2}}{(np)^{3/2}}.
\]
Then by (\ref{eq:trace_L_Lz}),
\begin{align*}
\left|\sqrt{\mathrm{Trace}(\bm{L}_{L\hat{z}}^{\dagger})}-\sqrt{\mathrm{Trace}(\bm{L}_{Lz}^{\dagger})}\right| & =\frac{\left|\mathrm{Trace}(\bm{L}_{Lz}^{\dagger})-\mathrm{Trace}(\bm{L}_{L\hat{z}}^{\dagger})\right|}{\sqrt{\mathrm{Trace}(\bm{L}_{Lz}^{\dagger})}+\sqrt{\mathrm{Trace}(\bm{L}_{L\hat{z}}^{\dagger})}}\\
 & =\frac{C_{6}m\kappa_{1}^{7/2}\kappa_{2}^{2}/(np)^{3/2}}{\sqrt{m/(2np)}}\le\frac{\sqrt{2}C_{6}\kappa_{1}^{7/2}\kappa_{2}^{2}\sqrt{m}}{np}.
\end{align*}
Using triangular inequality, we conclude that
\[
\left|\|\widehat{\bm{\theta}}-\bm{\theta}^{\star}\|-\sqrt{\mathrm{Trace}(\bm{L}_{L\hat{z}}^{\dagger})}\right|\le C_{4}\kappa_{1}^{3}\kappa_{2}\sqrt{\frac{\log(n)}{np}}+\frac{(C_{1}\kappa_{1}^{6}+\sqrt{2}C_{6}\kappa_{1}^{7/2}\kappa_{2}^{2})\sqrt{m}\log^{2}(n)}{np}.
\]

\paragraph{Proof of Lemma~\ref{lemma:Lz_perturb}.}

Recall $\sigma$ is the sigmoid function and its derivative $\sigma'$
is $1$-Lipschitz. By Theorem~\ref{thm:rasch_infty}, for all $(i,j)$,
\begin{align*}
\left|z_{ij}-\widehat{z}_{ij}\right| & =\left|\sigma'(\theta_{i}^{\star}-\theta_{j}^{\star})-\sigma'(\theta_{i}^{\star}-\theta_{j}^{\star})\right|\\
 & \le\left|\left(\widehat{\theta}_{i}-\widehat{\theta}_{j}\right)-\left(\theta_{i}^{\star}-\theta_{j}^{\star}\right)\right|\le C_{7}\kappa_{1}\kappa_{2}^{1/2}\sqrt{\frac{\log(n)}{np}}
\end{align*}
for some constant $C_{7}>0$. Then 
\begin{align*}
\|\bm{L}_{Lz}-\bm{L}_{L\hat{z}}\| & =\max_{\bm{v}\in\mathbb{R}^{m}:\|\bm{v}\|=1}\left|\bm{v}^{\top}\sum_{i>j:(i,j)\in\mathcal{E}_{Y}}L_{ij}\left(z_{ij}-\widehat{z}_{ij}\right)(\bm{e}_{i}-\bm{e}_{j})(\bm{e}_{i}-\bm{e}_{j})^{\top}\bm{v}\right|\\
 & \le\max_{\bm{v}\in\mathbb{R}^{m}:\|\bm{v}\|=1}\sum_{i>j:(i,j)\in\mathcal{E}_{Y}}\left|z_{ij}-\widehat{z}_{ij}\right|\bm{v}^{\top}L_{ij}(\bm{e}_{i}-\bm{e}_{j})(\bm{e}_{i}-\bm{e}_{j})^{\top}\bm{v}\\
 & \le C_{7}\kappa_{1}\kappa_{2}^{1/2}\sqrt{\frac{\log(n)}{np}}\cdot\|\bm{L}_{L}\|\\
 & \le3C_{7}\kappa_{1}\kappa_{2}^{1/2}\sqrt{np\log(n)},
\end{align*}
where the last line follows from Lemma~\ref{lemma:spectral}. As
$np\ge C_{8}\kappa_{1}^{4}\kappa_{2}^{3}\log^{2}(n)$ for some large
enough constant $C_{8}$, $\|\bm{L}_{Lz}-\bm{L}_{L\hat{z}}\|\le np/(32\kappa_{1}\kappa_{2})$.
By Weryl's inequality and Lemma~\ref{lemma:spectral},
\[
\lambda_{m-1}(\bm{L}_{L\hat{z}})\ge\lambda_{m-1}(\bm{L}_{Lz})-\|\bm{L}_{Lz}-\bm{L}_{L\hat{z}}\|\ge\frac{np}{16\kappa_{1}\kappa_{2}}-\frac{np}{32\kappa_{1}\kappa_{2}}=\frac{np}{32\kappa_{1}\kappa_{2}}.
\]
This implies that $\|\bm{L}_{Lz}^{\dagger}\|\le16\kappa_{1}\kappa_{2}/(np)$
and $\|\bm{L}_{L\hat{z}}^{\dagger}\|\le32\kappa_{1}\kappa_{2}/(np)$.
Using the perturbation bound of pseudo-inverse (see Theorem~4.1 in
\cite{wedin1973perturbation}), we have
\begin{align*}
\|\bm{L}_{Lz}^{\dagger}-\bm{L}_{L\hat{z}}^{\dagger}\| & \le3\cdot\|\bm{L}_{Lz}^{\dagger}\|\cdot\|\bm{L}_{L\hat{z}}^{\dagger}\|\cdot\|\bm{L}_{Lz}-\bm{L}_{L\hat{z}}\|\\
 & \le\frac{C_{9}\kappa_{1}^{3}\kappa_{2}^{5/2}}{(np)^{3/2}}
\end{align*}
for some constant $C_{9}>0$.

\section{Proofs for Section~\ref{subsec:Asymptotic_normality}}

In this section we prove the results for Section~\ref{subsec:Asymptotic_normality},
including Theorem~\ref{thm:asymp_normality}, Theorem~\ref{thm:asymp_normality_WPMLE},
Proposition~\ref{prop:quant_asymp_normality}, and (\ref{eq:Vsame_Vdiff}).
To facilitate our analysis, we decompose the loss function by each
user and random split. For \myalgW, recall that we let the loss function
associated with user $t$ to be 
% \[
% \mathcal{L}_{\mathrm{WP}}^{(t)}(\bm{\theta})\coloneqq-\sum_{\substack{(i,j):i>j\\
% (t,i),(t,j)\in\mathcal{G}_{X}
% }
% }\frac{\widetilde{m}_{t}}{m_{t}(m_{t}-1)}\left[\log\left(\frac{e^{\theta_{i}}}{e^{\theta_{i}}+e^{\theta_{j}}}\right)\mathds{1}\{X_{ti}>X_{tj}\}+\log\left(\frac{e^{\theta_{j}}}{e^{\theta_{i}}+e^{\theta_{j}}}\right)\mathds{1}\{X_{ti}<X_{tj}\}\right].
% \]
\begin{multline*}
\mathcal{L}_{\mathrm{WP}}^{(t)}(\bm{\theta})\coloneqq-\sum_{\substack{(i,j):i>j\\
(t,i),(t,j)\in\mathcal{G}_{X}
}
}\frac{\widetilde{m}_{t}}{m_{t}(m_{t}-1)}\left[\log\left(\frac{e^{\theta_{i}}}{e^{\theta_{i}}+e^{\theta_{j}}}\right)\mathds{1}\{X_{ti}>X_{tj}\}\right.\\
\left.\hspace{8em}{}+\log\left(\frac{e^{\theta_{j}}}{e^{\theta_{i}}+e^{\theta_{j}}}\right)\mathds{1}\{X_{ti}<X_{tj}\}\right].
\end{multline*}

We have also defined the loss function associated with random split
$k$ and user $t$ for \myalgM\ as 
\[
\mathcal{L}_{k}^{(t)}(\bm{\theta})=-\sum_{\substack{(i,j):i>j,\\
(i,j,t)\in\Omega_{k}
}
}\left[\log\left(\frac{e^{\theta_{i}}}{e^{\theta_{i}}+e^{\theta_{j}}}\right)\mathds{1}\{X_{ti}>X_{tj}\}+\log\left(\frac{e^{\theta_{j}}}{e^{\theta_{i}}+e^{\theta_{j}}}\right)\mathds{1}\{X_{ti}<X_{tj}\}\right],
\]
where $\Omega_{k}$ denotes the set of paired item-item-user tuples
for the $k$-th split. 

\subsection{Proof of Theorem~\ref{thm:asymp_normality} \label{subsec:Proof_asymp_MRPMLE}}

We first fix a random split $k$. By mean value theorem

\begin{align}
\frac{1}{n}\sum_{t=1}^{n}\nabla\mathcal{L}_{k}^{(t)}(\bm{\theta}^{\star}) & =\frac{1}{n}\sum_{t=1}^{n}\nabla\mathcal{L}_{k}^{(t)}(\widehat{\bm{\theta}}_{k}^{(n)})+\left[\int_{\tau=0}^{1}\frac{1}{n}\sum_{t=1}^{n}\nabla^{2}\mathcal{L}_{k}^{(t)}(\bm{\theta}^{\star}+\tau(\widehat{\bm{\theta}}_{k}^{(n)}-\bm{\theta}^{\star}))\mathrm{d}\tau\right](\bm{\theta}^{\star}-\widehat{\bm{\theta}}_{k}^{(n)})\nonumber \\
 & =\left[\int_{\tau=0}^{1}\frac{1}{n}\sum_{t=1}^{n}\nabla^{2}\mathcal{L}_{k}^{(t)}(\bm{\theta}^{\star}+\tau(\widehat{\bm{\theta}}_{k}^{(n)}-\bm{\theta}^{\star}))\mathrm{d}\tau\right](\bm{\theta}^{\star}-\widehat{\bm{\theta}}_{k}^{(n)}).\label{eq:mrpmle_mvt}
\end{align}
Here the second line comes from the optimality condition for $\widehat{\bm{\theta}}_{k}^{(n)}$. 

We first consider the Hessian part. We start by showing the almost
surely convergence of $\widehat{\bm{\theta}}_{k}^{(n)}$ and controlling
the $\ell_{\infty}$ norm of $\widehat{\bm{\theta}}_{k}^{(n)}$. By
Theorem~\ref{thm:rasch_infty}, we know that with probability at
least $1-O(n^{-10})$, $\|\widehat{\bm{\theta}}_{k}^{(n)}-\bm{\theta}^{\star}\|\le C\kappa_{1}\kappa_{2}^{1/2}\sqrt{\frac{m\log(n)}{np}}$
for some constant $C$. This and the Borel-Cantelli lemma implies
that $\widehat{\bm{\theta}}_{k}^{(n)}\overset{\mathrm{a.s.}}{\rightarrow}\bm{\theta}^{\star}$
and $\|\widehat{\bm{\theta}}_{k}^{(n)}\|_{\infty}\le2\kappa$ for
large enough $n$. Furthermore, observe that the Hessian
\[
\nabla^{2}\mathcal{L}_{k}^{(t)}(\bm{\theta})=\sum_{\substack{(i,j):i>j\\
(i,j,t)\in\Omega_{k}
}
}\frac{e^{\theta_{i}}e^{\theta_{j}}}{(e^{\theta_{i}}+e^{\theta_{j}})^{2}}(\bm{e}_{i}-\bm{e}_{j})(\bm{e}_{i}-\bm{e}_{j})^{\top}
\]
is Lipschitz-continuous for $\{\bm{\theta}:\|\bm{\theta}\|_{\infty}\le2\kappa\}$
with a uniform Lipschitz constant for all $k,t$. Then we conclude
that
\begin{equation}
\int_{\tau=0}^{1}\frac{1}{n}\sum_{t=1}^{n}\nabla^{2}\mathcal{L}_{k}^{(t)}(\bm{\theta}^{\star}+\tau(\widehat{\bm{\theta}}_{k}^{(n)}-\bm{\theta}^{\star}))\mathrm{d}\tau\overset{\mathrm{a.s.}}{\rightarrow}\frac{1}{n}\sum_{t=1}^{n}\nabla^{2}\mathcal{L}_{k}^{(t)}(\bm{\theta}^{\star}).\label{eq:Hessian_conv_groundtruth}
\end{equation}
Recall that the sampling, random pairing, and comparisons are all
independent between different users. Moreover, the Hessian and its
second moment are both bounded. Then we can invoke the strong law
of large number to say that 
\begin{equation}
\frac{1}{n}\sum_{t=1}^{n}\nabla^{2}\mathcal{L}_{k}^{(t)}(\bm{\theta}^{\star})\overset{\mathrm{a.s.}}{\rightarrow}\lim_{n\rightarrow\infty}\frac{1}{n}\sum_{t=1}^{n}\mathbb{E}_{\mathrm{s+r+d+c}}\nabla^{2}\mathcal{L}_{k}^{(t)}(\bm{\theta}^{\star})=\bm{H}^{\infty},\label{eq:H_k^infty}
\end{equation}
where $\bm{H}^{\infty}$ is defined in (\ref{eq:H_infty}). Note that
this equality holds trivially for $k=1$ and then by symmetry holds
all $k$. We also make the following claim. The proof is deferred
to the end of this section.

\begin{lemma}\label{lem:H_k}Instate the assumptions of Theorem~\ref{thm:rasch_infty},
we have that $\lambda_{m-1}\left(\bm{H}_{k}^{\infty}\right)>0$ and
$\bm{H}_{k}^{\infty\top}\bm{1}_{m}=\bm{0}_{m}$. 

\end{lemma}

Now we move to the gradient part. Recall that the gradient 
\begin{align*}
    &\mathbb{E}_{\mathrm{s+r+d+c}}\nabla\mathcal{L}_{k}^{(t)}(\bm{\theta})
    \\&=\mathbb{E}_{\mathrm{s+r+d+c}}\sum_{\substack{(i,j):i>j\\
(i,j,t)\in\Omega_{k}
}
}\left[\left(-\mathds{1}\{X_{ti}<X_{tj}\}+\frac{e^{\theta_{j}}}{e^{\theta_{i}}+e^{\theta_{j}}}\right)(\bm{e}_{i}-\bm{e}_{j})\mathds{1}\{X_{ti}\neq X_{tj}\}\right]
\\&=\bm{0}_{m}
\end{align*}
is a zero-mean random variable and independent across different $t$.
Recall that for all $k_{1}\in[n_{\mathrm{split}}]$,
\[
\lim_{n\rightarrow\infty}\frac{1}{n}\sum_{t=1}^{n}\mathbb{E}_{\mathrm{s+r+d+c}}\nabla\mathcal{L}_{k_{1}}^{(t)}(\bm{\theta}^{\star})\nabla\mathcal{L}_{k_{1}}^{(t)}(\bm{\theta}^{\star})^{\top}=\bm{V}_{\mathrm{same}}^{\infty}.
\]
Note that this equality holds trivially for $k=1$ and then by symmetry
holds all $k$. Similarly for all $k_{1}\neq k_{2}$, 
\[
\lim_{n\rightarrow\infty}\frac{1}{n}\sum_{t=1}^{n}\mathbb{E}_{\mathrm{s+r+d+c}}\nabla\mathcal{L}_{k_{1}}^{(t)}(\bm{\theta}^{\star})\nabla\mathcal{L}_{k_{2}}^{(t)}(\bm{\theta}^{\star})^{\top}=\bm{V}_{\mathrm{diff}}^{\infty}.
\]
Note that $\nabla\mathcal{L}_{k}^{(t)}(\bm{\theta}^{\star})$ is bounded
for all $k,t$. We can then invoke the central limit theorem for triangular
arrays (see Theorems~3.4.10 and 3.10.6 in \cite{Durrett2019}) to
reach that 
\begin{align*}
    &\frac{1}{\sqrt{n}}\sum_{t=1}^{n}\left[\frac{1}{n_{\mathrm{split}}}\sum_{k=1}^{n_{\mathrm{split}}}\nabla\mathcal{L}_{k}^{(t)}(\bm{\theta}^{\star})\right]\overset{\mathrm{d}}{\rightarrow}\\
    &\hspace{10em}\mathcal{N}\left(\bm{0},\frac{1}{n_{\mathrm{split}}^{2}}\sum_{k_{1}=1}^{n_{\mathrm{split}}}\sum_{k_{2}=1}^{n_{\mathrm{split}}}\lim_{n\rightarrow\infty}\frac{1}{n}\sum_{t=1}^{n}\mathbb{E}_{\mathrm{s+r+d+c}}\nabla\mathcal{L}_{k_{1}}^{(t)}(\bm{\theta}^{\star})\nabla\mathcal{L}_{k_{2}}^{(t)}(\bm{\theta}^{\star})^{\top}\right).
\end{align*}

We can then rewrite the above formula as 
\begin{equation}
\frac{1}{\sqrt{n}}\sum_{t=1}^{n}\left[\frac{1}{n_{\mathrm{split}}}\sum_{k=1}^{n_{\mathrm{split}}}\nabla\mathcal{L}_{k}^{(t)}(\bm{\theta}^{\star})\right]\overset{\mathrm{d}}{\rightarrow}\mathcal{N}\left(\bm{0},\frac{1}{n_{\mathrm{split}}}\bm{V}_{\mathrm{same}}^{\infty}+\frac{n_{\mathrm{split}}-1}{n_{\mathrm{split}}}\bm{V}_{\mathrm{diff}}^{\infty}\right).\label{eq:grad_CLT}
\end{equation}

We now combine the gradient and the Hessian parts together. By (\ref{eq:mrpmle_mvt}),
(\ref{eq:Hessian_conv_groundtruth}) and (\ref{eq:H_k^infty}), we
know that 
\[
\frac{1}{n}\sum_{t=1}^{n}\nabla\mathcal{L}_{k}^{(t)}(\bm{\theta}^{\star})-\bm{H}^{\infty}(\bm{\theta}^{\star}-\widehat{\bm{\theta}}_{k}^{(n)})\overset{\mathrm{a.s.}}{\rightarrow}\bm{0}.
\]
Recall that $\bm{\theta}^{\star\top}\bm{1}_{m}=\widehat{\bm{\theta}}_{k}^{(n)\top}\bm{1}_{m}=0$
by design and $\nabla\mathcal{L}_{k}^{(t)}(\bm{\theta}^{\star})^{\top}\bm{1}_{m}=0$
by definition. Combining this with Lemma~\ref{lem:H_k}, we can
take the pseudo-inverse of $\bm{H}^{\infty}$ to reach 
\[
\bm{\theta}^{\star}-\widehat{\bm{\theta}}_{k}^{(n)}\overset{\mathrm{a.s.}}{\rightarrow}\left(\bm{H}^{\infty}\right)^{\dagger}\left[\frac{1}{n}\sum_{t=1}^{n}\nabla\mathcal{L}_{k}^{(t)}(\bm{\theta}^{\star})\right].
\]
Now we average across all random splittings to conclude that 
\[
\widehat{\bm{\theta}}_{\mathrm{MRP}}^{(n)}-\bm{\theta}^{\star}\overset{\mathrm{a.s.}}{\rightarrow}-\left(\bm{H}^{\infty}\right)^{\dagger}\left[\frac{1}{n}\sum_{t=1}^{n}\frac{1}{n_{\mathrm{split}}}\sum_{k=1}^{n_{\mathrm{split}}}\nabla\mathcal{L}_{k}^{(t)}(\bm{\theta}^{\star})\right].
\]
Combining this with (\ref{eq:grad_CLT}), we have that 
\[
\sqrt{n}\left(\widehat{\bm{\theta}}_{\mathrm{MRP}}^{(n)}-\bm{\theta}^{\star}\right)\overset{\mathrm{d}}{\rightarrow}\mathcal{N}\left(\bm{0},\left(\bm{H}^{\infty}\right)^{\dagger}\left[\frac{1}{n_{\mathrm{split}}}\bm{V}_{\mathrm{same}}^{\infty}+\frac{n_{\mathrm{split}}-1}{n_{\mathrm{split}}}\bm{V}_{\mathrm{diff}}^{\infty}\right]\left(\bm{H}^{\infty}\right)^{\dagger}\right).
\]

\paragraph{Proof of Lemma~\ref{lem:H_k}.}

It suffices to show that $\lambda_{m-1}(\sum_{t=1}^{n}\nabla^{2}\mathcal{L}_{k}^{(t)}(\bm{\theta}^{\star}))\ge\beta n$
for some $\beta>0$ that does not depend on $n$. Recall that 
\begin{align*}
\sum_{t=1}^{n}\nabla^{2}\mathcal{L}_{k}^{(t)}(\bm{\theta}^{\star}) & =\sum_{t=1}^{n}\sum_{\substack{(i,j):i>j\\
(i,j,t)\in\Omega_{k}
}
}\left[\left(-\frac{e^{\theta_{i}^{\star}+\theta_{j}^{\star}}}{(e^{\theta_{i}^{\star}}+e^{\theta_{j}^{\star}})^{2}}\right)(\bm{e}_{i}-\bm{e}_{j})(\bm{e}_{i}-\bm{e}_{j})^{\top}\mathds{1}\{X_{ti}\neq X_{tj}\}\right]\\
 & =\sum_{t=1}^{n}\sum_{\substack{(i,j):i>j\\
(i,j,t)\in\Omega_{k}
}
}\left[-z_{ij}(\bm{e}_{i}-\bm{e}_{j})(\bm{e}_{i}-\bm{e}_{j})^{\top}\mathds{1}\{X_{ti}\neq X_{tj}\}\right].
\end{align*}
Invoking Lemma~\ref{lemma:spectral}, we have that with probability
at least $1-O(n^{-10})$, $\lambda_{m-1}(\sum_{t=1}^{n}\nabla^{2}\mathcal{L}_{k}^{(t)}(\bm{\theta}))\ge np/(16\kappa_{1}\kappa_{2})$,
so we are done.

\subsection{Proof for Theorem~\ref{thm:asymp_normality_WPMLE} \label{subsec:Proof_asymp_WPMLE}}

The proof for the asymptotic normality of \myalgW\ is very similar
to the proof for \myalgM. We start with a lemma on the asymptotic
consistency of \myalgW.

\begin{lemma}\label{lem:consistency_WPMLE}Instate the assumptions
of Theorem~\ref{thm:rasch_infty}. Then as $n\rightarrow\infty$,
\[
\widehat{\bm{\theta}}_{\mathrm{WP}}^{(n)}\overset{\mathrm{p}}{\rightarrow}\bm{\theta}^{\star}.
\]

\end{lemma}Recall that in the assumption of Theorem~\ref{thm:asymp_normality},
we do not specified the limit of Hessian and gradient of the loss
function of the \myalgW. In the following lemma, we show that they
are implied by the limit of Hessian and gradient of the \myalgM\
loss function.

\begin{lemma}\label{lem:WPMLE_Hess_grad}Let $\bm{H}^{\infty}$ and
$\bm{V}_{\mathrm{diff}}^{\infty}$ be defined as in Theorem~\ref{thm:asymp_normality}.
Then we have that 
\begin{equation}
\lim_{n\rightarrow\infty}\frac{1}{n}\sum_{t=1}^{n}\mathbb{E}_{\mathrm{s+d+c}}\nabla^{2}\mathcal{L}_{\mathrm{WP}}^{(t)}(\bm{\theta}^{\star})=\bm{H}^{\infty}\label{eq:WPMLE_Hes_lim}
\end{equation}
and
\begin{equation}
\lim_{n\rightarrow\infty}\frac{1}{n}\sum_{t=1}^{n}\mathbb{E}_{\mathrm{s+d+c}}\nabla\mathcal{L}_{\mathrm{WP}}^{(t)}(\bm{\theta})\nabla\mathcal{L}_{\mathrm{WP}}^{(t)}(\bm{\theta})^{\top}=\bm{V}_{\mathrm{diff}}^{\infty}.\label{eq:WPMLE_grad_lim}
\end{equation}

\end{lemma}The proof of these two lemmas are deferred to Section~\ref{subsec:proof_WPMLE_consistent}
and Section~\ref{subsec:proof_WPMLE_Hes}. 

The rest of the proof for the asymptotic normality of \myalgW\ is
identical to the proof of \myalgM. For conciseness we omit it here.

\subsubsection{Proof of Lemma~\ref{lem:consistency_WPMLE}\label{subsec:proof_WPMLE_consistent}}

Recall that 
\begin{align*}
\frac{1}{n}\mathcal{L}_{\mathrm{WP}}(\bm{\theta}) & =\frac{1}{n}\sum_{t=1}^{n}\mathcal{L}_{\mathrm{WP}}^{(t)}(\bm{\theta})\\
& =-\frac{1}{n}\sum_{t=1}^{n}\sum_{\substack{(i,j):i>j\\
(t,i),(t,j)\in\mathcal{G}_{X}
}
}\frac{\widetilde{m}_{t}}{m_{t}(m_{t}-1)}\left[\log\left(\frac{e^{\theta_{i}}}{e^{\theta_{i}}+e^{\theta_{j}}}\right)\mathds{1}\{X_{ti}>X_{tj}\}\right.\\
& \hspace{15em}{}\left.+\log\left(\frac{e^{\theta_{j}}}{e^{\theta_{i}}+e^{\theta_{j}}}\right)\mathds{1}\{X_{ti}<X_{tj}\}\right].
\end{align*}
Let $\bm{\Theta}\coloneqq\{\bm{\theta}\in\mathbb{R}^{m}:\bm{1}_{m}^{\top}\bm{\theta}=0,\|\bm{\theta}-\bm{\theta}^{\star}\|_{\infty}\le10\}$.
Let $\widetilde{\bm{\theta}}_{\mathrm{WP}}^{(n)}\coloneqq\arg\min_{\bm{\theta}\in\bm{\Theta}}\frac{1}{n}\sum_{t=1}^{n}\mathcal{L}_{\mathrm{WP}}^{(t)}(\bm{\theta})$.
We make the following claims.
\begin{enumerate}
\item As $n\rightarrow\infty$, the convergence of 
\[
\frac{1}{n}\sum_{t=1}^{n}\mathbb{E}_{\mathrm{s+d+c}}\mathcal{L}_{\mathrm{WP}}^{(t)}(\bm{\theta})\rightarrow\overline{\mathcal{L}}_{\mathrm{WP}}(\bm{\theta})
\]
is uniform on $\bm{\Theta}$. 
\item Property 1 implies that $\frac{1}{n}\mathcal{L}_{\mathrm{WP}}(\bm{\theta})$
converges uniformly in probability to $\overline{\mathcal{L}}_{\mathrm{WP}}(\bm{\theta})$
(see \cite{Newey1994} for the definition of this type of convergence).
\item $\bm{\theta}^{\star}$ is the unique minimizer of $\overline{\mathcal{L}}_{\mathrm{WP}}(\bm{\theta})$.
\end{enumerate}
It is obvious that $\bm{\theta}^{\star}\in\bm{\Theta}$. Then with
these properties, we can invoke Theorem~2.1 in \cite{Newey1994}
to conclude that $\widetilde{\bm{\theta}}_{\mathrm{WP}}^{(n)}$ is
consistent, i.e., $\widetilde{\bm{\theta}}_{\mathrm{WP}}^{(n)}\overset{\mathrm{p}}{\rightarrow}\bm{\theta}^{\star}$
as $n\rightarrow\infty$. This further shows that $\|\widetilde{\bm{\theta}}_{\mathrm{WP}}^{(n)}-\bm{\theta}^{\star}\|\le5$
with probability at least $1-n^{-10}$ for large enough $n$, and
therefore it is a global minimizer in $\bm{\Theta}$ and a local minimizer
in $\{\bm{\theta}\in\mathbb{R}^{m}:\bm{1}_{m}^{\top}\bm{\theta}=0\}$.
Similar to the case of \myalgM\ (see Lemma~\ref{lem:H_k}, we omit
the full proof here for conciseness), $\sum_{t=1}^{n}\mathcal{L}_{\mathrm{WP}}^{(t)}(\bm{\theta})$
is strictly convex (modulo $\bm{1}_{m}^{\top}\bm{\theta}=0$) for
large enough $n$. Then $\widetilde{\bm{\theta}}_{\mathrm{WP}}^{(n)}$
being a local minimum means that it is also the global minimum, i.e.,
and $\widetilde{\bm{\theta}}_{\mathrm{WP}}^{(n)}=\widehat{\bm{\theta}}_{\mathrm{WP}}^{(n)}$
a.s. with probability at least $1-n^{-10}$ for large enough $n$.
Then, we have that $\widetilde{\bm{\theta}}_{\mathrm{WP}}^{(n)}\overset{\mathrm{p}}{\rightarrow}\bm{\theta}^{\star}$
as $n\rightarrow\infty$ implies $\widehat{\bm{\theta}}_{\mathrm{WP}}^{(n)}\overset{\mathrm{p}}{\rightarrow}\bm{\theta}^{\star}$
as $n\rightarrow\infty$.

We now prove these properties in order.

\paragraph{Proof of Claim 1.}

Observe that for any $\theta\in\bm{\Theta}$,
\begin{align*}
    &\left|\left[\log\left(\frac{e^{\theta_{i}}}{e^{\theta_{i}}+e^{\theta_{j}}}\right)\mathds{1}\{X_{ti}>X_{tj}\}+\log\left(\frac{e^{\theta_{j}}}{e^{\theta_{i}}+e^{\theta_{j}}}\right)\mathds{1}\{X_{ti}<X_{tj}\}\right]\right|\\
    &\hspace{20em}\le-\log\left(\frac{1}{e^{\kappa_{1}+20}+1}\right)\le\kappa_{1}+21.
\end{align*}

Then 
\begin{align}
\left|\frac{1}{n}\mathcal{L}_{\mathrm{WP}}(\bm{\theta})\right| & \le\frac{1}{n}\sum_{t=1}^{n}\sum_{\substack{(i,j):i>j\\
(t,i),(t,j)\in\mathcal{G}_{X}
}
}\frac{\widetilde{m}_{t}}{m_{t}(m_{t}-1)}\cdot\nonumber\\
&\qquad\qquad\left|\log\left(\frac{e^{\theta_{i}}}{e^{\theta_{i}}+e^{\theta_{j}}}\right)\mathds{1}\{X_{ti}>X_{tj}\}  +\log\left(\frac{e^{\theta_{j}}}{e^{\theta_{i}}+e^{\theta_{j}}}\right)\mathds{1}\{X_{ti}<X_{tj}\}\right|\nonumber \\
 & \le\frac{1}{n}\sum_{t=1}^{n}\sum_{\substack{(i,j):i>j\\
(t,i),(t,j)\in\mathcal{G}_{X}
}
}\frac{\widetilde{m}_{t}(\kappa_{1}+21)}{m_{t}(m_{t}-1)}\nonumber \\
 & \le\frac{1}{n}\sum_{t=1}^{n}\binom{m_{t}}{2}\frac{\kappa_{1}+21}{m_{t}-1}\le\frac{m_{t}(\kappa_{1}+21)}{2}\le\frac{m(\kappa_{1}+21)}{2}.\label{eq:L_boundedness}
\end{align}
The last line holds since there can at most be $m_{t}$ pairs of responses
for each user $t$. Therefore $\frac{1}{n}\mathcal{L}_{\mathrm{WP}}(\bm{\theta})$
is uniformly bounded. Similarly, we can show that it is also uniformly
Lipschitz. Then by Arzel\`a--Ascoli theorem, the pointwise convergence
implies that the convergence is uniform on the compact set $\bm{\Theta}$.

\paragraph{Proof of Claim 2.}

By Hoeffding's inequality and (\ref{eq:L_boundedness}), for any $\bm{\theta}$,
\[
\mathbb{P}\left[\left|\frac{1}{n}\sum_{t=1}^{n}\mathcal{L}_{\mathrm{WP}}^{(t)}(\bm{\theta})-\frac{1}{n}\sum_{t=1}^{n}\mathbb{E}_{\mathrm{s+d+c}}\mathcal{L}_{\mathrm{WP}}^{(t)}(\bm{\theta})\right|>\epsilon\right]\le2\exp\left(-\frac{\epsilon^{2}}{nm(\kappa_{1}+1)}\right).
\]
Let $\bm{\Theta}_{\delta}\subset\bm{\Theta}$ such that $\bm{\Theta}_{\delta}$
is a minimal $\delta$-covering of $\bm{\Theta}$ with respect to
$\|\cdot\|_{\infty}$. It is well known that $|\bm{\Theta}_{\delta}|\le\mathcal{N}(\bm{\Theta},\|\cdot\|_{\infty},\delta/2)$
where $\mathcal{N}(\bm{\Theta},\|\cdot\|_{\infty},\delta/2)$ is the
packing number. By volume argument, 
\[
|\bm{\Theta}_{\delta}|\le\mathcal{N}(\bm{\Theta},\|\cdot\|_{\infty},\delta/2)\le\left(\frac{10}{\delta/2}\right)^{m-1}.
\]
We claim that 
\begin{equation}
\sup_{\bm{\theta}\in\bm{\Theta}}\inf_{\bm{\theta}'\in\bm{\Theta}_{\delta}}\left|\frac{1}{n}\sum_{t=1}^{n}\mathcal{L}_{\mathrm{WP}}^{(t)}(\bm{\theta})-\frac{1}{n}\sum_{t=1}^{n}\mathcal{L}_{\mathrm{WP}}^{(t)}(\bm{\theta}')\right|\le\frac{m(\kappa_{1}+20)}{2}\cdot\delta.\label{eq:ep_net}
\end{equation}
The proof is deferred to the end of this section. Take $\delta=2\epsilon/(m(\kappa_{1}+20))$.
We have that 
\begin{equation}
\sup_{\bm{\theta}\in\bm{\Theta}}\inf_{\bm{\theta}'\in\bm{\Theta}_{\delta}}\left|\frac{1}{n}\sum_{t=1}^{n}\mathcal{L}_{\mathrm{WP}}^{(t)}(\bm{\theta})-\frac{1}{n}\sum_{t=1}^{n}\mathcal{L}_{\mathrm{WP}}^{(t)}(\bm{\theta}')\right|\le\frac{m(\kappa_{1}+20)}{2}\cdot\delta=\epsilon.\label{eq:ep_net_epsilon}
\end{equation}
 Now for any $\epsilon,$ by union bound 
\begin{equation}
\mathbb{P}\left[\sup_{\bm{\theta}\in\bm{\Theta}_{\delta}}\left|\frac{1}{n}\sum_{t=1}^{n}\mathcal{L}_{\mathrm{WP}}^{(t)}(\bm{\theta})-\frac{1}{n}\sum_{t=1}^{n}\mathbb{E}_{\mathrm{s+r+d+c}}\mathcal{L}_{\mathrm{WP}}^{(t)}(\bm{\theta})\right|>\epsilon\right]\le2\left(\frac{10}{\delta/2}\right)^{m-1}\exp\left(-\frac{2\epsilon^{2}}{nm^{2}(\kappa_{1}+21)}\right).\label{eq:unif_prob_conv_expectation}
\end{equation}
Moreover, by assumption, 
\begin{equation}
\sup_{\bm{\theta}\in\bm{\Theta}_{\delta}}\left|\frac{1}{n}\sum_{t=1}^{n}\mathbb{E}_{\mathrm{s+r+d+c}}\mathcal{L}_{\mathrm{WP}}^{(t)}(\bm{\theta})-\overline{\mathcal{L}}_{\mathrm{WP}}(\bm{\theta})\right|\le\epsilon.\label{eq:unif_conv_assump}
\end{equation}
for large enough $n$. Combining (\ref{eq:ep_net_epsilon}), (\ref{eq:unif_prob_conv_expectation}),
(\ref{eq:unif_conv_assump}), we may conclude that 
\[
\mathbb{P}\left[\sup_{\bm{\theta}\in\bm{\Theta}}\left|\frac{1}{n}\sum_{t=1}^{n}\mathcal{L}_{\mathrm{WP}}^{(t)}(\bm{\theta})-\overline{\mathcal{L}}_{\mathrm{WP}}(\bm{\theta})\right|>3\epsilon\right]\le2\left(\frac{10}{\delta/2}\right)^{m-1}\exp\left(-\frac{2\epsilon^{2}}{nm^{2}(\kappa_{1}+21)}\right)
\]
for large enough $n$, which implies that $\frac{1}{n}\mathcal{L}_{\mathrm{WP}}(\bm{\theta})$
converges uniformly in probability to $\overline{\mathcal{L}}_{\mathrm{WP}}(\bm{\theta})$
on $\bm{\Theta}$.

\paragraph{Proof of Claim 3.}

We first show that $\bm{\theta}^{\star}$ is the minimizer for $\mathbb{E}_{\mathrm{c}}\mathcal{L}_{\mathrm{WP}}^{(t)}(\bm{\theta})$.
Compute $\mathbb{E}_{\mathrm{c}}\mathcal{L}_{\mathrm{WP}}^{(t)}(\bm{\theta})$,
we have that 
\[
\mathbb{E}_{\mathrm{c}}\mathcal{L}_{\mathrm{WP}}^{(t)}(\bm{\theta})=-\sum_{\substack{(i,j):i>j\\
(t,i),(t,j)\in\mathcal{G}_{X}
}
}\frac{\widetilde{m}_{t}}{m_{t}(m_{t}-1)}\left[\log\left(\frac{e^{\theta_{i}}}{e^{\theta_{i}}+e^{\theta_{j}}}\right)\frac{e^{\theta_{i}^{\star}}}{e^{\theta_{i}^{\star}}+e^{\theta_{j}^{\star}}}+\log\left(\frac{e^{\theta_{j}}}{e^{\theta_{i}}+e^{\theta_{j}}}\right)\frac{e^{\theta_{j}^{\star}}}{e^{\theta_{i}^{\star}}+e^{\theta_{j}^{\star}}}\right].
\]
Observe that $\log(x)p+\log(1-x)(1-p)$ as a function of $x$ is minimized
at $x=p$, so $\bm{\theta}^{\star}$ is the minimizer for all terms
inside $[\cdot]$ and $\mathbb{E}_{\mathrm{c}}\mathcal{L}_{\mathrm{WP}}^{(t)}(\bm{\theta})$
itself. 

As $\bm{\theta}^{\star}$ is the minimizer for $\mathbb{E}_{\mathrm{c}}\mathcal{L}_{\mathrm{WP}}^{(t)}(\bm{\theta})$,
$\bm{\theta}^{\star}$ is also a minimizer of $\mathbb{E}_{\mathrm{s+r+d+c}}\mathcal{L}_{\mathrm{WP}}^{(t)}(\bm{\theta})$
and 
\[
\overline{\mathcal{L}}_{\mathrm{WP}}(\bm{\theta})=\lim_{n\rightarrow\infty}\frac{1}{n}\sum_{t=1}^{n}\mathbb{E}_{\mathrm{s+r+d+c}}\mathcal{L}_{\mathrm{WP}}^{(t)}(\bm{\theta}).
\]
The uniqueness comes from the strict convexity of $\sum_{t=1}^{n}\mathbb{E}_{\mathrm{s+r+d+c}}\mathcal{L}_{\mathrm{WP}}^{(t)}(\bm{\theta})$.
This is similar to the case in \myalgM\ (see Lemma~\ref{lem:H_k},
we omit the full proof here for conciseness). 

\paragraph{Proof of (\ref{eq:ep_net}).}

Since $\bm{\Theta}_{\delta}$ is a $\delta$-covering, for any $\bm{\theta}\in\bm{\Theta}$,
there exists $\bm{\theta}'\in\bm{\Theta}_{\delta}$ such that $\|\bm{\theta}-\bm{\theta}'\|_{\infty}\le\delta$.
Observe that $x\mapsto1/(1+e^{-x})$ as a function is $(\kappa_{1}+20)$-Lipschitz
if $|x|\le\kappa_{1}+20$. Then similar to (\ref{eq:L_boundedness})
\[
\left|\frac{1}{n}\sum_{t=1}^{n}\mathcal{L}_{\mathrm{WP}}^{(t)}(\bm{\theta})-\frac{1}{n}\sum_{t=1}^{n}\mathcal{L}_{\mathrm{WP}}^{(t)}(\bm{\theta}')\right|\le\frac{m(\kappa_{1}+20)}{2}\cdot\delta.
\]

\subsubsection{Proof of Lemma~\ref{lem:WPMLE_Hess_grad}\label{subsec:proof_WPMLE_Hes} }

Recall that the loss function for \myalgW\ associated with user $t$
is 
\begin{multline*}
\mathcal{L}_{\mathrm{WP}}^{(t)}(\bm{\theta})=-\sum_{\substack{(i,j):i>j\\
(t,i),(t,j)\in\mathcal{G}_{X}
}
}\frac{\widetilde{m}_{t}}{m_{t}(m_{t}-1)}\left[\log\left(\frac{e^{\theta_{i}}}{e^{\theta_{i}}+e^{\theta_{j}}}\right)\mathds{1}\{X_{ti}>X_{tj}\}\right.\\
\left.\hspace{8em}{}+\log\left(\frac{e^{\theta_{j}}}{e^{\theta_{i}}+e^{\theta_{j}}}\right)\mathds{1}\{X_{ti}<X_{tj}\}\right],
\end{multline*}
and the loss function for \myalgM\ associated with user $t$ is 
\[
\mathcal{L}_{k}^{(t)}(\bm{\theta})=-\sum_{\substack{(i,j):i>j\\
(i,j,t)\in\Omega_{k}
}
}\left[\log\left(\frac{e^{\theta_{i}}}{e^{\theta_{i}}+e^{\theta_{j}}}\right)\mathds{1}\{X_{ti}>X_{tj}\}+\log\left(\frac{e^{\theta_{j}}}{e^{\theta_{i}}+e^{\theta_{j}}}\right)\mathds{1}\{X_{ti}<X_{tj}\}\right].
\]
As the random splitting only affect $\Omega_{k}$, we can see that
\begin{align*}
&\mathbb{E}_{\mathrm{r}}\mathcal{L}_{k}^{(t)}(\bm{\theta})\\
& =-\sum_{\substack{(i,j):i>j\\
(t,i),(t,j)\in\mathcal{G}_{X}
}
}\mathbb{E}_{\mathrm{r}}\mathds{1}_{\{(i,j,t)\in\Omega_{k}\}}\left[\log\left(\frac{e^{\theta_{i}}}{e^{\theta_{i}}+e^{\theta_{j}}}\right)\mathds{1}\{X_{ti}>X_{tj}\}+\log\left(\frac{e^{\theta_{j}}}{e^{\theta_{i}}+e^{\theta_{j}}}\right)\mathds{1}\{X_{ti}<X_{tj}\}\right]\\
 & =-\sum_{\substack{(i,j):i>j\\
(t,i),(t,j)\in\mathcal{G}_{X}
}
}\frac{\widetilde{m}_{t}}{m_{t}(m_{t}-1)}\left[\log\left(\frac{e^{\theta_{i}}}{e^{\theta_{i}}+e^{\theta_{j}}}\right)\mathds{1}\{X_{ti}>X_{tj}\}+\log\left(\frac{e^{\theta_{j}}}{e^{\theta_{i}}+e^{\theta_{j}}}\right)\mathds{1}\{X_{ti}<X_{tj}\}\right]\\
 & =\mathcal{L}_{\mathrm{WP}}^{(t)}(\bm{\theta}).
\end{align*}
Here 
\[
\mathbb{E}_{\mathrm{r}}\mathds{1}_{\{(i,j,t)\in\Omega_{k}\}}=\frac{\widetilde{m}_{t}}{m_{t}(m_{t}-1)}
\]
comes from the fact that we select $\widetilde{m}_{t}/2$ pairs out
of $m_{t}(m_{t}-1)/2$ total pairs and each pair is equally likely
to be selected due to symmetry. Similarly, we can deduct the same
thing for the gradient and the Hessian, i.e., for any $k\in[n_{\mathrm{split}}]$,

\begin{subequations}
\begin{align}
\mathbb{E}_{\mathrm{r}}\nabla\mathcal{L}_{k}^{(t)}(\bm{\theta}) & =\nabla\mathcal{L}_{\mathrm{WP}}^{(t)}(\bm{\theta});\label{eq:grad_WPMLE_equiv}\\
\mathbb{E}_{\mathrm{r}}\nabla^{2}\mathcal{L}_{k}^{(t)}(\bm{\theta}) & =\nabla^{2}\mathcal{L}_{\mathrm{WP}}^{(t)}(\bm{\theta}).\label{eq:Hes_WPMLE_equiv}
\end{align}

\end{subequations}

For the gradient, note that each random splitting is independent,
so we have that 
\begin{align*}
\mathbb{E}_{\mathrm{r}}\nabla\mathcal{L}_{1}^{(t)}(\bm{\theta}^{\star})\nabla\mathcal{L}_{2}^{(t)}(\bm{\theta}^{\star})^{\top} & =\left[\mathbb{E}_{\mathrm{r}}\nabla\mathcal{L}_{1}^{(t)}(\bm{\theta}^{\star})\right]\left[\mathbb{E}_{\mathrm{r}}\nabla\mathcal{L}_{2}^{(t)}(\bm{\theta}^{\star})\right]^{\top}\\
 & =\nabla\mathcal{L}_{\mathrm{WP}}^{(t)}(\bm{\theta}^{\star})\nabla\mathcal{L}_{\mathrm{WP}}^{(t)}(\bm{\theta}^{\star})^{\top}
\end{align*}
Then (\ref{eq:grad_WPMLE_equiv}) implies that 
\[
\lim_{n\rightarrow\infty}\frac{1}{n}\sum_{t=1}^{n}\mathbb{E}_{\mathrm{s+d+c}}\nabla\mathcal{L}_{\mathrm{WP}}^{(t)}(\bm{\theta})\nabla\mathcal{L}_{\mathrm{WP}}^{(t)}(\bm{\theta})^{\top}=\lim_{n\rightarrow\infty}\frac{1}{n}\sum_{t=1}^{n}\mathbb{E}_{\mathrm{s+r+d+c}}\nabla\mathcal{L}_{\mathrm{1}}^{(t)}(\bm{\theta})\nabla\mathcal{L}_{\mathrm{2}}^{(t)}(\bm{\theta})^{\top}=\bm{V}_{\mathrm{diff}}^{\infty}.
\]
For the Hessian, (\ref{eq:Hes_WPMLE_equiv}) implies that 
\[
\lim_{n\rightarrow\infty}\frac{1}{n}\sum_{t=1}^{n}\mathbb{E}_{\mathrm{s+d+c}}\nabla^{2}\mathcal{L}_{\mathrm{WP}}^{(t)}(\bm{\theta}^{\star})=\lim_{n\rightarrow\infty}\frac{1}{n}\sum_{t=1}^{n}\mathbb{E}_{\mathrm{s+r+d+c}}\nabla^{2}\mathcal{L}_{k}^{(t)}(\bm{\theta}^{\star})=\bm{H}^{\infty}.
\]

\subsection{Proof of Proposition~\ref{prop:quant_asymp_normality}\label{subsec:proof_special_case}}

Consider the expectation of the Hessian and the covariance, we claim
that the gradient and Hessian of this special case has expectation
\[
\bm{H}_{\pi}^{\infty}\coloneqq\mathbb{E}_{\mathrm{s+r+d+c+u}}\nabla^{2}\mathcal{L}_{1}^{(1)}(\bm{\theta}^{\star})=\frac{\beta mp}{4(m-1)}\left[\bm{I}_{m}-\frac{1}{m}\bm{1}_{m}\bm{1}_{m}^{\top}\right];
\]
\[
\bm{V}_{\mathrm{same},\pi}^{\infty}\coloneqq\mathbb{E}_{\mathrm{s+r+d+c+u}}\nabla\mathcal{L}_{1}^{(1)}(\bm{\theta}^{\star})\nabla\mathcal{L}_{1}^{(1)}(\bm{\theta}^{\star})^{\top}=\frac{\beta mp}{2(m-1)}\left[\bm{I}_{m}-\frac{1}{m}\bm{1}_{m}\bm{1}_{m}^{\top}\right];
\]
\[
\bm{V}_{\mathrm{diff},\pi}^{\infty}\coloneqq\mathbb{E}_{\mathrm{s+r+d+c+u}}\nabla\mathcal{L}_{1}^{(1)}(\bm{\theta}^{\star})\nabla\mathcal{L}_{2}^{(1)}(\bm{\theta}^{\star})^{\top}=\frac{\beta m^{2}p^{2}}{4(m-1)(mp-1)}\left[\bm{I}_{m}-\frac{1}{m}\bm{1}_{m}\bm{1}_{m}^{\top}\right].
\]

The rest of the proof of this corollary consists of two parts. In
the first part, we verify that with probability 1, the conditions
in Theorem~\ref{thm:asymp_normality} holds. In the second part,
we compute the asymptotic covariance explicitly. We then invoke Theorem~\ref{thm:asymp_normality}
to conclude that with probability 1, 

\begin{align*}
\sqrt{n}\left(\widehat{\bm{\theta}}_{\mathrm{MRP}}^{(n)}-\bm{\theta}^{\star}\right) & \overset{\mathrm{d}}{\rightarrow}\mathcal{N}\left(\bm{0},\left(\bm{H}_{\pi}^{\infty}\right)^{\dagger}\left[\frac{1}{n_{\mathrm{split}}}\bm{V}_{\mathrm{same},\pi}^{\infty}+\frac{n_{\mathrm{split}}-1}{n_{\mathrm{split}}}\bm{V}_{\mathrm{diff},\pi}^{\infty}\right]\left(\bm{H}_{\pi}^{\infty}\right)^{\dagger}\right)\\
 & =\mathcal{N}\left(\bm{0},\frac{8(m-1)}{\beta mp}\left(\frac{1}{n_{\mathrm{split}}}+\frac{n_{\mathrm{split}}-1}{n_{\mathrm{split}}}\cdot\frac{mp}{2(mp-1)}\right)\left[\bm{I}_{m}-\frac{1}{m}\bm{1}_{m}\bm{1}_{m}^{\top}\right]\right)
\end{align*}
and 
\begin{align*}
\sqrt{n}\left(\widehat{\bm{\theta}}_{\mathrm{WP}}^{(n)}-\bm{\theta}^{\star}\right) & \overset{\mathrm{d}}{\rightarrow}\mathcal{N}\left(\bm{0},\left(\bm{H}_{\pi}^{\infty}\right)^{\dagger}\bm{V}_{\mathrm{diff},\pi}^{\infty}\left(\bm{H}_{\pi}^{\infty}\right)^{\dagger}\right)\\
 & =\mathcal{N}\left(\bm{0},\frac{8(m-1)}{\beta mp}\left(\frac{mp}{2(mp-1)}\right)\left[\bm{I}_{m}-\frac{1}{m}\bm{1}_{m}\bm{1}_{m}^{\top}\right]\right)
\end{align*}

\subsubsection{Verifying the condition of Theorem~\ref{thm:asymp_normality} and
\ref{thm:asymp_normality_WPMLE}}

Consider the terms in (\ref{eq:H_infty})
\[
\mathbb{E}_{\mathrm{s+r+d+c}}\nabla^{2}\mathcal{L}_{1}^{(t)}(\bm{\theta}^{\star})
\]
as a function of the random variable $\delta_{t}^{\star}$. It has
mean $\mathbb{E}_{\mathrm{s+r+d+c+u}}\nabla^{2}\mathcal{L}_{1}^{(t)}(\bm{\theta}^{\star})$.
Since it is bounded, by strong law of large number, as $n\rightarrow\infty$
\[
\frac{1}{n}\sum_{t=1}^{n}\mathbb{E}_{\mathrm{s+r+d+c}}\nabla^{2}\mathcal{L}_{1}^{(t)}(\bm{\theta}^{\star})\overset{\mathrm{a.s.}}{\rightarrow}\mathbb{E}_{\mathrm{s+r+d+c+u}}\nabla^{2}\mathcal{L}_{1}^{(t)}(\bm{\theta}^{\star})=\bm{H}_{\pi}^{\infty}.
\]
Similarly we have that 
\[
\frac{1}{n}\sum_{t=1}^{n}\mathbb{E}_{\mathrm{s+r+d+c}}\nabla\mathcal{L}_{1}^{(t)}(\bm{\theta}^{\star})\nabla\mathcal{L}_{1}^{(t)}(\bm{\theta}^{\star})^{\top}\overset{\mathrm{a.s.}}{\rightarrow}\mathbb{E}_{\mathrm{s+r+d+c+u}}\nabla\mathcal{L}_{1}^{(t)}(\bm{\theta}^{\star})\nabla\mathcal{L}_{1}^{(t)}(\bm{\theta}^{\star})^{\top}=\bm{V}_{\mathrm{same},\pi}^{\infty}
\]
and 
\[
\frac{1}{n}\sum_{t=1}^{n}\mathbb{E}_{\mathrm{s+r+d+c}}\nabla\mathcal{L}_{1}^{(t)}(\bm{\theta}^{\star})\nabla\mathcal{L}_{2}^{(t)}(\bm{\theta}^{\star})^{\top}\overset{\mathrm{a.s.}}{\rightarrow}\mathbb{E}_{\mathrm{s+r+d+c+u}}\nabla\mathcal{L}_{1}^{(t)}(\bm{\theta}^{\star})\nabla\mathcal{L}_{2}^{(t)}(\bm{\theta}^{\star})^{\top}=\bm{V}_{\mathrm{diff},\pi}^{\infty}.
\]
For \myalgW, we further verify that for any $\bm{\theta}$, invoking
the strong law of large number, as $n\rightarrow\infty$, 
\[
\frac{1}{n}\sum_{t=1}^{n}\mathbb{E}_{\mathrm{s+d+c}}\mathcal{L}_{\mathrm{WP}}^{(t)}(\bm{\theta})\overset{\mathrm{a.s.}}{\rightarrow}\mathbb{E}_{\mathrm{s+d+c+u}}\mathcal{L}_{\mathrm{WP}}^{(t)}(\bm{\theta}).
\]

\subsubsection{Computing the Hessian }

In this section we compute $\bm{H}_{\pi}^{\infty}$. Recall that 
\[
\bm{H}_{\pi}^{\infty}=\mathbb{E}_{\mathrm{s+r+d+c+u}}\nabla^{2}\mathcal{L}_{1}^{(t)}(\bm{\theta}^{\star})=\mathbb{E}_{\mathrm{s+d+c+u}}\nabla^{2}\mathcal{L}_{\mathrm{WP}}^{(t)}(\bm{\theta}^{\star})
\]
and 
\begin{align*}
\nabla^{2}\mathcal{L}_{\mathrm{WP}}^{(t)}(\bm{\theta}^{\star}) & =\sum_{\substack{(i,j):i>j,\\
(t,i),(t,j)\in\mathcal{G}_{X}
}
}\frac{\widetilde{m}_{t}}{m_{t}(m_{t}-1)}\left[\left(\frac{e^{\theta_{i}^{\star}+\theta_{j}^{\star}}}{(e^{\theta_{i}^{\star}}+e^{\theta_{j}^{\star}})^{2}}\right)(\bm{e}_{i}-\bm{e}_{j})(\bm{e}_{i}-\bm{e}_{j})^{\top}\mathds{1}\{X_{ti}\neq X_{tj}\}\right]\\
 & =\sum_{\substack{(i,j):i>j}
}\frac{\mathds{1}\{t,i),(t,j)\in\mathcal{G}_{X}\}}{mp-1}\left[\frac{1}{4}(\bm{e}_{i}-\bm{e}_{j})(\bm{e}_{i}-\bm{e}_{j})^{\top}\mathds{1}\{X_{ti}\neq X_{tj}\}\right]
\end{align*}
Taking expectation, we have 
\begin{align*}
\mathbb{E}_{\mathrm{s+r+d+c+u}}\nabla^{2}\mathcal{L}_{\mathrm{WP}}^{(t)}(\bm{\theta}^{\star}) & =\sum_{\substack{(i,j):i>j}
}\frac{\mathbb{E}_{\mathrm{s}}\mathds{1}\{(t,i),(t,j)\in\mathcal{G}_{X}\}}{4(mp-1)}\cdot\\
&\qquad\mathbb{E}_{\mathrm{d+c+u}}\left[\mathds{1}\{X_{ti}\neq X_{tj}\}\mid(t,i),(t,j)\in\mathcal{G}_{X}\right](\bm{e}_{i}-\bm{e}_{j})(\bm{e}_{i}-\bm{e}_{j})^{\top}.
\end{align*}
For sampling we have 
\[
\mathbb{E}_{\mathrm{s}}\mathds{1}\{(t,i),(t,j)\in\mathcal{G}_{X}\}=\frac{\binom{mp}{2}}{\binom{m}{2}}=\frac{mp(mp-1)}{m(m-1)}.
\]
For comparison we have 
\begin{align*}
 & \mathbb{E}_{\mathrm{d+c}}\left[\mathds{1}\{X_{ti}\neq X_{tj}\}\mid(t,i),(t,j)\in\mathcal{G}_{X}\right]\\
 & \quad=\mathbb{P}_{\mathrm{d+c}}\left[X_{ti}>X_{tj}\mid(t,i),(t,j)\in\mathcal{G}_{X}\right]+\mathbb{P}_{\mathrm{d+c}}\left[X_{ti}<X_{tj}\mid(t,i),(t,j)\in\mathcal{G}_{X}\right]\\
 & \quad=\frac{e^{\theta_{i}^{\star}}}{(e^{\theta_{i}^{\star}}+e^{\zeta_{t}^{\star}})}\frac{e^{\zeta_{t}^{\star}}}{(e^{\theta_{i}^{\star}}+e^{\zeta_{t}^{\star}})}+\frac{e^{\zeta_{t}^{\star}}}{(e^{\theta_{i}^{\star}}+e^{\zeta_{t}^{\star}})}\frac{e^{\theta_{j}^{\star}}}{(e^{\theta_{i}^{\star}}+e^{\zeta_{t}^{\star}})}\\
 & \quad=\frac{e^{\zeta_{t}^{\star}}}{(e^{\zeta_{t}^{\star}}+1)^{2}},
\end{align*}
where the last line uses the assumption $\bm{\theta}^{\star}=\bm{0}_{m}$.
Combining these with the definition of $\beta$ in (\ref{eq:beta_def}),
we have 
\begin{align*}
\mathbb{E}_{\mathrm{s+c+u}}\nabla^{2}\mathcal{L}_{\mathrm{WP}}^{(t)}(\bm{\theta}^{\star}) & =\sum_{\substack{(i,j):i>j}
}\frac{p}{4(m-1)}\mathbb{E}_{\mathrm{u}}\left[\frac{e^{\zeta_{t}^{\star}}}{(e^{\zeta_{t}^{\star}}+1)^{2}}\right](\bm{e}_{i}-\bm{e}_{j})(\bm{e}_{i}-\bm{e}_{j})^{\top}\\
 & =\sum_{\substack{(i,j):i>j}
}\frac{\beta p}{4(m-1)}(\bm{e}_{i}-\bm{e}_{j})(\bm{e}_{i}-\bm{e}_{j})^{\top}\\
 & =\frac{\beta mp}{4(m-1)}\left[\bm{I}_{m}-\frac{1}{m}\bm{1}_{m}\bm{1}_{m}^{\top}\right].
\end{align*}

\subsubsection{Intra-split covariance}

In this section we compute $\bm{V}_{\mathrm{same},\pi}^{\infty}$.
Recall that
\[
\bm{V}_{\mathrm{same},\pi}^{\infty}=\mathbb{E}_{\mathrm{s+r+d+c+u}}\nabla\mathcal{L}_{1}^{(t)}(\bm{\theta}^{\star})\nabla\mathcal{L}_{1}^{(t)}(\bm{\theta}^{\star})^{\top}
\]
We expand the term $\nabla\mathcal{L}_{1}^{(t)}(\bm{\theta}^{\star})$
with 
\[
\nabla\mathcal{L}_{1}^{(t)}(\bm{\theta}^{\star})=\sum_{\substack{(i,j):i>j\\
(i,j,t)\in\Omega_{k}
}
}\underbrace{\left[\left(-\mathds{1}\{X_{ti}<X_{tj}\}+\frac{e^{\theta_{j}^{\star}}}{e^{\theta_{i}^{\star}}+e^{\theta_{j}^{\star}}}\right)(\bm{e}_{i}-\bm{e}_{j})\mathds{1}\{X_{ti}\neq X_{tj}\}\right]}_{\eqqcolon\bm{u}_{ij}^{(t)}}.
\]
We then have that 
\begin{align}
\mathbb{E}_{\mathrm{s+r+d+c+u}}\nabla\mathcal{L}_{1}^{(t)}(\bm{\theta}^{\star})\nabla\mathcal{L}_{1}^{(t)}(\bm{\theta}^{\star})^{\top} & =\mathbb{E}_{\mathrm{s+r+d+c+u}}\sum_{(i,j):i>j}\mathds{1}\{(i,j,1)\in\Omega_{1}\}\bm{u}_{ij}^{(t)}\bm{u}_{ij}^{(t)\top}\nonumber \\
 & =\sum_{(i,j):i>j}a_{ij}(\bm{e}_{i}-\bm{e}_{j})(\bm{e}_{i}-\bm{e}_{j})^{\top}\label{eq:intra_a}
\end{align}
where 
\begin{align*}
a_{ij} & =\mathbb{E}_{\mathrm{s+r+d+c+u}}\mathds{1}\{(i,j,t)\in\Omega_{1}\}\left[\left(-\mathds{1}\{X_{ti}<X_{tj}\}+\frac{e^{\theta_{j}^{\star}}}{e^{\theta_{i}^{\star}}+e^{\theta_{j}^{\star}}}\right)^{2}\mathds{1}\{X_{ti}\neq X_{1j}\}\right]\\
 & =\mathbb{P}_{\mathrm{s+r}}\left[(i,j,t)\in\Omega_{1}\right]\mathbb{E}_{\mathrm{d+c+u}}\left[\left(-\mathds{1}\{X_{ti}<X_{tj}\}+\frac{e^{\theta_{j}^{\star}}}{e^{\theta_{i}^{\star}}+e^{\theta_{j}^{\star}}}\right)^{2}\mathds{1}\{X_{ti}\neq X_{1j}\}\mid(i,j,1)\in\Omega_{1}\right].
\end{align*}
The last equality holds since comparison and user parameter draw are
independent since sampling and random pairing. Observe the fact that
\[
\sum_{(i,j):i>j}\mathds{1}\{(i,j,1)\in\Omega_{1}\}=\text{number of pairs in a splitting}=m_{t}/2=mp/2.
\]
By symmetry, 
\begin{align*}
\mathbb{P}_{\mathrm{s+r}}\{(i,j,t)\in\Omega_{1}\} & =\binom{m}{2}^{-1}\mathbb{E}_{\mathrm{s+r}}\sum_{(i,j):i>j}\mathds{1}\{(i,j,1)\in\Omega_{1}\}\\
 & =\frac{p}{m-1}.
\end{align*}
Condition on the event $(i,j,t)\in\Omega_{1}$ (which we omit in the
formulas below for formatting),
\begin{align*}
 & \mathbb{E}_{\mathrm{d+c+u}}\left[\left(-\mathds{1}\{X_{ti}<X_{tj}\}+\frac{e^{\theta_{j}^{\star}}}{e^{\theta_{i}^{\star}}+e^{\theta_{j}^{\star}}}\right)^{2}\mathds{1}\{X_{ti}\neq X_{tj}\}\right]\\
 & \quad=\mathbb{E}_{\mathrm{u}}\left[\mathbb{P}(X_{ti}<X_{tj})\left(-1+\frac{e^{\theta_{j}^{\star}}}{e^{\theta_{i}^{\star}}+e^{\theta_{j}^{\star}}}\right)^{2}+\mathbb{P}(X_{ti}>X_{tj})\left(\frac{e^{\theta_{j}^{\star}}}{e^{\theta_{i}^{\star}}+e^{\theta_{j}^{\star}}}\right)^{2}\right]\\
 & \quad=\mathbb{E}_{\mathrm{u}}\left[\frac{e^{\zeta_{1}^{\star}}}{e^{\zeta_{1}^{\star}}+e^{\theta_{i}^{\star}}}\frac{e^{\theta_{j}^{\star}}}{e^{\zeta_{1}^{\star}}+e^{\theta_{j}^{\star}}}\frac{e^{2\theta_{i}^{\star}}}{(e^{\theta_{i}^{\star}}+e^{\theta_{j}^{\star}})^{2}}+\frac{e^{\theta_{i}^{\star}}}{e^{\zeta_{1}^{\star}}+e^{\theta_{i}^{\star}}}\frac{e^{\zeta_{1}^{\star}}}{e^{\zeta_{1}^{\star}}+e^{\theta_{j}^{\star}}}\frac{e^{2\theta_{j}^{\star}}}{(e^{\theta_{i}^{\star}}+e^{\theta_{j}^{\star}})^{2}}\right]\\
 & \quad=\mathbb{E}_{\mathrm{u}}\left[\frac{e^{\zeta_{1}^{\star}+\theta_{i}^{\star}+\theta_{j}^{\star}}}{(e^{\zeta_{1}^{\star}}+e^{\theta_{i}^{\star}})(e^{\zeta_{1}^{\star}}+e^{\theta_{j}^{\star}})(e^{\theta_{i}^{\star}}+e^{\theta_{j}^{\star}})}\right].
\end{align*}
Substitute in the assumption $\bm{\theta}^{\star}=\bm{0}_{m}$, we
have 
\begin{align}
 & \mathbb{E}_{\mathrm{d+c+u}}\left[\left(-\mathds{1}\{X_{ti}<X_{tj}\}+\frac{e^{\theta_{j}^{\star}}}{e^{\theta_{i}^{\star}}+e^{\theta_{j}^{\star}}}\right)^{2}\mathds{1}\{X_{ti}\neq X_{tj}\}\mid(i,j,1)\in\Omega_{1}\right]\nonumber \\
 & \quad=\mathbb{E}_{\mathrm{u}}\left[\frac{e^{\zeta_{1}^{\star}}}{(e^{\zeta_{1}^{\star}}+1)(e^{\zeta_{1}^{\star}}+1)(1+1)}\right]=\frac{\beta}{2}.\label{eq:comparison_expectation_alpha}
\end{align}
Then 
\[
a_{ij}=\frac{\beta p}{2(m-1)}.
\]
In conclusion 
\begin{align*}
\bm{V}_{\mathrm{same},\pi}^{\infty} & =\sum_{(i,j):i>j}\frac{\beta p}{2(m-1)}(\bm{e}_{i}-\bm{e}_{j})(\bm{e}_{i}-\bm{e}_{j})^{\top}\\
 & =\frac{\beta mp}{2(m-1)}\left[\bm{I}_{m}-\frac{1}{m}\bm{1}_{m}\bm{1}_{m}^{\top}\right].
\end{align*}

\subsubsection{Inter-split covariance }

In this section we compute $\bm{V}_{\mathrm{diff},\pi}^{\infty}$.
Recall 
\begin{align*}
\nabla\mathcal{L}_{1}^{(t)}(\bm{\theta}^{\star}) & =\sum_{(i,j):i>j}\mathds{1}\{(i,j,t)\in\Omega_{1}\}\bm{u}_{ij};\\
\nabla\mathcal{L}_{2}^{(t)}(\bm{\theta}^{\star}) & =\sum_{(i,j):i>j}\mathds{1}\{(i,j,t)\in\Omega_{2}\}\bm{u}_{ij}.
\end{align*}
Then 

\begin{align}
\bm{V}_{\mathrm{diff},\pi}^{\infty} & =\mathbb{E}_{\mathrm{s+r+d+c+u}}\nabla\mathcal{L}_{1}^{(t)}(\bm{\theta}^{\star})\nabla\mathcal{L}_{2}^{(t)}(\bm{\theta}^{\star})^{\top}\nonumber \\
 & =\sum_{(i_{1},j_{1}):i_{1}>j_{1}}\sum_{(i_{2},j_{2}):i_{2}>j_{2}}\mathbb{E}_{\mathrm{s+r+d+c+u}}\mathds{1}\{(i_{1},j_{1},t)\in\Omega_{1}\}\mathds{1}\{(i_{2},j_{2},t)\in\Omega_{2}\}\bm{u}_{i_{1}j_{1}}\bm{u}_{i_{2}j_{2}}^{\top}.\label{eq:V_diff_pi}
\end{align}
Since $\bm{u}_{ij}$ is zero-mean for all $(i,j)$, it suffices to
only consider the terms where $i_{1},i_{2},j_{1},j_{2}$ has some
overlap. We first consider the terms where $(i_{1},j_{1})=(i_{2},j_{2})=(i,j)$.
\begin{align}
 & \mathbb{E}_{\mathrm{s+r+d+c+u}}\mathds{1}\{(i,j,t)\in\Omega_{1}\}\mathds{1}\{(i,j,t)\in\Omega_{2}\}\bm{u}_{ij}\bm{u}_{ij}^{\top}\nonumber \\
 & \quad=\mathbb{P}_{\mathrm{s+r}}\left[(i,j,t)\in\Omega_{1}\cap\Omega_{2}\right]\mathbb{E}_{\mathrm{d+c+u}}\left[\bm{u}_{ij}\bm{u}_{ij}^{\top}\mid\mathds{1}\{(i,j,t)\in\Omega_{1}\cap\Omega_{2}\}\right].\label{eq:cross_term_two_rep}
\end{align}
Similar to (\ref{eq:intra_a}) and (\ref{eq:comparison_expectation_alpha}),
we have $\mathbb{E}_{\mathrm{d+c+u}}\left[\bm{u}_{ij}\bm{u}_{ij}^{\top}\mid\mathds{1}\{(i,j,t)\in\Omega_{1}\cap\Omega_{2}\}\right]=0.5\beta(\bm{e}_{i}-\bm{e}_{j})(\bm{e}_{i}-\bm{e}_{j})^{\top}$.
On the other hand, 
\begin{align*}
\mathbb{P}_{\mathrm{s+r}}\left[(i,j,t)\in\Omega_{1}\cap\Omega_{2}\right] & =\mathbb{P}_{\mathrm{s}}\left[A_{ti}=A_{tj}=1\right]\mathbb{P}_{\mathrm{r}}\left[(i,j,t)\in\Omega_{1}\cap\Omega_{2}\mid A_{ti}=A_{tj}=1\right]\\
 & =\frac{\binom{mp}{2}}{\binom{m}{2}}\left(\frac{1}{mp-1}\right)^{2}\\
 & =\frac{p}{(m-1)(mp-1)}.
\end{align*}
Then 
\begin{equation}
\mathbb{E}_{\mathrm{s+r+d+c+u}}\mathds{1}\{(i,j,t)\in\Omega_{1}\}\mathds{1}\{(i,j,t)\in\Omega_{2}\}\bm{u}_{ij}\bm{u}_{ij}^{\top}=\frac{\beta p}{2(m-1)(mp-1)}(\bm{e}_{i}-\bm{e}_{j})(\bm{e}_{i}-\bm{e}_{j})^{\top}.\label{eq:two_pair_overlap}
\end{equation}
Now we consider the terms where only two of the four indices overlaps.
Without loss of generality, we take $i_{1}=i_{2}=i$ and remove the
$i_{1}>j_{1},i_{2}>j_{2}$ restriction. In other words, we consider
\begin{align*}
& \sum_{\substack{ (i,j_1,j_2): \\ i\neq j_1, i\neq j_2,  j_1\neq j_2 }}\mathbb{E}_{\mathrm{s+r+d+c+u}}\mathds{1}\{(i,j_{1},t)\in\Omega_{1}\}\mathds{1}\{(i,j_{2},t)\in\Omega_{2}\}\bm{u}_{ij_{1}}\bm{u}_{ij_{2}}^{\top}\\
& \quad=\sum_{\substack{ (i,j_1,j_2): \\ i\neq j_1, i\neq j_2,  j_1\neq j_2 }}\mathbb{P}_{\mathrm{s+r}}\left[(i,j_{1},t)\in\Omega_{1},(i,j_{2},t)\in\Omega_{2}\right]\mathbb{E}_{\mathrm{u+c}}\left[\bm{u}_{ij_{1}}\bm{u}_{ij_{2}}^{\top}\mid(i,j_{1},t)\in\Omega_{1},(i,j_{2},t)\in\Omega_{2}\right].
\end{align*}
We first deal with $\mathbb{P}_{\mathrm{s+r}}\{(i,j_{1},t)\in\Omega_{1},(i,j_{2},t)\in\Omega_{2}\}$,
\begin{align*}
\mathbb{P}_{\mathrm{s+r}}\{(i,j_{1},t)\in\Omega_{1},(i,j_{2},t)\in\Omega_{2}\} & =\mathbb{P}_{\mathrm{s}}\left[A_{ti}=A_{tj_{1}}=A_{tj_{2}}=1\right]\\
 & \quad\cdot\mathbb{P}_{\mathrm{r}}\left[(i,j_{1},t)\in\Omega_{1},(i,j_{2},t)\in\Omega_{2}\mid A_{ti}=A_{ti}=A_{ti}=1\right]\\
 & =\frac{\binom{mp}{3}}{\binom{m}{3}}\left(\frac{1}{mp-1}\right)^{2}\\
 & =\frac{p(mp-2)}{(m-1)(m-2)(mp-1)}.
\end{align*}
Now we consider the $\bm{u}_{ij_{1}}\bm{u}_{ij_{2}}^{\top}$ part.
Given $(i,j_{1},t)\in\Omega_{1}$ and $(i,j_{2},t)\in\Omega_{2}$
(for notation simplicity we omit this from now on), we compute 
\begin{equation}
\mathbb{E}_{\mathrm{d+c}}\left[\delta_{ij_{1}}\delta_{ij_{2}}\mathds{1}\{X_{ti}\neq X_{tj_{1}},X_{ti}\neq X_{1j_{2}}\}(\bm{e}_{i}-\bm{e}_{j_{1}})(\bm{e}_{i}-\bm{e}_{j_{2}})^{\top}\right].\label{eq:cross_term_single}
\end{equation}
where 
\[
\delta_{ij}\coloneqq\left(-\mathds{1}\{X_{1i}<X_{1j}\}+\frac{e^{\theta_{j}^{\star}}}{e^{\theta_{i}^{\star}}+e^{\theta_{j}^{\star}}}\right).
\]
For the scope of this proof we let $E_{1}$, $E_{2}$ be two events,
where $E_{1}\coloneqq\{X_{ti}=1,X_{tj_{1}}=0,X_{1j_{2}}=0\}$ and
$E_{2}\coloneqq\{X_{ti}=0,X_{tj_{1}}=1,X_{tj_{2}}=1\}$. Then we can
express (\ref{eq:cross_term_single}) as
\begin{align*}
 & \mathbb{E}_{\mathrm{d+c}}\left[\delta_{ij_{1}}\delta_{ij_{2}}\mathds{1}\{X_{ti}\neq X_{tj_{1}},X_{ti}\neq X_{1j_{2}}\}\right]\\
 & \quad=\mathbb{P}(E_{1})\frac{e^{\theta_{j_{1}}^{\star}}}{e^{\theta_{i}^{\star}}+e^{\theta_{j_{1}}^{\star}}}\frac{e^{\theta_{j_{2}}^{\star}}}{e^{\theta_{i}^{\star}}+e^{\theta_{j_{2}}^{\star}}}+\mathbb{P}(E_{2})\left[-1+\frac{e^{\theta_{j_{1}}^{\star}}}{e^{\theta_{i}^{\star}}+e^{\theta_{j_{1}}^{\star}}}\right]\left[-1+\frac{e^{\theta_{j_{2}}^{\star}}}{e^{\theta_{i}^{\star}}+e^{\theta_{j_{2}}^{\star}}}\right]\\
 & \quad=\frac{e^{\theta_{i}^{\star}}}{e^{\theta_{i}^{\star}}+e^{\zeta_{t}^{\star}}}\frac{e^{\zeta_{t}^{\star}}}{e^{\theta_{j_{1}}^{\star}}+e^{\zeta_{t}^{\star}}}\frac{e^{\zeta_{t}^{\star}}}{e^{\theta_{j_{2}}^{\star}}+e^{\zeta_{t}^{\star}}}\frac{e^{\theta_{j_{1}}^{\star}}}{e^{\theta_{i}^{\star}}+e^{\theta_{j_{1}}^{\star}}}\frac{e^{\theta_{j_{2}}^{\star}}}{e^{\theta_{i}^{\star}}+e^{\theta_{j_{2}}^{\star}}}\\
 & \quad\quad+\frac{e^{\zeta_{t}^{\star}}}{e^{\theta_{i}^{\star}}+e^{\zeta_{t}^{\star}}}\frac{e^{\theta_{j_{1}}^{\star}}}{e^{\theta_{j_{1}}^{\star}}+e^{\zeta_{t}^{\star}}}\frac{e^{\theta_{j_{2}}^{\star}}}{e^{\theta_{j_{2}}^{\star}}+e^{\zeta_{t}^{\star}}}\frac{e^{\theta_{i}^{\star}}}{e^{\theta_{i}^{\star}}+e^{\theta_{j_{1}}^{\star}}}\frac{e^{\theta_{i}^{\star}}}{e^{\theta_{i}^{\star}}+e^{\theta_{j_{2}}^{\star}}}\\
 & \quad=\frac{e^{\zeta_{t}^{\star}+\theta_{i}^{\star}+\theta_{j_{1}}^{\star}+\theta_{j_{2}}^{\star}}}{(e^{\theta_{j_{1}}^{\star}}+e^{\zeta_{t}^{\star}})(e^{\theta_{j_{2}}^{\star}}+e^{\zeta_{t}^{\star}})(e^{\theta_{i}^{\star}}+e^{\theta_{j_{1}}^{\star}})(e^{\theta_{i}^{\star}}+e^{\theta_{j_{2}}^{\star}})}.
\end{align*}
Furthermore we taken expectation $\mathbb{E}_{\mathrm{u}}$ to reach
\[
\mathbb{E}_{\mathrm{d+c+u}}\left[\delta_{ij_{1}}\delta_{ij_{2}}\mathds{1}\{X_{1i}\neq X_{1j_{1}},X_{1i}\neq X_{1j_{2}}\}\right]=\beta\cdot\frac{e^{\theta_{i}^{\star}+\theta_{j_{1}}^{\star}+\theta_{j_{2}}^{\star}}}{(e^{\theta_{i}^{\star}}+e^{\theta_{j_{1}}^{\star}})(e^{\theta_{i}^{\star}}+e^{\theta_{j_{2}}^{\star}})}=\frac{\beta}{4}.
\]
Then 
\begin{align}
 & \mathbb{E}_{\mathrm{s+r+d+c+u}}\mathds{1}\{(i,j_{1},t)\in\Omega_{1}\}\mathds{1}\{(i,j_{2},t)\in\Omega_{2}\}\bm{u}_{ij_{1}}\bm{u}_{ij_{2}}^{\top}\nonumber \\
 & \quad=\frac{\beta p(mp-2)}{4(m-1)(m-2)(mp-1)}(\bm{e}_{i}-\bm{e}_{j_{1}})(\bm{e}_{i}-\bm{e}_{j_{2}})^{\top}.\label{eq:one_pair_overlap}
\end{align}
We can now compute $\bm{V}_{\mathrm{diff},\pi}^{\infty}$ with (\ref{eq:V_diff_pi}),
(\ref{eq:two_pair_overlap}) and (\ref{eq:one_pair_overlap}),
\begin{align*}
\bm{V}_{\mathrm{diff},\pi}^{\infty} & =\sum_{(i_{1},j_{1}):i_{1}>j_{1}}\sum_{(i_{2},j_{2}):i_{2}>j_{2}}\mathbb{E}_{\mathrm{s+r+d+c+u}}\mathds{1}\{(i_{1},j_{1},t)\in\Omega_{1}\}\mathds{1}\{(i_{2},j_{2},t)\in\Omega_{2}\}\bm{u}_{i_{1}j_{1}}\bm{u}_{i_{2}j_{2}}^{\top}.\\
 & =\sum_{(i_{1},j_{1}):i_{1}>j_{1}}\mathbb{E}_{\mathrm{s+r+d+c+u}}\mathds{1}\{(i,j,t)\in\Omega_{1}\}\mathds{1}\{(i,j,t)\in\Omega_{2}\}\bm{u}_{ij}\bm{u}_{ij}^{\top}\\
 & \quad+\sum_{(i,j_{1},j_{2}):i\neq j_{1},i\neq j_{2},j_{1}\neq j_{2}}\mathbb{E}_{\mathrm{s+r+d+c+u}}\mathds{1}\{(i,j_{1},t)\in\Omega_{1}\}\mathds{1}\{(i,j_{2},t)\in\Omega_{2}\}\bm{u}_{ij_{1}}\bm{u}_{ij_{2}}^{\top}\\
 & =\sum_{(i_{1},j_{1}):i_{1}>j_{1}}\frac{\beta p}{2(m-1)(mp-1)}(\bm{e}_{i}-\bm{e}_{j})(\bm{e}_{i}-\bm{e}_{j})^{\top}\\
 & \quad+\sum_{(i,j_{1},j_{2}):i\neq j_{1},i\neq j_{2},j_{1}\neq j_{2}}\frac{\beta p(mp-2)}{4(m-1)(m-2)(mp-1)}(\bm{e}_{i}-\bm{e}_{j_{1}})(\bm{e}_{i}-\bm{e}_{j_{2}})^{\top}.
\end{align*}
Simplifying this expression gives us 
\begin{align*}
\bm{V}_{\mathrm{diff},\pi}^{\infty} & =\frac{\beta mp}{2(m-1)(mp-1)}\bm{I}_{m}-\frac{\beta p}{2(m-1)(mp-1)}\bm{1}_{m}\bm{1}_{m}^{\top}\\
 & \quad+\frac{\beta mp(mp-2)}{4(m-1)(mp-1)}\bm{I}_{m}-\frac{\beta p(mp-2)}{4(m-1)(mp-1)}\bm{1}_{m}\bm{1}_{m}^{\top}\\
 & =\frac{\beta m^{2}p^{2}}{4(m-1)(mp-1)}\left[\bm{I}_{m}-\frac{1}{m}\bm{1}_{m}\bm{1}_{m}^{\top}\right].
\end{align*}

\subsection{Proof of (\ref{eq:Vsame_Vdiff}) \label{subsec:Proof_Vsame_Vdiff}}

We first introduce a more general result in the following lemma. The
proof is deferred to the end of this section. 

\begin{lemma}\label{lem:expectation-cauchy} Let $\bm{A}\in\mathbb{R}^{n\times d}$,
$\bm{B}\in\mathbb{R}^{n\times d}$ be real-value random matrices.
Suppose that $\mathbb{E}\bm{B}\bm{B}^{\top}=\mathbb{E}\bm{A}\bm{A}^{\top}$
and $\mathbb{E}\bm{A}\bm{B}^{\top}$ is symmetric. Then 

\[
\mathbb{E}\bm{A}\bm{B}^{\top}\preceq\mathbb{E}\bm{A}\bm{A}^{\top}.
\]

\end{lemma}Now let $\bm{A}$ be $\frac{1}{n}\sum_{t=1}^{n}\nabla\mathcal{L}_{1}^{(t)}(\bm{\theta}^{\star})$
and $\bm{B}$ be $\frac{1}{n}\sum_{t=1}^{n}\nabla\mathcal{L}_{2}^{(t)}(\bm{\theta}^{\star})$.
By symmetry of $\bm{A}$ and $\bm{B}$, $\mathbb{E}\bm{B}\bm{B}^{\top}=\mathbb{E}\bm{A}\bm{A}^{\top}$
and $\mathbb{E}\bm{A}\bm{B}^{\top}=\mathbb{E}\bm{B}\bm{A}^{\top}$
so $\mathbb{E}\bm{A}\bm{B}^{\top}$ is symmetric. Observe that $\mathbb{E}_{\mathrm{s+r+d+c}}\nabla\mathcal{L}_{1}^{(t)}(\bm{\theta}^{\star})=\bm{0}$
and fact that $\nabla\mathcal{L}_{1}^{(t_{1})}(\bm{\theta}^{\star}),\nabla\mathcal{L}_{1}^{(t_{2})}(\bm{\theta}^{\star})$
are independent for any $t_{1}\neq t_{2}$. Then
\begin{align*}
\mathbb{E}\bm{A}\bm{A}^{\top} & =\frac{1}{n^{2}}\sum_{t_{1}=1}^{n}\sum_{t_{2}=1}^{n}\mathbb{E}_{\mathrm{s+r+d+c}}\nabla\mathcal{L}_{1}^{(t_{1})}(\bm{\theta}^{\star})\nabla\mathcal{L}_{1}^{(t_{2})}(\bm{\theta}^{\star})^{\top}\\
 & =\frac{1}{n^{2}}\sum_{t=1}^{n}\mathbb{E}_{\mathrm{s+r+d+c}}\nabla\mathcal{L}_{1}^{(t)}(\bm{\theta}^{\star})\nabla\mathcal{L}_{1}^{(t)}(\bm{\theta}^{\star})^{\top}.
\end{align*}
Similarly, 
\[
\mathbb{E}\bm{A}\bm{B}^{\top}=\frac{1}{n^{2}}\sum_{t=1}^{n}\mathbb{E}_{\mathrm{s+r+d+c}}\nabla\mathcal{L}_{1}^{(t)}(\bm{\theta}^{\star})\nabla\mathcal{L}_{2}^{(t)}(\bm{\theta}^{\star})^{\top}
\]
and 
\[
\mathbb{E}\bm{B}\bm{B}^{\top}=\frac{1}{n^{2}}\sum_{t=1}^{n}\mathbb{E}_{\mathrm{s+r+d+c}}\nabla\mathcal{L}_{2}^{(t)}(\bm{\theta}^{\star})\nabla\mathcal{L}_{2}^{(t)}(\bm{\theta}^{\star})^{\top}.
\]
Invoking Lemma~\ref{lem:expectation-cauchy}, we have that
\[
\frac{1}{n}\sum_{t=1}^{n}\mathbb{E}_{\mathrm{s+r+d+c}}\nabla\mathcal{L}_{1}^{(t)}(\bm{\theta}^{\star})\nabla\mathcal{L}_{2}^{(t)}(\bm{\theta}^{\star})^{\top}\preceq\frac{1}{n}\sum_{t=1}^{n}\mathbb{E}_{\mathrm{s+r+d+c}}\nabla\mathcal{L}_{1}^{(t)}(\bm{\theta}^{\star})\nabla\mathcal{L}_{1}^{(t)}(\bm{\theta}^{\star})^{\top}.
\]
As this holds for every $n$, 
\[
\bm{V}_{\mathrm{diff}}^{\infty}\preceq\bm{V}_{\mathrm{same}}^{\infty}.
\]
The proof is now completed.

\paragraph{Proof of Lemma~\ref{lem:expectation-cauchy}}

Let $\bm{v}\in\mathbb{R}^{d}$ be an arbitrary vector. It suffices
to show 
\[
\bm{v}^{\top}\left(\mathbb{E}\bm{A}\bm{B}^{\top}\right)\bm{v}\le\bm{v}^{\top}\left(\mathbb{E}\bm{A}\bm{A}^{\top}\right)\bm{v}
\]
By Cauchy-Schwarz inequality and linearity of expectations, 
\begin{align*}
\bm{v}^{\top}\left(\mathbb{E}\bm{A}\bm{B}^{\top}\right)\bm{v} & =\mathbb{E}\bm{v}^{\top}\bm{A}\bm{B}^{\top}\bm{v}\\
 & \le\sqrt{\mathbb{E}\bm{v}^{\top}\bm{A}\bm{A}^{\top}\bm{v}}\cdot\sqrt{\mathbb{E}\bm{v}^{\top}\bm{B}\bm{B}^{\top}\bm{v}}\\
 & =\mathbb{E}\bm{v}^{\top}\bm{A}\bm{A}^{\top}\bm{v}=\bm{v}^{\top}\left(\mathbb{E}\bm{A}\bm{A}^{\top}\right)\bm{v}.
\end{align*}

\section{Auxiliary lemmas}

In this section, we gather some auxiliary results that are useful
throughout this paper.

\begin{lemma}[Range of $z_{ij}$]\label{lemma:z_range} Recall 
\[
z_{ij}=\frac{e^{\theta_{i}^{\star}}e^{\theta_{j}^{\star}}}{(e^{\theta_{i}^{\star}}+e^{\theta_{j}^{\star}})^{2}}=\frac{e^{\theta_{i}^{\star}-\theta_{j}^{\star}}}{(1+e^{\theta_{i}^{\star}-\theta_{j}^{\star}})^{2}}.
\]
 For any $(i,j)$, 
\[
\frac{1}{4\kappa_{1}}\le z_{ij}\le\frac{1}{4}.
\]

\end{lemma}\begin{proof}

Consider the function $f:[0,\infty)\rightarrow\mathbb{R}$ defined
by $f(x)=x/(1+x)^{2}$. It has derivative $(1-x^{2})/(1+x)^{4}$,
so it is increasing at $x\in[0,1)$ and decreasing at $x\in(1,\infty)$.
By the definition of $\kappa_{1}$, $|\theta_{i}^{\star}-\theta_{j}^{\star}|\le\log(\kappa_{1})$.
Then 
\[
\frac{1}{4\kappa_{1}}\le f(e^{-\log(\kappa_{1})})\wedge f(e^{\log(\kappa_{1})})\le z_{ij}\le f(1)=\frac{1}{4}.
\]
\end{proof}

\begin{lemma}[Maximum eigenvalue of Laplacian]\label{lemma:top_eigen_Laplacian}
Let $\bm{L}=\sum_{(i,j):i>j}w_{ij}(\bm{e}_{i}-\bm{e}_{j})(\bm{e}_{i}-\bm{e}_{j})^{\top}$
be a weighted graph Laplacian. Then $\lambda_{1}(\bm{L})\le2\max_{i}\sum_{j}w_{ij}$.

\end{lemma}

\begin{proof}Let $\bm{v}\in\mathbb{R}^{m}$, then 
\begin{align*}
\bm{v}^{\top}\bm{L}\bm{v} & =\bm{v}^{\top}\sum_{(i,j):i>j}w_{ij}(\bm{e}_{i}-\bm{e}_{j})(\bm{e}_{i}-\bm{e}_{j})^{\top}\bm{v}\\
 & =\sum_{(i,j):i>j}w_{ij}(v_{i}-v_{j})^{2}\\
 & \le2\sum_{(i,j):i>j}w_{ij}(v_{i}^{2}+v_{j}^{2})\\
 & \le2\sum_{i}\sum_{j\neq i}w_{ij}v_{j}^{2}\\
 & \le2\sum_{i}\max_{i}\sum_{j}w_{ij}\|\bm{v}\|^{2}.
\end{align*}
So $\lambda_{1}(\bm{L})=\max_{\bm{v}\in\mathbb{R}^{m},\|\bm{v}\|=1}\bm{v}^{\top}\bm{L}\bm{v}\le2\max_{i}\sum_{j}w_{ij}$.\end{proof}

\begin{lemma}[A quantitative version of Sylvester's law of inertia,~\cite{ostrowski1959quantitative}]\label{lemma:sylvester}
For any real symmetric matrix $\bm{A}\in\mathbb{R}^{n\times n}$ and
$\bm{S}\in\mathbb{R}^{n\times n}$ be a non-singular matrix. Then
for any $i\in[n]$, $\lambda_{i}(\bm{S}\bm{A}\bm{S}^{\top})$ lies
between $\lambda_{i}(\bm{A})\lambda_{1}(\bm{S}^{\top}\bm{S})$ and
$\lambda_{i}(\bm{A})\lambda_{n}(\bm{S}^{\top}\bm{S})$.

\end{lemma}

\begin{fact}\label{fact:Laplacian} Let $\mathcal{G}$ be an arbitrary
graph with $m$ vertices and let $\bm{L}_{w}$ be a weighted graph
Laplacian defined by 
\[
\bm{L}_{w}\coloneqq\sum_{i>j:(i,j)\in\mathcal{G}}w_{ij}(\bm{e}_{i}-\bm{e}_{j})(\bm{e}_{i}-\bm{e}_{j})^{\top}.
\]
If $w_{ij}>0$ for all $(i,j)\in\mathcal{G}$ and $\mathcal{G}$ is
a connected graph, then $\bm{L}_{w}$ is rank $m-1$, $\bm{L}_{w}\bm{1}_{m}=\bm{0}_{m}$
and $\bm{L}_{w}^{\dagger}\bm{1}_{m}=0$. Moreover for any $i\in[n-1]$,
$\lambda_{i}(\bm{L}_{w}^{\dagger})=\lambda_{n-i}(\bm{L}_{w})$.

\end{fact}\begin{proof} The fact that $\bm{L}_{w}$ is rank $m-1$
when $\mathcal{G}$ is connected is well-known. See e.g. \cite{spielman2007spectral}
for reference. Since $\bm{L}_{w}$ is a real symmetric matrix, it
has an eigendecomposition $\bm{L}_{w}=\bm{U}\bm{\Sigma}\bm{U}^{\top}$
and then $\bm{L}_{w}^{\dagger}=\bm{U}\bm{\Sigma}^{\dagger}\bm{U}^{\top}$.
The rest follows from this decomposition and the form of $\bm{L}_{w}$.\end{proof}

\subsection{Proof of Fact~\ref{fact:BTL}\label{subsec:Proof_BTL}}

Expanding the probability of the random events, we have 

\begin{align*}
\mathbb{P}[X_{ti}<X_{tj}\mid X_{ti}\neq X_{tj}] & =\frac{\mathbb{P}[X_{ti}=0,X_{tj}=1]}{\mathbb{P}[X_{ti}=0,X_{tj}=1\text{\text{ or }}X_{ti}=1,X_{tj}=0]}\\
 & =\frac{e^{\zeta_{t}^{\star}}e^{\theta_{j}^{\star}}}{(e^{\zeta_{t}^{\star}}+e^{\theta_{i}^{\star}})(e^{\zeta_{t}^{\star}}+e^{\theta_{j}^{\star}})}\cdot\\ &\qquad\left(\frac{e^{\zeta_{t}^{\star}}e^{\theta_{j}^{\star}}}{(e^{\zeta_{t}^{\star}}+e^{\theta_{i}^{\star}})(e^{\zeta_{t}^{\star}}+e^{\theta_{j}^{\star}})}+\frac{e^{\theta_{i}^{\star}}e^{\zeta_{t}^{\star}}}{(e^{\zeta_{t}^{\star}}+e^{\theta_{i}^{\star}})(e^{\zeta_{t}^{\star}}+e^{\theta_{j}^{\star}})}\right)^{-1}\\
 & =\frac{e^{\zeta_{t}^{\star}}e^{\theta_{j}^{\star}}}{e^{\zeta_{t}^{\star}}(e^{\theta_{i}^{\star}}+e^{\theta_{j}^{\star}})}=\frac{e^{\theta_{j}^{\star}}}{e^{\theta_{i}^{\star}}+e^{\theta_{j}^{\star}}}.
\end{align*}
Now consider $\mathbb{P}[X_{ti}\neq X_{tj}]$, we have
\begin{align}
\mathbb{P}[X_{ti}\neq X_{tj}] & =\frac{e^{\zeta_{t}^{\star}}e^{\theta_{j}^{\star}}}{(e^{\zeta_{t}^{\star}}+e^{\theta_{i}^{\star}})(e^{\zeta_{t}^{\star}}+e^{\theta_{j}^{\star}})}+\frac{e^{\theta_{i}^{\star}}e^{\zeta_{t}^{\star}}}{(e^{\zeta_{t}^{\star}}+e^{\theta_{i}^{\star}})(e^{\zeta_{t}^{\star}}+e^{\theta_{j}^{\star}})}\nonumber \\
 & =\frac{e^{\theta_{j}^{\star}-\zeta_{t}^{\star}}+e^{\theta_{i}^{\star}-\zeta_{t}^{\star}}}{(1+e^{\theta_{i}^{\star}-\zeta_{t}^{\star}})(1+e^{\theta_{j}^{\star}-\zeta_{t}^{\star}})}.\label{eq:X_neq}
\end{align}
Let $f:[1/\kappa_{2},\kappa_{2}]^{2}\rightarrow\mathbb{R}$ defined
by 
\[
f(a,b)\coloneqq\frac{a+b}{(1+a)(1+b)}.
\]
Its partial derivatives are
\[
\frac{\partial}{\partial a}f(a,b)=\frac{b^{2}-1}{(1+a)^{2}(1+b)^{2}}\quad\text{and}\quad\frac{\partial}{\partial b}f(a,b)=\frac{a^{2}-1}{(1+a)^{2}(1+b)^{2}}.
\]
It is now easy to see that the minimum or maximum of $f$ can only
happen if $(a,b)=(1,1)$ or $(a,b)\in\{1/\kappa_{2},\kappa_{2}\}^{2}$.
After comparing the value of $f$ at these points, we conclude that
$f$ achieves minimum at 
\[
f(1/\kappa_{2},1/\kappa_{2})=f(\kappa_{2},\kappa_{2})=\frac{2\kappa_{2}}{(1+\kappa_{2})^{2}}.
\]
By the definition of $\kappa_{2}$, $|\theta_{l}^{\star}-\zeta_{t}^{\star}|\le\log(\kappa_{2})$
for any $l\in[m]$. Then (\ref{eq:X_neq}) fits the definition of
$f$ and the proof is completed.

\end{document}